\documentclass[10pt,reqno]{amsart}
\usepackage{algorithm,algorithmicx}
\usepackage{algpseudocode}
\usepackage{bbm}
\usepackage[mathlines,pagewise,left]{lineno}
\usepackage{amssymb}
\usepackage[dvipsnames]{xcolor}
\usepackage{amsmath}
\usepackage{mathtools}
\usepackage{bcol}
\usepackage{multirow}
\usepackage{nicefrac}
\usepackage[numbers]{natbib}
\usepackage{comment}
\usepackage[toc,page]{appendix}
\usepackage{graphicx} 
\usepackage{booktabs} 
\usepackage{pdflscape}
\usepackage{subfig}

\usepackage{amsthm}
\newtheorem{proposition}{Proposition}

\newtheorem{theorem}{Theorem}

\theoremstyle{definition}

\theoremstyle{remark}
\newtheorem{remark}{Remark}
\newtheorem{example}{Example}

\newcommand{\R}{\ensuremath{\mathbb{R}}}

\newcommand{\Z}{\ensuremath{\mathbb{Z}}}
\newcommand{\Sp}{\ensuremath{\mathcal{S}_+}}
\newcommand{\bs}[1]{\ensuremath{\boldsymbol{#1}}}

\newcommand{\Ps}{\ensuremath{\mathcal{P}}}

\newcommand{\conv}{\ensuremath{\text{conv}}}
\newcommand{\clconv}{\ensuremath{\text{cl conv}}}
\newcommand{\be}[1]{\begin{equation}\label{#1}}
\newcommand{\ee}{\end{equation}}

\DeclareMathOperator*{\argmin}{arg\,min}
\newcommand{\diag}{\ensuremath{\text{diag}}}
\newcommand{\sign}{\ensuremath{\text{sign}}}
\newcommand{\lasso}{\texttt{lasso}}
\newcommand{\persp}{\texttt{perspective relaxation}}
\newcommand{\MC}{\texttt{MC}$_+$}
\newcommand{\RO}{\texttt{R1}}
\newcommand*{\sdp}[1]{\ensuremath{\texttt{sdp}_{\texttt{#1}}}}
\newcommand*{\cqp}[1]{\ensuremath{\texttt{cqp}_{#1}}}
\newcommand*{\mio}{\ensuremath{\texttt{big-M}}}
\newcommand*{\rev}[1]{{#1}}
\newcommand*{\mioP}{\ensuremath{\texttt{persp}}}

\usepackage[colorinlistoftodos,prependcaption,textsize=small]{todonotes}

\begin{document}

\title[Rank-one convexification for sparse regression]{Rank-one convexification for sparse regression}
\author{Alper Atamt\"urk and  Andr\'{e}s G\'{o}mez}

\thanks{ \noindent \hskip -5mm
A. Atamt\"urk: Department of Industrial Engineering \& Operations Research, University of California, Berkeley, CA 94720.
\texttt{atamturk@berkeley.edu}   \\
A. G\'{o}mez: Department of Industrial Engineering, Swanson School of Engineering, University of Pittsburgh, PA 15261. \texttt{agomez@pitt.edu}
}

\ignore{
\begin{frontmatter}
	
\title{Rank-one convexification for sparse regression}
\runtitle{Rank-one convexification for sparse regression}

\begin{aug}
	
	 \author{\fnms{Alper}  \snm{Alper Atamt\"urk\thanksref{t2}}
	 	\corref{}\ead[label=e1]{atamturk@berkeley.edu}} 
 	\and
	\author{\fnms{Andr\'{e}s} \snm{G\'{o}mez}\thanksref{t3}\ead[label=e2]{agomez@pitt.edu}}
	
	 \thankstext{t2}{A. Atamt\"urk: Department of Industrial Engineering \& Operations Research, University of California, Berkeley, CA 94720.
	 	\texttt{atamturk@berkeley.edu}}
 	
 		 \thankstext{t3}{A. G\'{o}mez: Department of Industrial \& Systems Engineering, Viterbi School of Engineering, University of Southern California, CA 90014. \texttt{gomezand@usc.edu}}
	
	
\end{aug}

\end{frontmatter}
}

\begin{abstract}
	Sparse regression models are increasingly prevalent due to their ease of interpretability and superior out-of-sample performance. However, the exact model of sparse regression with an $\ell_0$ constraint restricting the support of the estimators is a challenging (\NP-hard) non-convex optimization problem. 
	In this paper, we derive new strong convex relaxations for sparse regression. These relaxations are based on the ideal (convex-hull) formulations 
	for rank-one quadratic terms with indicator variables. The new relaxations can be formulated as semidefinite optimization problems in an 
	extended space and are stronger and more general than the state-of-the-art formulations, including the perspective reformulation and  formulations with the reverse Huber penalty
	and the minimax concave penalty functions. Furthermore, the proposed rank-one strengthening can be interpreted as a 
	\textit{non-separable, non-convex, unbiased} sparsity-inducing regularizer, 
	which dynamically adjusts its penalty according to the shape of the error function without inducing bias for the sparse solutions.
	In our computational experiments with benchmark datasets, the proposed conic formulations are solved within seconds and result in near-optimal solutions (with 0.4\% optimality gap) for non-convex $\ell_0$-problems. Moreover, the resulting estimators also outperform alternative convex approaches from a statistical perspective, achieving high prediction accuracy and good interpretability.
\end{abstract}

	\maketitle

\begin{center}
	January 2019; \rev{October 2020}
\end{center}

\BCOLReport{19.01}

\pagebreak

\section{Introduction}\label{sec:intro}

Given a model matrix $\bs{X}=[\bs{x_1},\ldots,\bs{x_p}]\in \R^{n\times p}$ of explanatory variables, a vector $\bs{y}\in \R^n$ of response variables, regularization parameters $\lambda,\mu\geq 0$ and a desired sparsity $k\in \Z_+$, we consider the least squares regression problem
\begin{equation} \label{eq:bestSubsetSelection}
\min_{\bs{\beta} \in \R^p}\;\|\bs{y}-\bs{X\beta}\|_2^2+\lambda \|\bs{\beta}\|_2^2 +\mu \|\bs{\beta}\|_1 \text{ s.t. }\|\bs{\beta}\|_0\leq k,
\end{equation}
where $\|\bs{\beta}\|_0$ denotes cardinality of the support of $\bs{\beta}$. Problem \eqref{eq:bestSubsetSelection} encompasses a broad range of the regression models. It includes as special cases: \texttt{ridge} regression \cite{hoerl1970ridge}, when $\lambda>0$, $\mu=0$ and $k\geq p$; \texttt{lasso} \cite{tibshirani1996regression}, when $\lambda=0$, $\mu\geq 0$ and $k\geq p$; \texttt{elastic net} \cite{zou2005regularization} when $\lambda,\mu>0$ and $k\geq p$; \texttt{best subset selection} \cite{miller2002subset}, when $\lambda=\mu=0$ and $k<p$. Additionally, \citet{bertsimas2017sparse} propose to solve \eqref{eq:bestSubsetSelection} with $\lambda>0$, $\mu=0$ and $k<p$ for high-dimensional regression problems, while \citet{mazumder2017subset} study \eqref{eq:bestSubsetSelection} with $\lambda=0$, $\mu>0$ and $k<p$ for problems with low Signal-to-Noise Ratios (SNR). The results in this paper cover all versions of \eqref{eq:bestSubsetSelection} with $k < p$; moreover, they
can be extended to problems with non-separable regularizations of the form $\lambda \|\bs{A\beta}\|_2^2 +\mu \|\bs{C\beta}\|_1$, resulting in sparse variants of the \texttt{fused lasso} \cite{padilla2017dfs,tibshirani2005sparsity}, \texttt{generalized lasso} \cite{lin2014alternating,tibshirani2011solution} and \texttt{smooth lasso} \cite{hebiri2011smooth}, among others. 

\subsection*{Regularization techniques} The motivation and benefits of the regularization are well-documented in the literature. \citet{hastie2001elements} coined the bet on sparsity principle, i.e., using an inference procedure that performs well in sparse problems since no procedure can do well in dense problems.  \texttt{Best subset selection} with $k<p$ and $\lambda = \mu = 0$ is the direct approach to enforce sparsity without incurring bias. In contrast, \text{ridge} regression with $\lambda>0$ (Tikhonov regularization) is known to induce shrinkage and bias, which can be desirable, for example, when $\bs{X}$ is not orthogonal, but it does not result in sparsity. On the other hand, \texttt{lasso}, the $\ell_1$ regularization  with $\mu>0$ simultaneously causes shrinkage and induces sparsity, but the inability to separately control for shrinkage and sparsity may result in subpar performance in some cases \cite{miller2002subset,zhang2008sparsity,zhang2012general,zhang2014lower,zhao2006model,zou2006adaptive}. 
Moreover, achieving a target sparsity level $k$ with \texttt{lasso} 
requires significant experimentation with the penalty parameter $\mu$ \cite{chichignoud2016practical}. 
When $k \ge p$, the cardinality constraint on $\ell_0$ is redundant and \eqref{eq:bestSubsetSelection} reduces to a convex optimization problem and can be solved easily. On the other hand, when $k<p$, problem \eqref{eq:bestSubsetSelection} is non-convex and \NP-hard \cite{natarajan1995sparse}, thus finding an optimal solution may require excessive computational effort and methods to solve it approximately are used instead \cite{huang2018constructive,nevo2017identifying}. Due to the perceived difficulties of tackling the non-convex $\ell_0$ constraint in \eqref{eq:bestSubsetSelection}, \texttt{lasso}-type simpler approaches are still preferred for inference problems with sparsity \cite{hastie2015statistical}. 

Nonetheless, there has been a substantial effort to develop sparsity-inducing methodologies that do not incur as much shrinkage and bias as \texttt{lasso} does. The resulting techniques often result in optimization problems of the form \begin{equation}\label{eq:regularization}\min_{\bs{\beta} \in \R^p}\;\|\bs{y}-\bs{X\beta}\|_2^2+\sum_{i=1}^p\rho_i(\beta_i)\end{equation}
where $\rho_i:\R\to\R$ are non-convex regularization functions. Examples of such regularization functions include $\ell_q$ penalties with $0<q<1$ \cite{frank1993statistical} and SCAD \cite{fan2001variable}. Although optimal solutions of \eqref{eq:regularization} with non-convex regularizations may substantially improve upon the estimators obtained by \texttt{lasso}, solving \eqref{eq:regularization} to optimality is still a difficult task \cite{hunter2005variable,mazumder2011sparsenet,zou2008one}, and suboptimal solutions may not benefit from the improved statistical properties. To address such difficulties, \citet{zhang2010nearly} propose the \texttt{minimax concave penalty} (\MC), a class of sparsity-inducing penalty functions where the non-convexity of $\rho$ is offset by the convexity of $\|\bs{y}-\bs{X\beta}\|_2^2$ for sufficiently sparse solutions, so that \eqref{eq:regularization} remains convex -- \citet{zhang2010nearly} refer to this property as sparse convexity. Thus, in the ideal scenario (and with proper tuning of the parameter controlling the concavity of $\rho$), the \MC\ penalty is able to retain the sparsity and unbiasedness of \texttt{best subset selection} while preserving convexity, resulting in the best of both worlds. However, due to the \emph{separable} form of the regularization term, the effectiveness of \MC\ greatly depends on the diagonal dominance of the matrix $\bs{X^\top X}$ (this statement will be made more precise in \S\ref{sec:regularization}), and may result in poor performance when the diagonal dominance is low. 

Unfortunately, in many practical applications, the matrix $\bs{X^\top X}$ has low eigenvalues and is not diagonally dominant at all. To illustrate, Table~\ref{tab:diagonalDominance} presents the diagonal dominance of five datasets from the UCI Machine Learning Repository \citep{Dua:2017} used in \cite{gomez2018mixed,miyashiro2015mixed}, as well as the \texttt{diabetes} dataset with all second interactions used in \cite{bertsimas2016best,efron2004least}. The diagonal dominance of a positive semidefinite matrix $\bs{A}$ is computed as 
$$\texttt{dd}(\bs{A}):=(1/\text{tr}(\bs{A}))\max_{\bs{d}\in \R_+^{p}}\bs{e^\top d} \text{ s.t. }\bs{A-\diag(\bs{d})}\succeq 0,
$$
where $\bs{e}$ is the $p$-dimensional vector of ones, $\diag(\bs{d})$ is the diagonal matrix such that $\diag(\bs{d})_{ii}=d_i$ and $\text{tr}(\bs{A})$ denotes the trace of $\bs{A}$. Accordingly, the diagonal dominance is the trace of the largest diagonal matrix that can be extracted from $\bs{A}$ without violating positive semidefiniteness, divided by the trace of $\bs{A}$. Observe in Table~\ref{tab:diagonalDominance} that the diagonal dominance of $\bs{X^\top X}$ is very low or even $0\%$, and \MC\ struggles for these datasets as we demonstrate in \S\ref{sec:computations}. 

	\begin{table}
		\caption{Diagonal dominance of $\bs{X^\top X}$ for benchmark datasets.}  
		\label{tab:diagonalDominance}
		\scalebox{1.0}{
		\begin{tabular}{c | c c | c }
			\hline \hline
			\multirow{2}{*}{\texttt{dataset}} & \multirow{2}{*}{$p$} & \multirow{2}{*}{$n$} &\multirow{2}{*}{\texttt{dd}$\times 100\%$}\\
			&&&\\
			\hline   
			\texttt{housing}&13&506& 26.7\% \\
			\texttt{servo}& 19& 167&0.0\%\\
			\texttt{auto MPG}& 25 & 	392 & 1.5\% \\
			\texttt{solar flare}& 26 & 	1,066& 8.8\%\\
			\texttt{breast cancer}& 37 & 196& 3.6\% \\
			\texttt{diabetes}& 	64 & 442 & 0.0\%\\
			\texttt{\rev{crime}}& \rev{100}	  & \rev{1993}  & \rev{13.5 \%}\\
			\hline \hline
		\end{tabular}}
	\end{table}

\subsection*{Mixed-integer optimization formulations} An alternative to utilizing non-convex regularizations is to leverage the recent advances in mixed-integer optimization (MIO) to tackle \eqref{eq:bestSubsetSelection} exactly \cite{bertsimas2015or,bertsimas2016best,cozad2014learning}. 
By introducing indicator variables $\bs{z}\in \{0,1\}^p$, where $z_i=\mathbbm{1}_{\beta_i\neq 0}$, problem \eqref{eq:bestSubsetSelection} can be reformulated as
\begin{subequations}\label{eq:MIO}
\begin{align}
\bs{y}^\top\bs{y}+\min_{\rev{\bs{\beta},\bs{z},\bs{u}}}\;&-2\bs{y}^\top\bs{X}\bs{\beta}+\bs{\beta^\top}\left(\bs{X^\top X}+\lambda \bs{I}\right)\bs{\beta}+\mu\sum_{i=1}^p u_i\\
\text{ s.t.}\;&\sum_{i=1}^pz_i\leq k\\
	& \beta_i\leq u_i,\; -\beta_i\leq u_i \quad i=1,\ldots,p\\
&\beta_i(1-z_i)=0\quad\quad \; \;\; i=1,\ldots,p\label{eq:MIO_complementary}\\
&\bs{\beta}\in \R^p,\; \bs{z}\in \{0,1\}^p,\; \bs{u}\in \R_+^p.
\end{align}
\end{subequations}
The non-convexity of \eqref{eq:bestSubsetSelection} is captured by the complementary constraints \eqref{eq:MIO_complementary} and the integrality constraints $\bs{z}\in \{0,1\}^p$. In fact, one of the main challenges for solving \eqref{eq:MIO} is handling constraints \eqref{eq:MIO_complementary}. A standard approach in the MIO literature is to use the so-called big-$M$ constraints and replace \eqref{eq:MIO_complementary} with 
\begin{equation}
\label{eq:bigM}
-Mz_i\leq \beta_i\leq Mz_i
\end{equation}
for a sufficiently large number $M$ to bound the variables $\beta_i$. However, these so-called big-$M$ constraints \eqref{eq:bigM} are poor approximations of constraints \eqref{eq:MIO_complementary}, \emph{especially in the case of regression problems where no natural big-$M$ value is available}. \citet{bertsimas2016best} propose approaches to compute provable big-$M$ values, but such values often result in prohibitively large computational times even in problems with a few dozens variables (or, even worse, may lead to numerical instabilities and cause convex solvers to crash). Alternatively, heuristic values for the big-$M$ values can be estimated, e.g., setting $M=\tau \|\hat{\bs{\beta}}\|_\infty$ where $\tau\in \R_+$ and $\hat{\bs{\beta}}$ is a feasible solution of \eqref{eq:bestSubsetSelection} found via a heuristic\footnote{This method with $\tau=2$ was used in the computations in \cite{bertsimas2016best}.}. While using such heuristic values yield reasonable performance for small enough values of $\tau$, it may eliminate optimal solutions.

Branch-and-bound algorithms for MIO leverage strong convex relaxations of problems to prune the search space and reduce the number of sub-problems to be enumerated (and, in some cases, eliminate the need for enumeration altogether). Thus, a critical step to speed-up the solution times for \eqref{eq:MIO} is to derive convex relaxations that approximate the non-convex problem well \cite{AN:conicmir:ipco}. Such strong relaxations can also be used directly to find good estimators for the inference problems (without branch-and-bound); in fact, it is well-known than the natural convex relaxation of \eqref{eq:MIO} with $\lambda=\mu=0$ and big-$M$ constraints is precisely \texttt{lasso}, see \cite{dong2015regularization} for example. Therefore, sparsity-inducing techniques that more accurately capture the properties of the non-convex constraint $\|\bs{\beta}\|_0\leq k$ can be found by deriving tighter convex relaxations of~\eqref{eq:bestSubsetSelection}.
\citet{pilanci2015sparse} exploit the Tikhonov regularization term and convex analysis to construct an improved convex relaxation using the \texttt{reverse Huber penalty}. In a similar vein, \citet{bertsimas2017sparse} leverage the Tikhonov regularization and duality to propose an efficient algorithm for high-dimensional sparse regression. 

\subsection*{The perspective relaxation} \label{sec:perspective}Problem \eqref{eq:MIO} is a mixed-integer convex quadratic optimization problem with indicator variables, a class of problems which has received a fair amount of attention in the optimization literature. In particular, the \persp\ \cite{akturk2009strong,frangioni2006perspective,gunluk2010perspective} is, by now, a standard technique that can be used to substantially strengthen the convex relaxations by exploiting \emph{separable} quadratic terms. Specifically, consider the mixed-integer epigraph of a one-dimensional quadratic function with an indicator constraint, 
$$Q_1=\left\{z\in \{0,1\},\beta\in \R,t\in \R_+: \beta_i^2\leq t,\; \beta_i(1-z_i)=0\right\} \cdot$$ 
The convex hull of $Q_1$ is obtained by relaxing the integrality constraint to bound constraints and using the closure of the perspective function\footnote{We use the convention that $\frac{\beta_i^2}{z_i}=0$ when $\beta_i=z_i=0$ and $\frac{\beta_i^2}{z_i}=\infty$ if $z_i=0$ and $\beta_i\neq 0$.} of $\beta_i^2$, expressed as a rotated cone constraint: 
$$\conv(Q_1)=\left\{z\in [0,1],\beta\in \R,t\in \R_+: \frac{\beta_i^2}{z_i}\leq t\right\} \cdot$$

\citet{xie2018ccp} apply the \persp\ to the separable quadratic regularization term $\lambda\|\bs{\beta}\|_2^2$, i.e., reformulate \eqref{eq:MIO} as 
\begin{subequations}\label{eq:MIOPersp}
	\begin{align}
	\bs{y}^\top\bs{y}+\min_{\rev{\bs{\beta},\bs{z},\bs{u}}}\;&-2\bs{y}^\top\bs{X}\bs{\beta}+\bs{\beta^\top}\left(\bs{X^\top X}\right)\bs{\beta}+\lambda\sum_{i=1}^p\frac{\beta_i^2}{z_i}+\mu\sum_{i=1}^p u_i\\
	\text{ s.t.}\;&\sum_{i=1}^pz_i\leq k\\
	& \beta_i\leq u_i,\; -\beta_i\leq u_i \quad i=1,\ldots,p\\
	&\bs{\beta}\in \R^p,\; \bs{z}\in \{0,1\}^p,\; \bs{u}\in \R_+^p.
	\end{align}
\end{subequations}
Moreover, they show that the continuous relaxation of \eqref{eq:MIOPersp} is equivalent to the continuous relaxation of the formulation used by \citet{bertsimas2017sparse}. \citet{dong2015regularization} also study the \persp\ in the context of regression: first, they show that using the \texttt{reverse Huber penalty} \cite{pilanci2015sparse} is, in fact, equivalent to just solving the convex relaxation of \eqref{eq:MIOPersp} --- thus the relaxations of \cite{bertsimas2017sparse,pilanci2015sparse,xie2018ccp} all coincide; second, they propose to use an \emph{optimal} \persp, i.e., by applying the perspective relaxation to a separable quadratic function $\bs{\beta^\top D\beta}$, where $\bs{D}$ is a nonnegative diagonal matrix such that $\bs{X^{\top}X}+\lambda \bs{I}-\bs{D}\succeq 0$; finally, they show that solving this stronger convex relaxation of the optimal \persp\ is, in fact, equivalent to using the \MC\ penalty \cite{zhang2010nearly}. 

\rev{The perspective relaxation is now a state-of-the-art method to convexify problems with separable terms and indicators variables. However, there are relatively few convexification techniques for problems without separable terms \cite{frangioni2020decompositions,gomez2018strong,han20202x2,jeon2017quadratic}. In fact, among the previously discussed methods for sparse regression,} the optimal \persp\ of \citet{dong2015regularization} is the only one that does not explicitly require the use of the Tikhonov regularization $\lambda\|\bs{\beta}\|_2^2$. Nonetheless, as the authors point out, if $\lambda=0$ then the method is effective only when the matrix $\bs{X^\top X}$ is sufficiently diagonally dominant, which, as illustrated in Table~\ref{tab:diagonalDominance}, is not necessarily the case in practice. As a consequence, \persp\ techniques 
may be insufficient to tackle problems when large shrinkage is undesirable and, hence, $\lambda$ is small.

\subsection*{Our contributions} In this paper we derive stronger convex relaxations of \eqref{eq:MIO} than the optimal \persp. These relaxations are obtained from the study of ideal (convex-hull) formulations of the mixed-integer epigraphs of \emph{non-separable rank-one quadratic functions with indicators}. Since the \persp\ corresponds to the ideal formulation of a \emph{one-dimensional} rank-one quadratic function, the proposed relaxations generalize and strengthen the existing results. In particular, they \emph{dominate} \persp\ approaches for all values of the regularization parameter $\lambda$ and, critically, are able to achieve high-quality approximations of \eqref{eq:bestSubsetSelection} even in low diagonal dominance settings with $\lambda=0$. Alternatively, our results can also be interpreted as a new \emph{non-separable}, \emph{non-convex}, \emph{unbiased} regularization penalty $\rho_{\RO}(\bs{\beta})$ which: \textit{(i)} imposes larger penalties than the separable minimax concave penalty \cite{zhang2010nearly} $\rho_{\text{\MC}}(\bs{\beta})$  to dense estimators, thus achieving better sparsity-inducing properties; and \textit{(ii)} the nonconvexity of the penalty function is offset by the convexity of the term $\|\bs{y}-\bs{X\beta}\|_2^2$, and the resulting continuous problem can be solved to global optimality using convex optimization tools. In fact, they can be formulated as semidefinite optimization and, in certain special cases, as conic quadratic optimization.

To illustrate the regularization point of view for the proposed relaxations, consider a two-predictor regression problem in Lagrangean form: 
\begin{equation} 
\min_{\bs{\beta} \in \R^2}\;\|\bs{y}-\bs{X\beta}\|_2^2+\lambda \|\bs{\beta}\|_2^2 +\mu \|\bs{\beta}\|_1 +\kappa\|\bs{\beta}\|_0,
\end{equation}
where
$\bs{X^{\top}X}=\begin{pmatrix}
1+\delta & 1\\ 1 & 1+\delta
\end{pmatrix}$ and $\delta\geq 0$ is a parameter controlling the diagonal dominance.
 Figure~\ref{fig:regularization} depicts the graphs of well-known regularizations including \lasso\ ($\lambda=\kappa=0$, $\mu=1$), \texttt{ridge} ($\mu=\kappa=0$, $\lambda=1$), \texttt{elastic net} ($\kappa=0$, $\lambda=\mu=0.5$), the \MC penalty for different values of $\delta$ and the proposed rank-one \RO\ regularization.
 The graphs of \MC\ and \RO\ are obtained by setting $\lambda=\mu=0$ and $\kappa=1$, and using the appropriate convex strengthening, see \S\ref{sec:regularization} for details. Observe that the \RO\ regularization results in larger penalties than \MC for all values of $\delta$, and the improvement increases  as $\delta\to 0$. In addition, Figure~\ref{fig:constrained} shows the effect of using the \texttt{lasso} constraint $\|\bs{\beta}\|_1\leq k$, the $\text{\MC}$ constraint $\rho_{\text{\MC}}(\bs{\beta})\leq k$, and the rank-one constraint $\rho_{\text{\RO}}(\bs{\beta})\leq k$ in a two-dimensional problem to achieve sparse solutions satisfying $\|\bs{\beta}\|_0\leq 1$. Specifically, let 
 $$\varepsilon^*=\min_{\|\bs{\beta}\|_0\leq 1}\|\bs{y}-\bs{X\beta}\|_2^2$$
 be the minimum residual error of a sparse solution of the least squares problem. Figure~\ref{fig:constrained} shows in gray the (possibly dense) points satisfying $\|\bs{y}-\bs{X\beta}\|_2^2\leq \varepsilon^*$, and it shows in color the set of feasible points satisfying $\rho(\bs{\beta})\leq k$, where $\rho$ is a given regularization and $k$ is chosen so that the feasible region (color) intersects the level sets (gray).
 We see that neither \lasso\ nor \MC\ is able to exactly recover an optimal sparse solution for any diagonal dominance parameter $\delta$,
 despite significant shrinkage ($k < 1$).  In contrast, the rank-one constraint
 $\rho_{\text{\RO}}(\bs{\beta})\leq k$ adapts to the curvature of the error function $\|\bs{y}-\bs{X\beta}\|_2^2$ to induce higher sparsity: in particular, the ``natural" constraint $\rho_{\text{\RO}}(\bs{\beta})\leq 1$, with the target sparsity $k=1$, results in exact recovery without shrinkage in all cases.
 
\begin{figure}[!hp]
	\centering
\subfloat[\texttt{ridge}]{\includegraphics[width=0.33\textwidth,trim={13cm 5.5cm 13cm 5.5cm},clip]{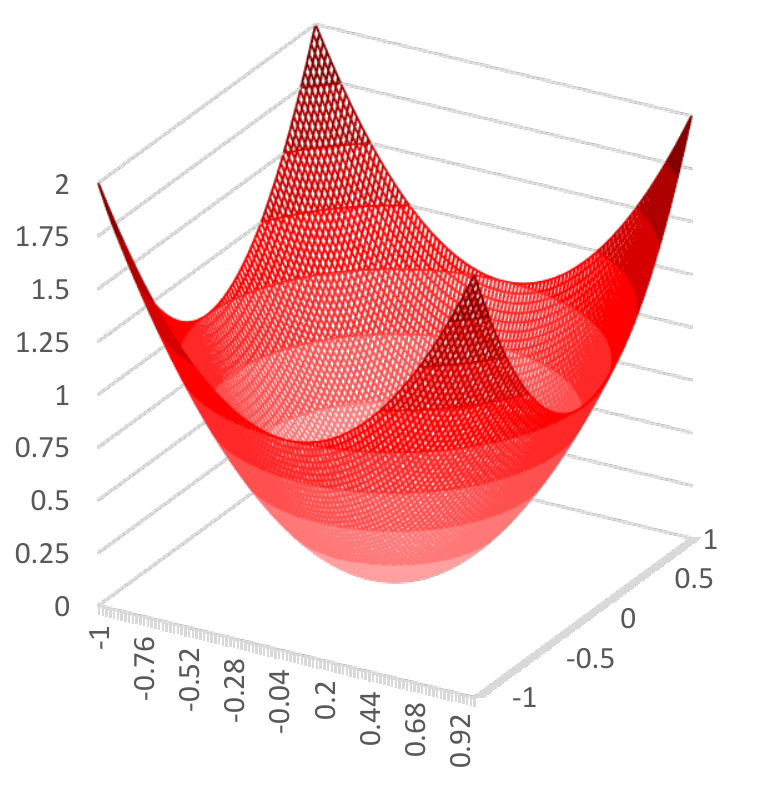}}\hfill\subfloat[\texttt{elastic net}]{\includegraphics[width=0.33\textwidth,trim={13cm 5.5cm 13cm 5.5cm},clip]{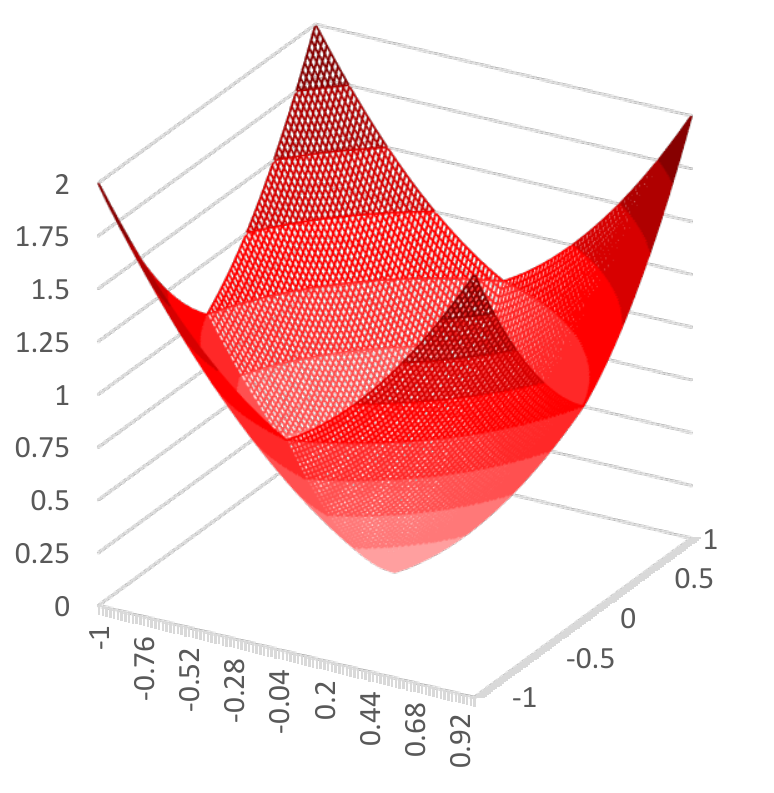}}\hfill	\subfloat[\texttt{lasso}]{\includegraphics[width=0.33\textwidth,trim={13cm 5.5cm 13cm 5.5cm},clip]{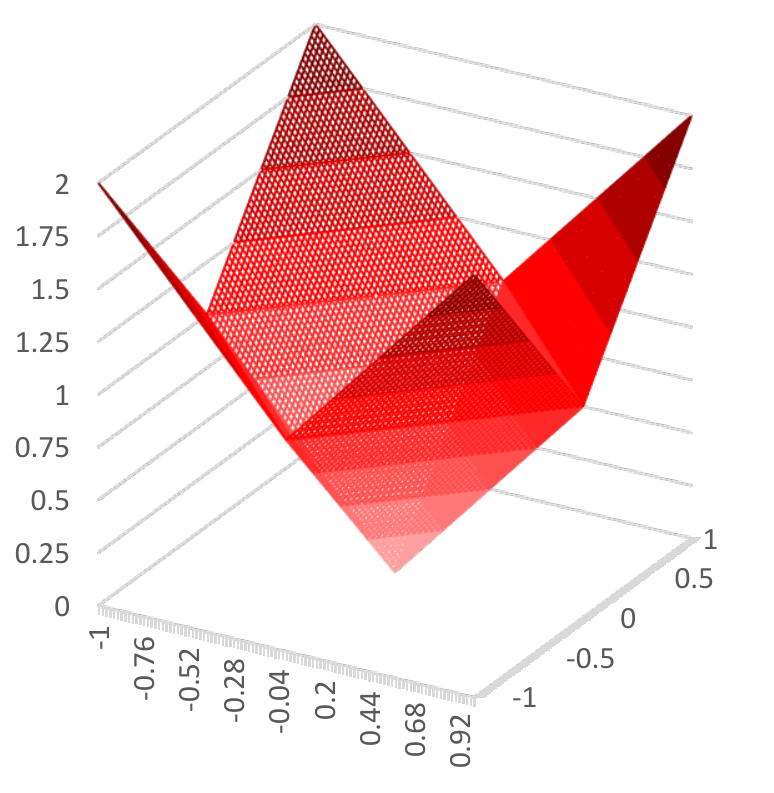}}\hfill\newline \vskip 5mm \subfloat[\MC, $\delta=1.0$]{\includegraphics[width=0.33\textwidth,trim={13cm 5.5cm 13cm 5.5cm},clip]{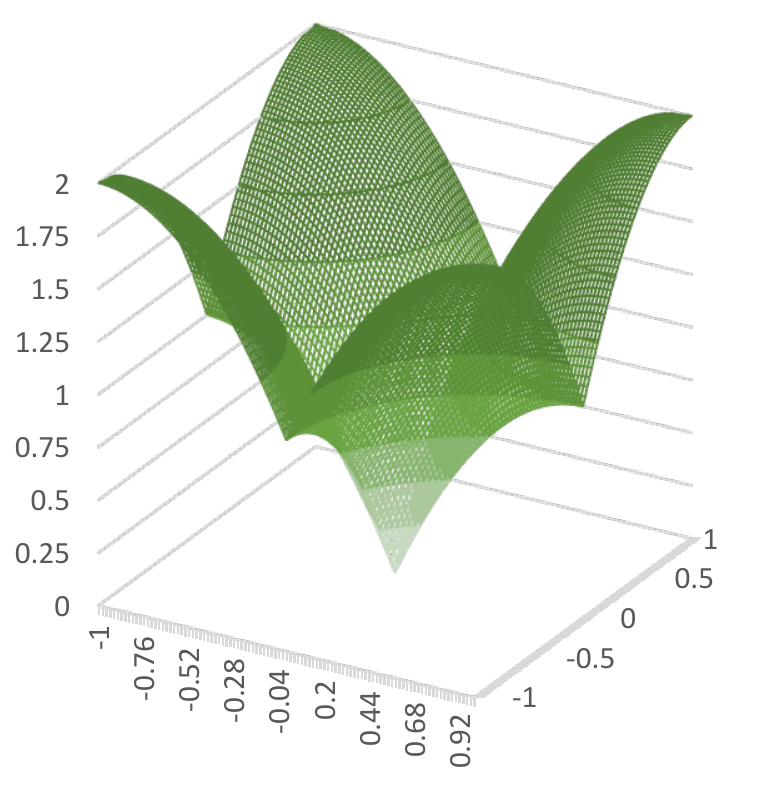}}\hfill\subfloat[\MC, $\delta=0.3$]{\includegraphics[width=0.33\textwidth,trim={13cm 5.5cm 13cm 5.5cm},clip]{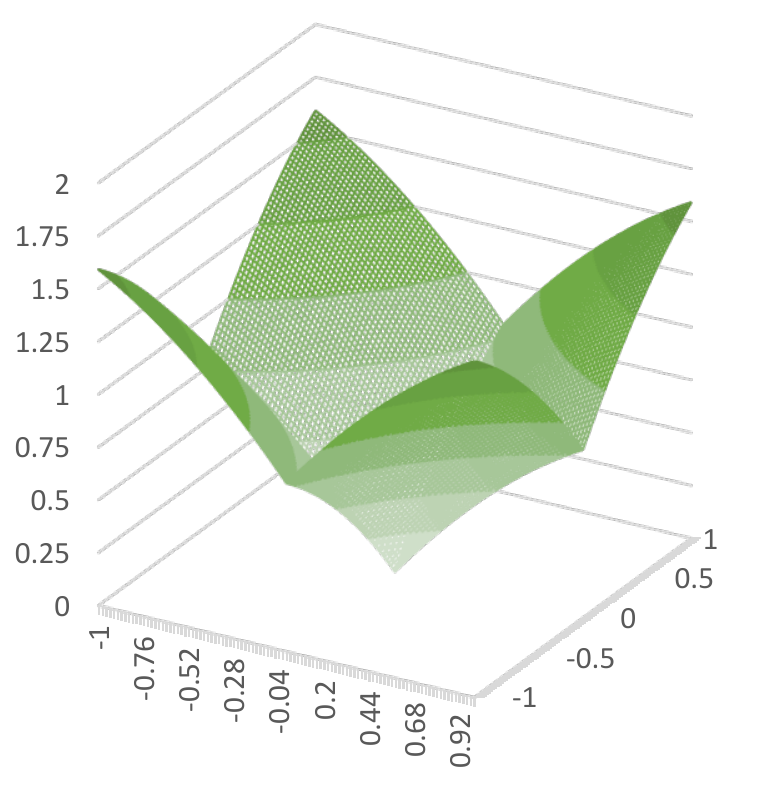}}\hfill\subfloat[\MC, $\delta=0.1$]{\includegraphics[width=0.33\textwidth,trim={13cm 5.5cm 13cm 5.5cm},clip]{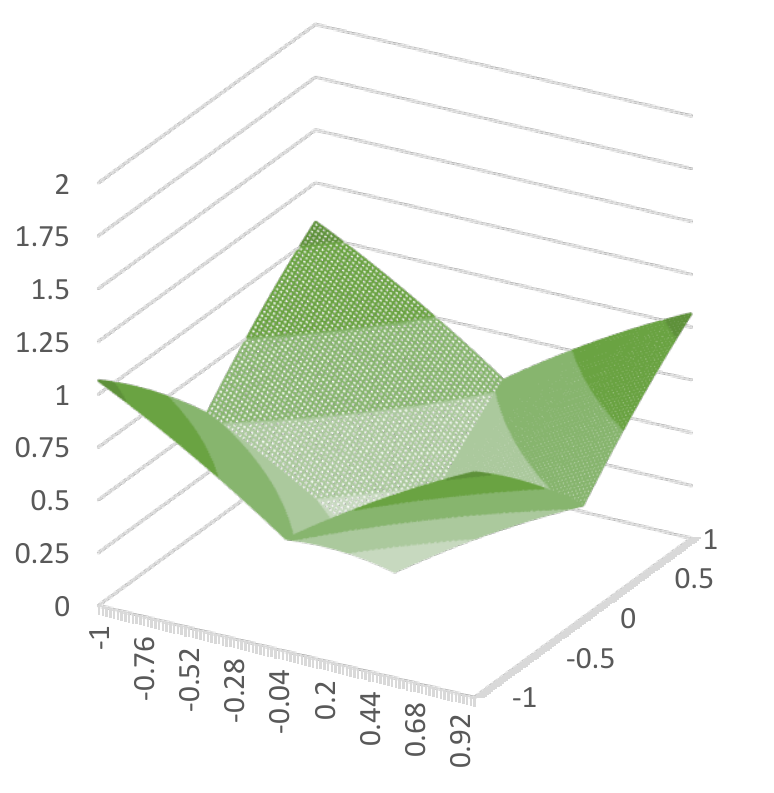}}\hfill\newline \vskip 5mm
	\subfloat[\RO, $\delta=1.0$]{\includegraphics[width=0.33\textwidth,trim={13cm 5.5cm 13cm 5.5cm},clip]{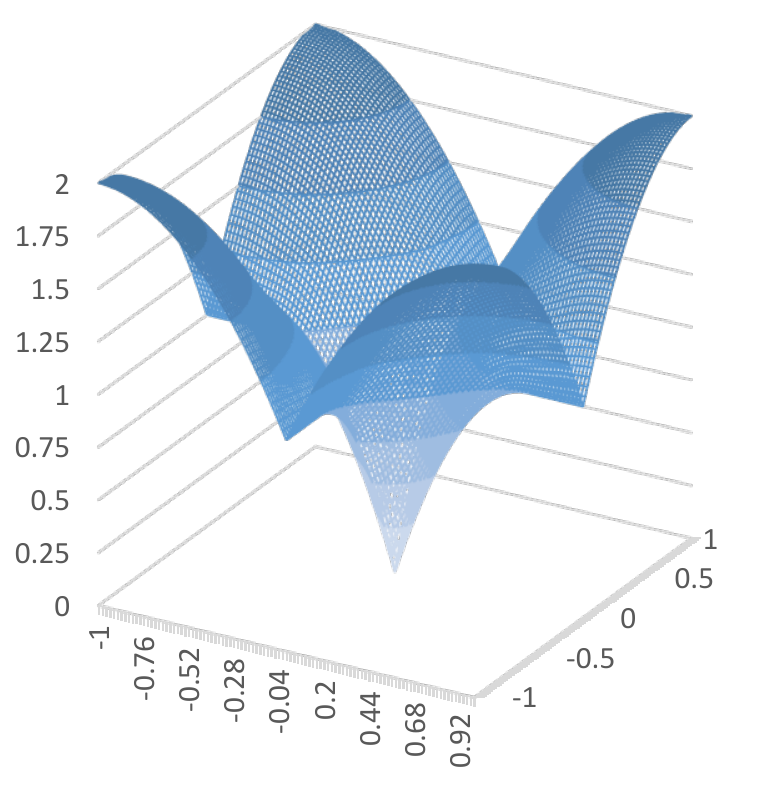}}\hfill\subfloat[\RO, $\delta=0.3$]{\includegraphics[width=0.33\textwidth,trim={13cm 5.5cm 13cm 5.5cm},clip]{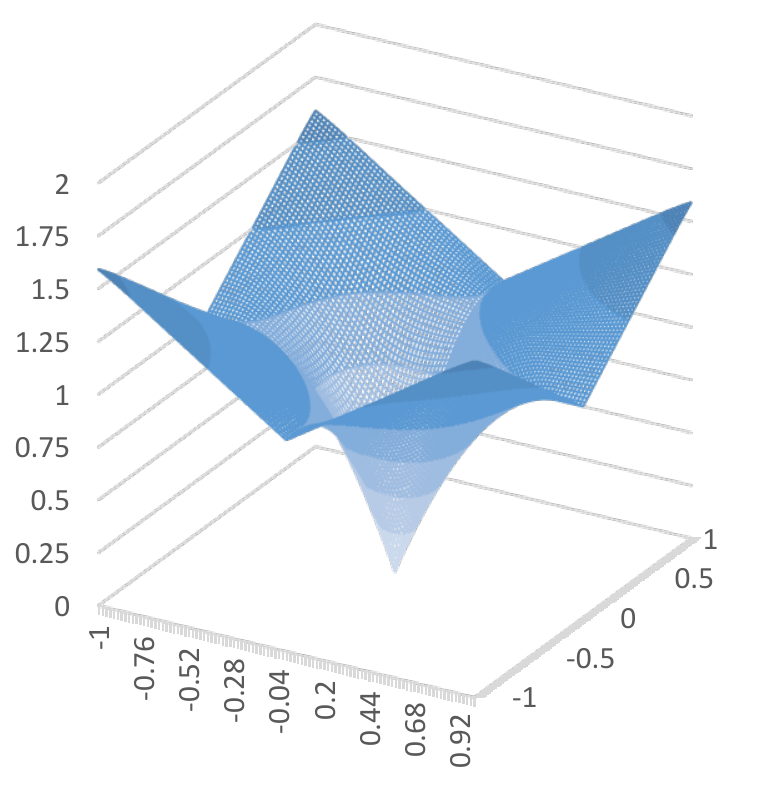}}\hfill\subfloat[\RO, $\delta=0.1$]{\includegraphics[width=0.33\textwidth,trim={13cm 5.5cm 13cm 5.5cm},clip]{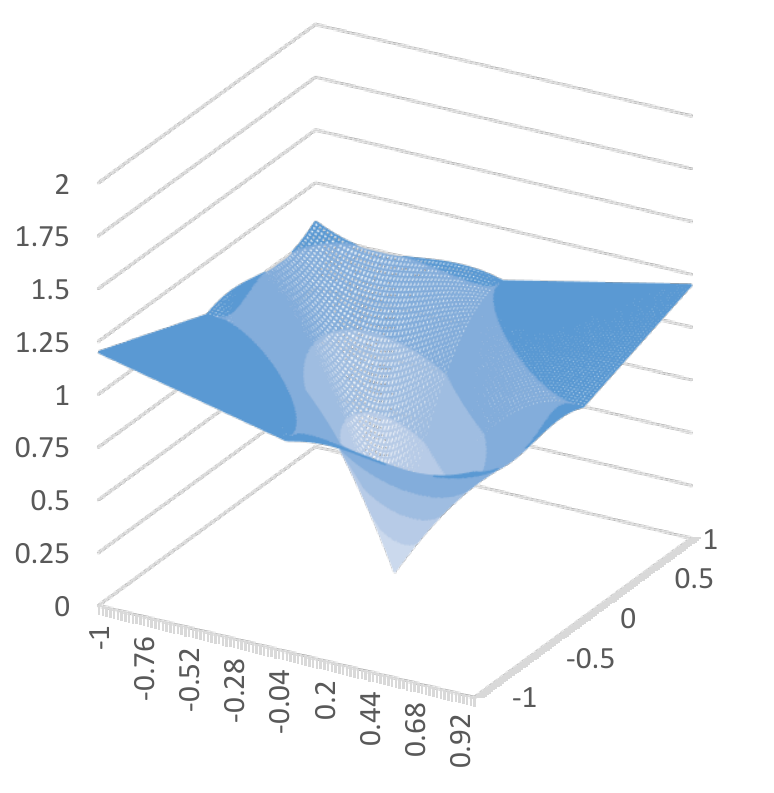}}\hfill
	\caption{\small Graphs of regularization penalties with $p=2$. \rev{The horizontal axes correspond to values of $\beta_1$ and $\beta_2$, and the vertical axis corresponds to the regularization penalty.} The \texttt{ridge}, \texttt{elastic net}, and \texttt{lasso} (top row) regularizations do not depend on the diagonal dominance, but induce substantial bias. The \MC\ regularization (second row) does not induce as much bias, but it depends on the diagonal dominance $(\delta)$. The new non-separable, non-convex \RO\ regularization (bottom row) induces larger penalties than \MC\ for all diagonal dominance values and is a closer approximation for the exact $\ell_0$ penalty.}
	\label{fig:regularization}
\end{figure}

\begin{figure}[!hp]
	\centering
	\subfloat[\footnotesize $\delta=1.0$, $\left.\|\bs{\beta}\|_1\leq 0.60\right.$]{\includegraphics[width=0.33\textwidth,trim={13cm 5.6cm 13cm 5.6cm},clip]{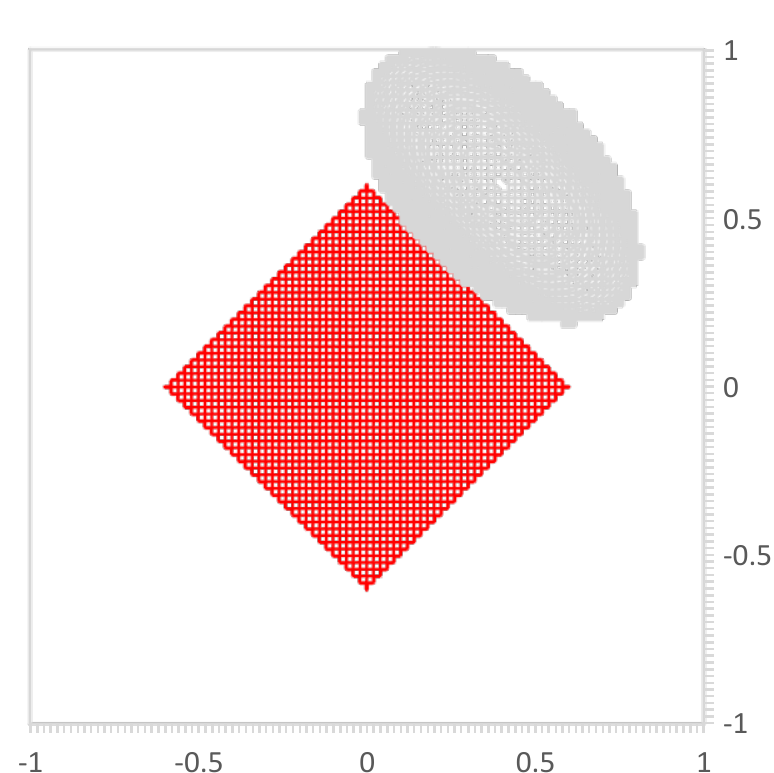}}\hfill\subfloat[\footnotesize $\delta=0.3$, $\left.\|\bs{\beta}\|_1\leq 0.84\right.$]{\includegraphics[width=0.33\textwidth,trim={13cm 5.6cm 13cm 5.6cm},clip]{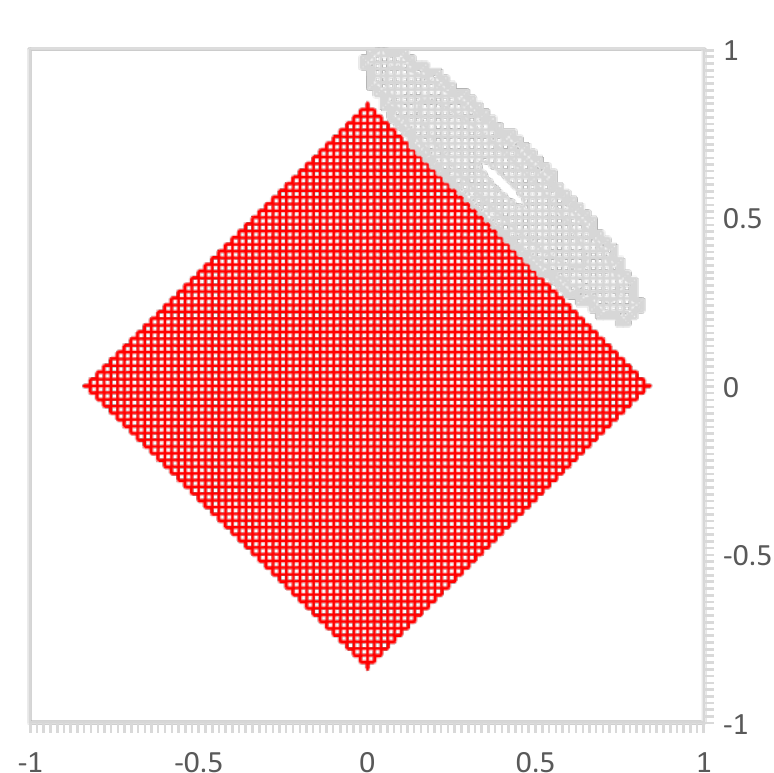}}\hfill\subfloat[\footnotesize $\delta=0.1$, $\left.\|\bs{\beta}\|_1\leq 0.96\right.$]{\includegraphics[width=0.33\textwidth,trim={13cm 5.6cm 13cm 5.6cm},clip]{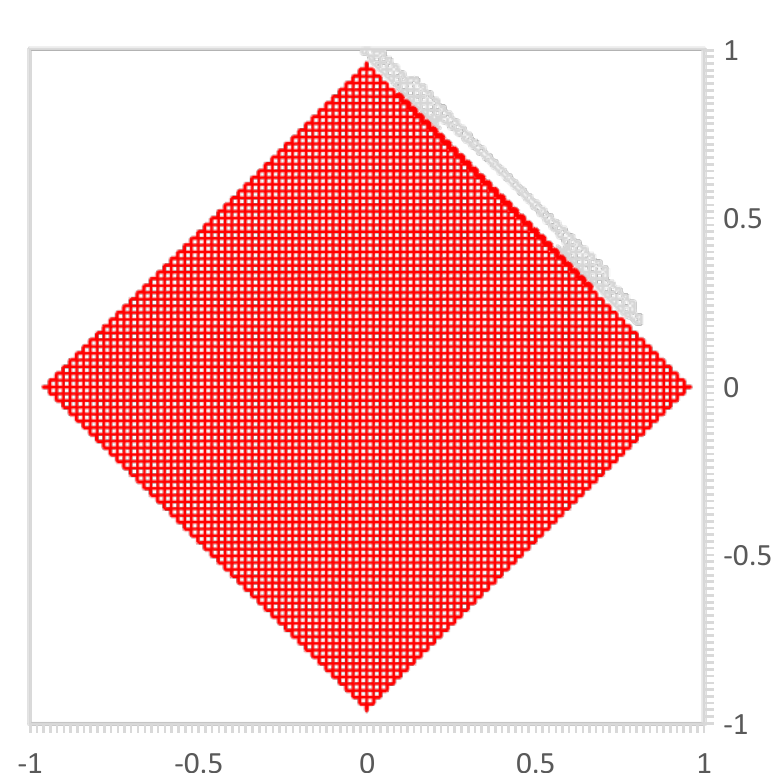}}\hfill\newline
	\subfloat[\footnotesize $\delta=1.0$, $\left.\rho_{\text{\MC}}(\bs{\beta})\leq 0.95\right.$, $\left.\rho_{\text{\RO}}(\bs{\beta})\leq 0.95\right.$]{\includegraphics[width=0.33\textwidth,trim={13cm 5.6cm 13cm 5.6cm},clip]{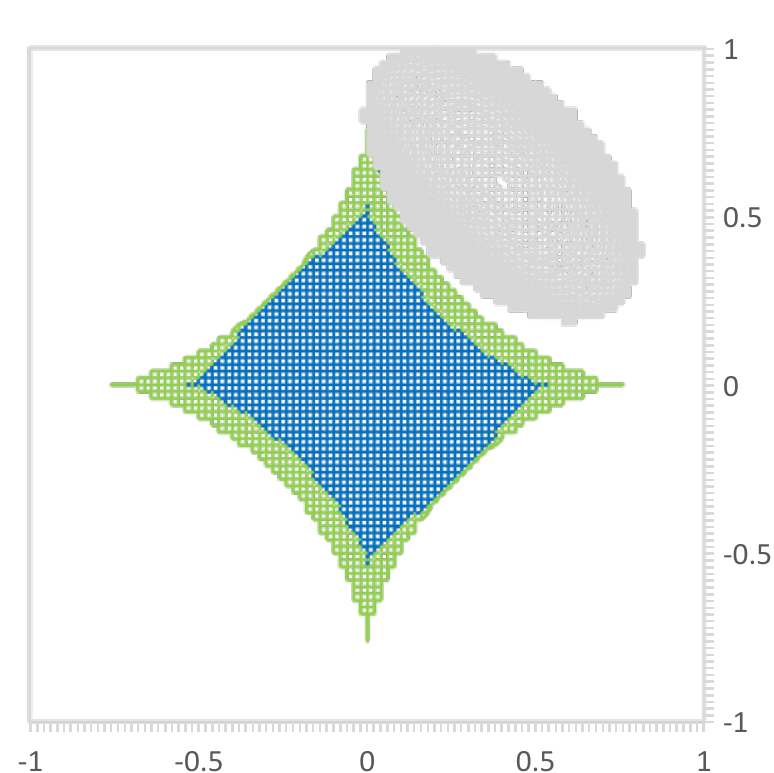}}\hfill\subfloat[\footnotesize $\delta=0.3$, $\left.\rho_{\text{\MC}}(\bs{\beta})\leq 0.77\right.$, $\left.\rho_{\text{\RO}}(\bs{\beta})\leq 0.77\right.$]{\includegraphics[width=0.33\textwidth,trim={13cm 5.6cm 13cm 5.6cm},clip]{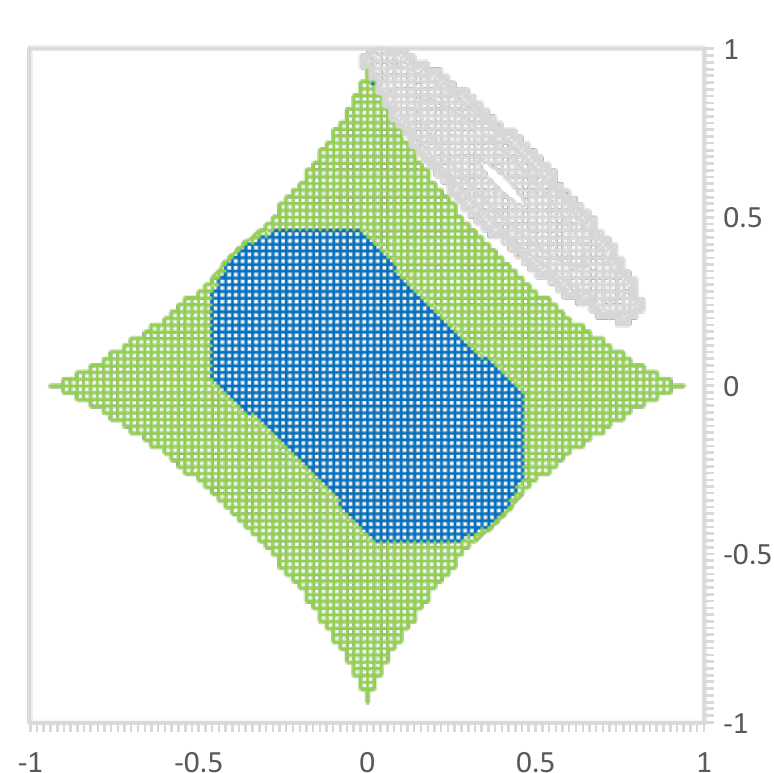}}\hfill\subfloat[\footnotesize $\delta=0.1$, $\left.\rho_{\text{\MC}}(\bs{\beta})\leq 0.53\right.$, $\left.\rho_{\text{\RO}}(\bs{\beta})\leq 0.53\right.$]{\includegraphics[width=0.33\textwidth,trim={13cm 5.6cm 13cm 5.6cm},clip]{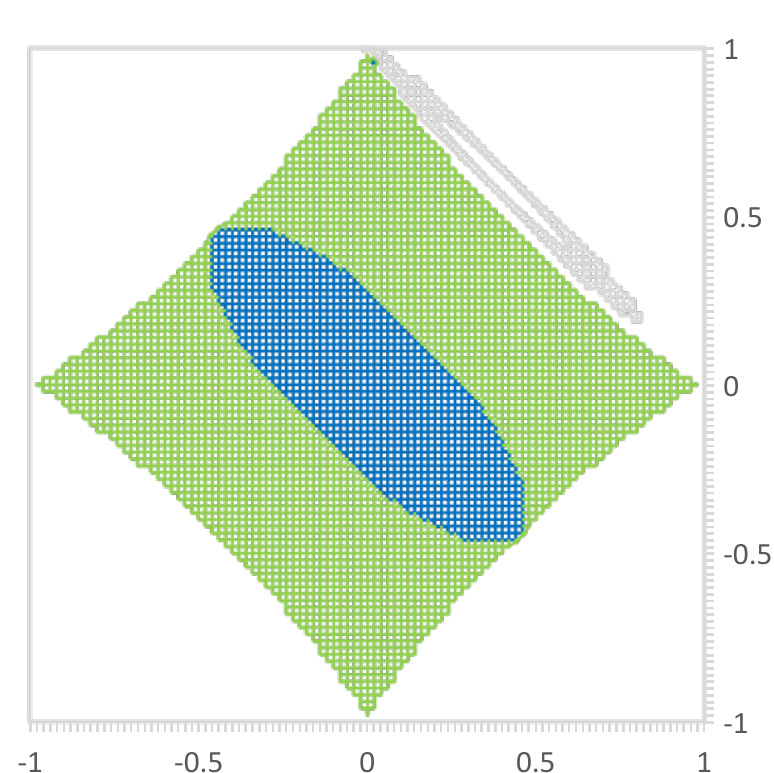}}\hfill\newline
	\subfloat[\footnotesize $\delta=1.0$, $\left.\rho_{\text{\RO}}(\bs{\beta})\leq 1.00\right.$]{\includegraphics[width=0.33\textwidth,trim={13cm 5.6cm 13cm 5.6cm},clip]{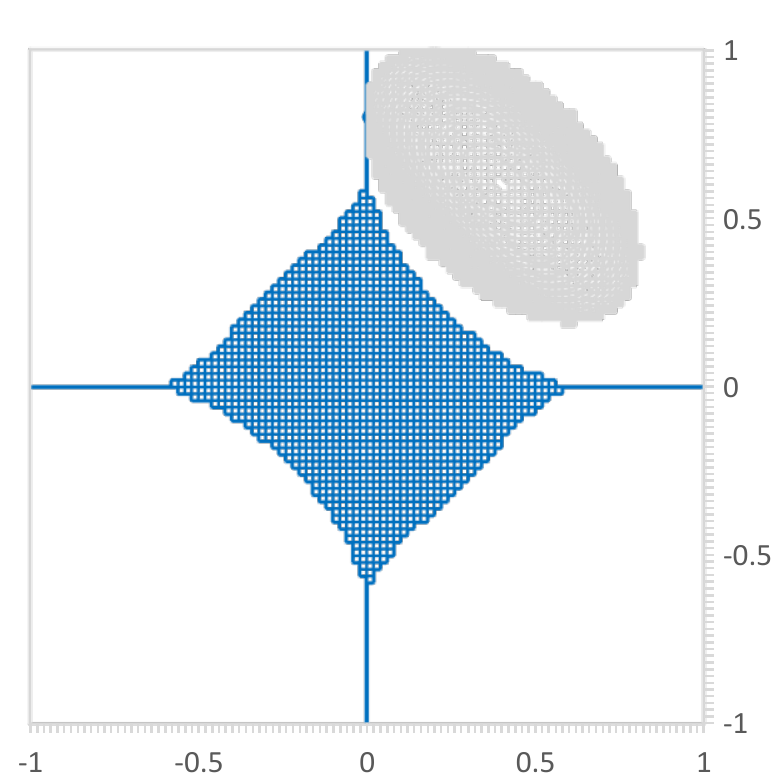}}\hfill\subfloat[\footnotesize $\delta=0.3$, $\left.\rho_{\text{\RO}}(\bs{\beta})\leq 1.00\right.$]{\includegraphics[width=0.33\textwidth,trim={13cm 5.6cm 13cm 5.6cm},clip]{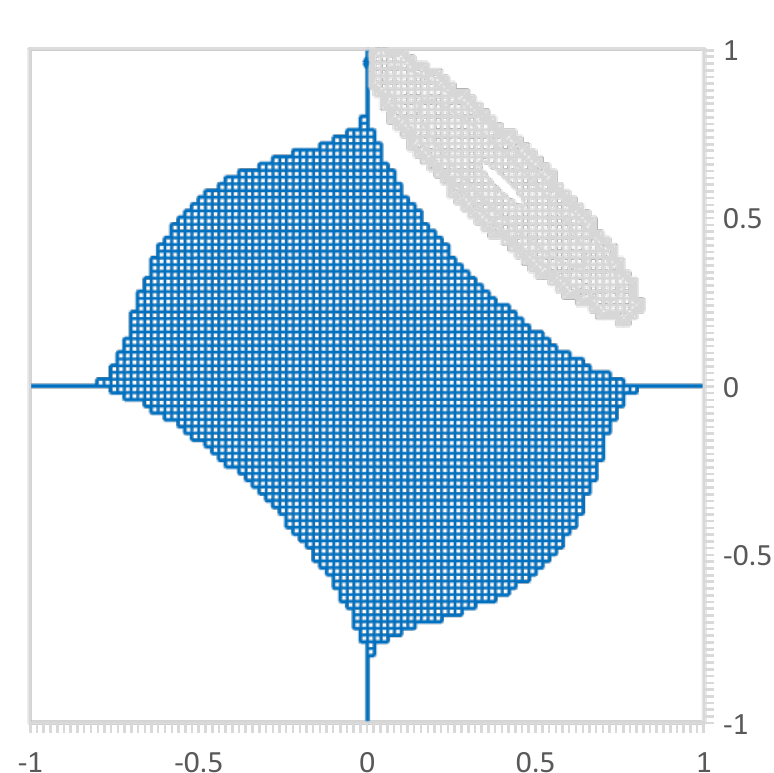}}\hfill\subfloat[\footnotesize $\delta=0.1$, $\left.\rho_{\text{\RO}}(\bs{\beta})\leq 1.00\right.$]{\includegraphics[width=0.33\textwidth,trim={13cm 5.6cm 13cm 5.6cm},clip]{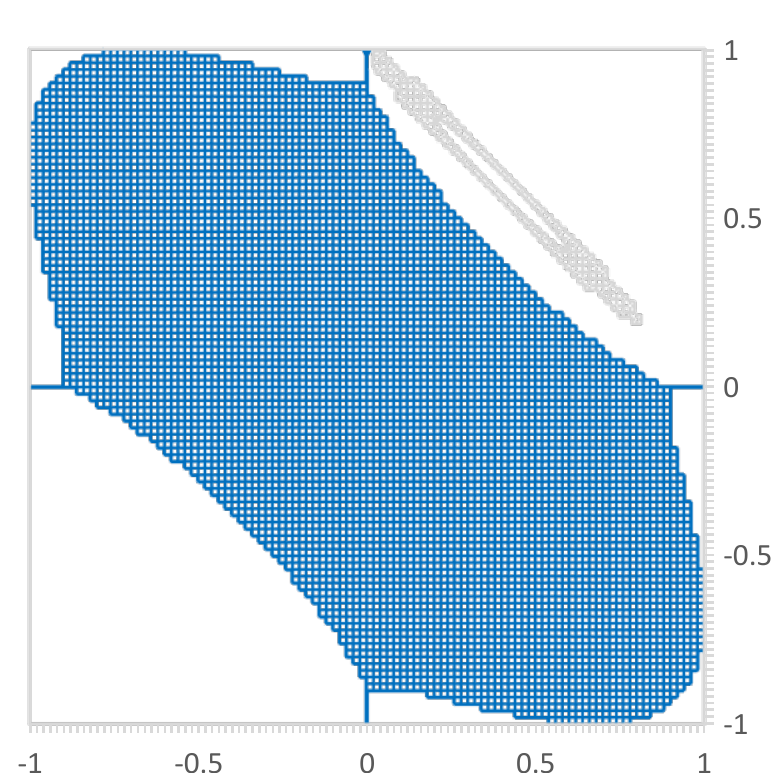}}\hfill
	\caption{\small  The axes correspond to the sparse solutions satisfying $\|\bs{\beta}\|_0\leq 1$.
		In gray: level sets given by $\|\bs{y}-\bs{X\beta}\|_2^2\leq \varepsilon^*$; in red: feasible region for $\|\bs{\beta}\|_1\leq k$; in green: feasible region for $\rho_{\text{\MC}}(\bs{\beta})\leq k$; in blue: feasible region for $\rho_{\text{\RO}}(\bs{\beta})\leq k$. All \texttt{lasso} and \MC \ solutions above are dense even with significant shrinkage ($k < 1$). Rank-one constraint attains sparse solutions on the axes with no shrinkage ($k=1$) for all diagonal dominance values $\delta$. 
	}
	\label{fig:constrained}
\end{figure}

\ignore{ 
\begin{figure}[!hp]
	\centering
	\subfloat[\footnotesize $\delta=1.0$, $\left.\|\bs{\beta}\|_1\leq 0.60\right.$, $\left.\bs{\hat\beta}=(0.20,0.40)\right.$]{\includegraphics[width=0.33\textwidth,trim={13cm 5.6cm 13cm 5.6cm},clip]{./images/Lassod10C.pdf}}\hfill\subfloat[\footnotesize $\delta=0.3$, $\left.\|\bs{\beta}\|_1\leq 0.84\right.$, $\left.\bs{\hat\beta}=(0.32,0.52)\right.$]{\includegraphics[width=0.33\textwidth,trim={13cm 5.6cm 13cm 5.6cm},clip]{./images/Lassod05C.pdf}}\hfill\subfloat[\footnotesize $\delta=0.1$, $\left.\|\bs{\beta}\|_1\leq 0.96\right.$, $\left.\bs{\hat\beta}=(0.38,0.56)\right.$]{\includegraphics[width=0.33\textwidth,trim={13cm 5.6cm 13cm 5.6cm},clip]{./images/Lassod01C.pdf}}\hfill\newline
\subfloat[\footnotesize $\delta=1.0$, $\left.\rho_{\text{\MC}}(\bs{\beta})\leq 0.95\right.$, $\left.\rho_{\text{\RO}}(\bs{\beta})\leq 0.95\right.$, $\left.\bs{\hat\beta}=(0.04,0.64)\right.$]{\includegraphics[width=0.33\textwidth,trim={13cm 5.6cm 13cm 5.6cm},clip]{./images/MCd10C.pdf}}\hfill\subfloat[\footnotesize $\delta=0.3$, $\left.\rho_{\text{\MC}}(\bs{\beta})\leq 0.77\right.$, $\left.\rho_{\text{\RO}}(\bs{\beta})\leq 0.77\right.$, $\left.\bs{\hat\beta}=(0.02,0.90)\right.$]{\includegraphics[width=0.33\textwidth,trim={13cm 5.6cm 13cm 5.6cm},clip]{./images/MCd05C.pdf}}\hfill\subfloat[\footnotesize $\delta=0.1$, $\left.\rho_{\text{\MC}}(\bs{\beta})\leq 0.53\right.$, $\left.\rho_{\text{\RO}}(\bs{\beta})\leq 0.53\right.$, $\left.\bs{\hat\beta}=(0.02,0.96)\right.$]{\includegraphics[width=0.33\textwidth,trim={13cm 5.6cm 13cm 5.6cm},clip]{./images/MCd01C.pdf}}\hfill\newline
	\subfloat[\footnotesize $\delta=1.0$, $\left.\rho_{\text{\RO}}(\bs{\beta})\leq 1.00\right.$, $\left.\bs{\hat\beta}=(0.00,0.80)\right.$]{\includegraphics[width=0.33\textwidth,trim={13cm 5.6cm 13cm 5.6cm},clip]{./images/R1d10C.pdf}}\hfill\subfloat[\footnotesize $\delta=0.3$, $\left.\rho_{\text{\RO}}(\bs{\beta})\leq 1.00\right.$, $\left.\bs{\hat\beta}=(0.00,0.96)\right.$]{\includegraphics[width=0.33\textwidth,trim={13cm 5.6cm 13cm 5.6cm},clip]{./images/R1d05C.pdf}}\hfill\subfloat[\footnotesize $\delta=0.1$, $\left.\rho_{\text{\RO}}(\bs{\beta})\leq 1.00\right.$, $\left.\bs{\hat\beta}=(0.00,1.00)\right.$]{\includegraphics[width=0.33\textwidth,trim={13cm 5.6cm 13cm 5.6cm},clip]{./images/R1d01C.pdf}}\hfill
	\caption{\small  The axes correspond to the sparse solutions satisfying $\|\bs{\beta}\|_0\leq 1$.
		In gray: level sets given by $\|\bs{y}-\bs{X\beta}\|_2^2\leq \varepsilon^*$; in red: feasible region for $\|\bs{\beta}\|_1\leq k$; in green: feasible region for $\rho_{\text{\MC}}(\bs{\beta})\leq k$; in blue: feasible region for $\rho_{\text{\RO}}(\bs{\beta})\leq k$; $\bs{\hat\beta}$: resulting estimator.$\left.\bs{\hat\beta}=(0.32,0.52)\right.$ \texttt{Lasso} and \MC \ solutions are dense even with significant shrinkage ($k < 1$). Rank-one constraint attains sparse solutions on the axes with no shrinkage ($k=1$) for all diagonal dominance values $\delta$. 
	}
	\label{fig:constrained}
\end{figure}
}

\rev{Finally, Figure~\ref{fig:formulations} shows the strength of relaxations of \eqref{eq:bestSubsetSelection} discussed in this paper. The ``big-$M$" relaxation is the natural convex relaxation of \eqref{eq:MIO} obtained by replacing $z \in \{0,1\}^p$ by $z \in [0,1]^p$, used in \cite{bertsimas2016best,cozad2014learning}. The perspective relaxation is the natural convex relaxation of \eqref{eq:MIOPersp}, which is the basis of recent methods \cite{bertsimas2017sparse,hazimeh2020sparse,pilanci2015sparse,xie2018ccp} -- note that this formulation may only be used if $\lambda > 0$. 
	The ``optimal perspective" relaxation, also referred to as \sdp{1} in this paper, was explicitly given in \cite{dong2015regularization}. This paper proposes new relaxations \sdp{r}, discussed in \S\ref{sec:convexification}, which dominate all existing relaxations in terms of strength. It also proposes the new formulation \sdp{LB}, discussed in \S\ref{sec:socp}, which is easier to solve than \sdp{r} but still compares favorably with the ``big-$M$" and perspective formulations.}

\begin{figure}[!h]
	\centering
	{\includegraphics[width=0.8\textwidth,trim={6cm 3.5cm 6cm 1cm},clip]{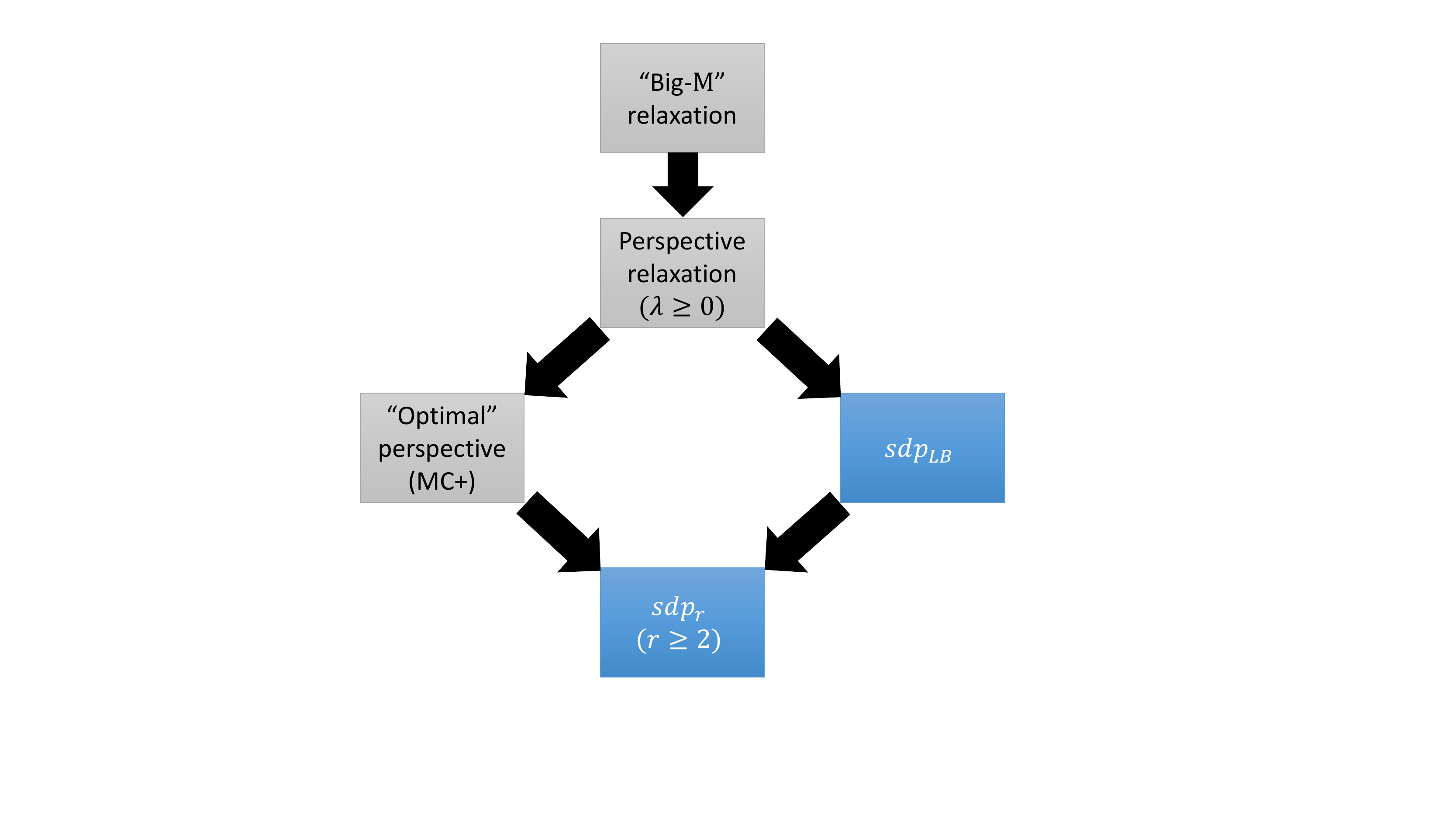}}
	\caption{\small \rev{Strength of relaxations discussed in the paper.  ``$A\Rightarrow B$" indicates that $B$ is a stronger relaxation than $A$, i.e., is a better approximation for the non-convex problem \eqref{eq:bestSubsetSelection}. Blue boxes correspond to the new formulations proposed in this paper.}
	}
	\label{fig:formulations}
\end{figure}

\subsection*{Outline} The rest of the paper is organized as follows. In \S\ref{sec:convexification} we derive the proposed convex relaxations based on ideal formulations for rank-one quadratic terms with indicator variables. We also give an interpretation of the convex relaxations as unbiased regularization penalties, and we give an explicit semidefinite optimization (SDP) formulation in an extended space, which can be implemented with off-the-shelf conic optimization solvers. In \S\ref{sec:regularization} we derive an explicit form of the regularization penalty for the two-dimensional case. In \S\ref{sec:socp} we discuss the implementation of the proposed relaxation in a conic quadratic framework. In \S\ref{sec:computations} we present computational experiments with synthetic as well as benchmark datasets, demonstrating that \textit{(i)} the proposed formulation delivers near-optimal solutions (with provable optimality gaps) of \eqref{eq:bestSubsetSelection} in most cases, \textit{(ii)} using the proposed convex relaxation results in superior statistical performance when compared with usual estimators obtained from convex optimization approaches. In \S\ref{sec:conclusions} we conclude the paper with a few final remarks. 

\subsection*{Notation} Define $P=\{1,\ldots,p\}$ and $\bs{e}\in \R^\rev{p}$ be the vector of ones. Given $T\subseteq P$ and a vector $\bs{a}\in \R^\rev{p}$, define $\bs{a_T}$ as the subvector of $\bs{a}$ induced by $T$, $a_i=\bs{a_{\{i\}}}$ as the $i$-th element of $\bs{a}$, and define $a(T)=\sum_{i\in T}a_i$. 
Given a symmetric matrix $\bs{A}\in \R^{\rev{p}\times \rev{p}}$, let $\bs{A_T}$ be the submatrix of $\bs{A}$ induced by $T\subseteq P$, and let $\Sp^T$ be the set of \rev{$|T|\times |T|$} symmetric positive semidefinite matrices, i.e., $\bs{A_T}\succeq 0 \Leftrightarrow \bs{A_T}\in \Sp^T$. We use $\bs{a_T}$ or $\bs{A_T}$ to make explicit that a given vector or matrix 
is indexed by the elements of $T$ or $T\times T$, respectively. 
Given matrices $\bs{A}$, $\bs{B}$ of the same dimension, $\bs{A}\circ\bs{B}$ denotes the Hadamard product of $\bs{A}$ and $\bs{B}$, and $\langle \bs{A},\bs{B}\rangle $ denotes their inner product.
Given a vector $\bs{a}\in \R^n$, let $\diag(\bs{a})$ be the $n\times n$ diagonal matrix $\bs{A}$ with $A_{ii}=a_i$. 
For a set $X\subseteq \R^\rev{p}$, $\clconv(X)$ denotes the closure of the convex hull of $X$. Throughout the paper, we adopt the following convention for division by 0: given a scalar $s\geq 0$, $s/0=\infty$ if $s>0$ and $s/0$ if $s=0$. For a scalar $a\in \R$, let $\sign(a)=a/|a|$.

\section{Convexification}\label{sec:convexification}
In this section we introduce the proposed relaxations of problem \eqref{eq:bestSubsetSelection}. First, in \S\ref{sec:rank1}, we describe the \emph{ideal} relaxations for the mixed-integer epigraph of a rank-one quadratic term. Then, in \S\ref{sec:general}, we use the relaxations derived in \S\ref{sec:rank1} to give strong relaxations of \eqref{eq:bestSubsetSelection}. Next, in \S\ref{sec:regularization2}, we give an interpretation of the proposed relaxations as unbiased sparsity-inducing regularizations. Finally, in \S\ref{sec:exactSDP} we present an explicit SDP representation of the proposed relaxations in an extended space.

\subsection{Rank-one case}\label{sec:rank1}
We first give a valid inequality for the mixed-integer epigraph of a convex quadratic function defined over the subsets of $P$.
Given $A_T\in \Sp^T$, consider the set
$$Q_{T}=\left\{(\bs{z},\bs{\beta},t)\in \{0,1\}^\rev{|T|}\times \R^\rev{|T|}\times \R_+: \bs{\beta^\top A_T\beta}\leq t,\; \beta_i(1-z_i)=0,\forall i\in T\right\}.$$

\begin{proposition}
	The inequality 
	\begin{equation}\label{eq:valid}
	\frac{\bs{\beta^\top A_T\beta}}{z(T)}\leq t
	\end{equation}
	is valid for $Q_T$.
\end{proposition}
\begin{proof}
Let $(\bs{z},\bs{\beta},t)\in Q_T$, and we verify that inequality \eqref{eq:valid} is satisfied. First observe that if $\bs{z}=\bs{0}$, then $\bs{\beta}=\bs{0}$ and inequality \eqref{eq:valid} reduces to $0\leq t$, which is satisfied. Otherwise, if $z_i=1$ for some $i\in T$, then $z(T)\geq 1$ and we find that $\frac{\bs{\beta^\top A_T\beta}}{z(T)}\leq \bs{\beta^\top A_T\beta}\leq t$, and inequality \eqref{eq:valid} is satisfied again. 
\end{proof}
Observe that if $T$ is a singleton, i.e., $T=\{i\}$, then \eqref{eq:valid} reduces to the well-known perspective inequality $A_{ii}\beta_i^2\leq tz_i$. Moreover, if $T=\{i,j\}$ and $\bs{A_T}$ is rank-one, i.e.,
\rev{$\bs{A_T}=a_Ta_T^\top$ with $a_T=(a_i\; a_j)^\top$ and} $\bs{\beta^\top A_T\beta}=|A_{ij}|\left(a\beta_i^2\pm 2\beta_i\beta_j+(1/a)\beta_j^2\right)$ for $A_{ij}\rev{=a_ia_j}$ and $a\rev{=a_i/a_j}$, then \eqref{eq:valid} reduces to \begin{equation}\label{eq:2drank1}|A_{ij}|\left(a\beta_i^2\pm 2\beta_i\beta_j+(1/a)\beta_j^2\right)\leq t(z_i+z_j),\end{equation}
one of the inequalities proposed in \cite{jeon2017quadratic} in the context of quadratic optimization with indicators and bounded continuous variables. Note that inequality \eqref{eq:2drank1} is, in general, weak for bounded continuous variables (as non-negativity or other bounds can be used to strengthen the inequalities, see \cite{atamturk2018strong} for additional discussion); and inequality \eqref{eq:valid} is, in general, weak for arbitrary matrices $\bs{A_T}\in \Sp^T$. Nonetheless, as we show next, inequality \eqref{eq:valid} is sufficient to describe the \emph{ideal} (convex hull) description for $Q_T$ if $\bs{A_T}\rev{=\bs{a_Ta_T^\top}}$ is a rank-one matrix. Consider the
special case of $Q_T$ defined with a rank-one matrix:
$$Q_{T}^{r1}=\left\{(\bs{z},\bs{\beta},t)\in \{0,1\}^\rev{|T|}\times \R^\rev{|T|}\times \R_+: (\bs{a_T^\top\beta})^2\leq t,\; \beta_i(1-z_i)=0,\forall i\in T\right\}.$$

\begin{theorem}
If $a_i\neq 0$ for all $i\in T$, then 
$$\conv(Q_{T}^{r1})=\left\{(\bs{z},\bs{\beta},t)\in [0,1]^\rev{|T|}\times \R^\rev{|T|}\times \R_+: (\bs{a_T^\top \beta})^2\leq t,\; \frac{(\bs{a_T^\top\beta})^2}{z(T)}\leq t\right\} \cdot$$
\end{theorem}
\begin{proof}
Consider the optimization of an arbitrary linear function over $Q_{T}^{r1}$ and $\bar Q_T:=\Big\{(\bs{z},\bs{\beta},t)\in [0,1]^\rev{|T|}\times \R^\rev{|T|}\times \R_+: (\bs{a_T^\top\beta})^2\leq t,\; \frac{(\bs{a_T^\top\beta})^2}{z(T)}\leq t\Big\}$:
\begin{align}
&\min_{(\bs{z},\bs{\beta},t)\in Q_T^{r1}}\bs{u_T^\top z}+\bs{v_T^\top\beta}+\kappa t, \label{eq:discrete}\\
&\min_{(\bs{z},\bs{\beta},t)\in \bar Q_T}\bs{u_T^\top z}+\bs{v_T^\top\beta}+\kappa t,\label{eq:continuous}
\end{align}
where $\bs{u_T},\bs{v_T}\in \R^{\rev{|T|}}$ and $\kappa\in \R$.
We now show that either there exists an optimal solution of \eqref{eq:continuous} that is feasible for \eqref{eq:discrete}, hence also optimal for \eqref{eq:discrete} as $\bar Q_T$ is a relaxation of $Q_T^{r1}$, or that \eqref{eq:discrete} and \eqref{eq:continuous} are both unbounded.

Observe that if $\kappa<0$, then letting $\bs{z}=\bs{\beta}=\bs{0}$ and $t\to\infty$ we see that both problems are unbounded. If $\kappa=0$ and $\bs{v_T}=\bs{0}$, then \eqref{eq:continuous} reduces to $\min_{\bs{z}\in [0,1]^\rev{|T|}}\bs{u_T^\top z}$, which has an optimal integral solution $\bs{z}^*$, and $(\bs{z}^*,\bs{0},0)$ is optimal for \eqref{eq:discrete} and \eqref{eq:continuous}. If $\kappa=0$ and $v_i\neq 0$ for some $i\in T$, then letting $\beta_i\to \pm \infty$, $z_i=1$, and $\beta_j=z_j=t=0$ for $j\neq i$, we find that both problems are unbounded. Thus, we may assume, without loss of generality that $\kappa>0$, and, by scaling, $\kappa=1$. 

Additionally, as $\bs{a_T}$ has no zero entry, we may assume, without loss of generality, that $\bs{a_T}=\bs{e_T}$, since otherwise $\bs{\beta}$ and $\bs{v_T}$ can be scaled by letting $\bar \beta_i=a_i\beta_i$ and $\bar v_i =v_i/a_i$ to arrive at an equivalent problem.
 Moreover, a necessary condition for \eqref{eq:discrete}--\eqref{eq:continuous} to be bounded is that 
\begin{equation}\label{eq:bounded}-\infty <\min_{\bs{\beta}\in \R^\rev{|T|}} \bs{v_T^\top\beta} \text{ s.t. } \beta(T)=\zeta\end{equation}
for any fixed $\zeta \in \R$.
It is easily seen that \eqref{eq:bounded} has an optimal solution if and only if $v_i=v_j$ for all $i\neq j$. Thus, we may also assume without loss of generality that $\bs{v_T^\top\beta}=v_0\beta(T)$  for some scalar $v_0$. 
Performing the above simplifications, we find that \eqref{eq:continuous} reduces to 
\begin{equation}\label{eq:continuous2}
\min_{\bs{z}\in [0,1]^\rev{|T|},\bs{\beta}\in \R^{\rev{|T|}},t\in \R}\bs{u_T^\top z}+v_0\beta(T)+t \text{ s.t. } \beta(T)^2\leq t,\; \beta(T)^2\leq tz(T).
\end{equation}
Since the one-dimensional optimization $\min_{\beta\in \R}\left\{v_0\beta+\beta^2\right\}$ has an optimal solution, it follows that \eqref{eq:continuous2} is bounded and has an optimal solution. We now prove that \eqref{eq:continuous2} has an optimal solution that is integral in $\bs{z}$ and satisfies $\bs{\beta}\circ(\bs{e}-\bs{z})=0$. 

Let $(\bs{z}^*,\bs{\beta}^*,t^*)$ be an optimal solution of \eqref{eq:continuous2}. First note that if $0<z^*(T)<1$, then $(\gamma \bs{z}^*,\gamma \bs{\beta}^*,\gamma t^*)$ is feasible for \eqref{eq:continuous} for $\gamma$ sufficiently close to $1$, with objective value $\gamma\left(\bs{u_T^\top z}^*+v_0\beta^*(T)+ t^*\right)$. If $\bs{u_T^\top z}^*+v_0\beta^*(T)+ t^* \ge 0$, then for $\gamma = 0$, $(\gamma \bs{z}^*,\gamma \bs{\beta}^*,\gamma t^*)$ has an objective value equal or lower. Otherwise, for $\gamma=1/z^*(T)$, $(\gamma \bs{z}^*,\gamma \bs{\beta}^*,\gamma t^*)$ is feasible and has a lower objective value.
Thus, we find that either $\bs{0}$ is optimal for \eqref{eq:continuous2} (and the proof is complete), or there exists an optimal solution with $z^*(T)\geq 1$. In the later case, observe that any $(\bar{ \bs{z}},\bs{\beta}^*,t^*)$ with $\bar{ \bs{z}}\in\argmin\{\bs{u_T^\top z}: z^*(T)\geq 1, z\in [0,1]^\rev{|T|}\}$ is also optimal for \eqref{eq:continuous2}, an in particular there exists an optimal solution with $\bar{\bs{z}}$ integral. 

Finally, let $i\in T$ be any index  with $\bar z_i=1$. Setting $\bar \beta_i = \beta^*(T)$ and $\bar \beta_j=0$ for $i\neq j$, 
we find another optimal solution $(\bar{\bs{z}},\bar{\bs{\beta}},t^*)$ for \eqref{eq:continuous2} that satisfies the complementary constraints, and thus is feasible and optimal for \eqref{eq:discrete}. 
\end{proof}

\begin{remark}
Observe that describing $\conv(Q_{T}^{r1})$ requires two nonlinear inequalities in the original space of variables. More compactly, we can specify $\conv(Q_{T}^{r1})$ using a single convex inequality, as 
$$\conv(Q_{T}^{r1})=\left\{(\bs{z},\bs{\beta},t)\in [0,1]^\rev{|T|}\times \R^\rev{|T|}\times \R_+: \frac{(\bs{a_T^\top\beta})^2}{\min\{1,z(T)\}}\leq t \right\} \cdot$$ Finally, we point out that $\conv(Q_{T}^{r1})$ is conic quadratic representable, as $(\bs{z},\bs{\beta},t)\in \conv(Q_{T}^{r1})$ if and only if there exists $w$ such that the system 
 $$\bs{z}\in [0,1]^\rev{|T|},\;\bs{\beta}\in \R^\rev{|T|},\; t\in \R_+,\;w\in \R_+,\; w\leq 1,\; w\leq z(T),\; (\bs{a_T^\top\beta})^2\leq tw$$
is feasible, where the last constraint is a rotated conic quadratic constraint and all other constraints are linear.\qed
\end{remark}

\subsection{General case}\label{sec:general}Now consider again the mixed-integer optimization \eqref{eq:MIO}
\begin{subequations}\label{eq:MIOEpi}
	\begin{align}
	\bs{y^\top y}+\min_{\rev{\bs{\beta},\bs{z},\bs{u}}}\;&-2\bs{y^\top X}\bs{\beta}+\mu\left(\bs{e^\top u}\right)+t\\
	\text{ s.t.}\;& \bs{\beta^\top}\left(\bs{X}^\top\bs{X}+\lambda \bs{I}\right)\bs{\beta}\leq t\label{eq:MIOEpi_epigraph}\\
	&\bs{e^\top z}\leq k\\
	& \bs{\beta}\leq \bs{u},\; -\bs{\beta}\leq \bs{u}\\
	&\bs{\beta}\circ (\bs{e}-\bs{z})=\bs{0}\label{eq:MIOEpi_complementary}\\
	&\bs{\beta}\in \R^\rev{p},\; \bs{z}\in \{0,1\}^\rev{p},\; \bs{u}\in \R_+^\rev{p},\; t\in \R
	\end{align}
\end{subequations}
where the nonlinear terms of the objective is moved to constraint \eqref{eq:MIOEpi_epigraph}. A direct application of \eqref{eq:valid} yields the inequality $\bs{\beta}\rev{^\top}\left(\bs{X}^\top\bs{X}+\lambda \bs{I}\right)\bs{\beta}\leq tz(P)$, which is weak and has no effect when $z(P)\geq 1$. Instead, a more effective approach is to decompose the matrix $\bs{X^\top X}+\lambda \bs{I}$ into a sum of low-dimensional rank-one matrices, and use inequality~\eqref{eq:valid} to strengthen each quadratic term in the decomposition separately, as illustrated in Example~\ref{ex:decomposition} bellow.

\begin{example}\label{ex:decomposition}
Consider the example with $p=3$ and $\bs{X^{\top}X}+\lambda \bs{I}=\begin{pmatrix} 25 &15& -5\\ 15 & 18 & 0\\ -5 & 0 & 11\end{pmatrix}.$ Then, it follows that 
\begin{align}
\bs{\beta}\rev{^\top}\left(\bs{X}^\top\bs{X}+\lambda \bs{I}\right)\bs{\beta}=&\left(5\beta_1+3\beta_2-\beta_3\right)^2+\left(3\beta_2+\beta_3\right)^2+9\beta_3^2 \notag
\end{align}
and we have the corresponding valid inequality
\begin{align}
 &\frac{\left(5\beta_1+3\beta_2-\beta_3\right)^2}{\min\{1,z_1+z_2+z_3\}}+\frac{\left(3\beta_2+\beta_3\right)^2}{\min\{1,z_2+z_3\}}+9\frac{\beta_3^2}{z_3}\leq t.\label{eq:decompExample}
\end{align}
\qed
\end{example}

The decomposition of $\bs{X^\top X}+\lambda \bs{I}$ illustrated in Example~\ref{ex:decomposition} is not unique. 
\rev{Since one does not obtain a strengthening when the denominator is one, it is important to have decomposition both rank-one and sparse.
This motivates the question on how to find a decomposition that results in the best convex relaxation,} i.e., that maximizes the left hand side of \eqref{eq:decompExample}. Specifically,
let $\Ps\subseteq 2^P$ be a subset of the power set of $P$, \rev{i.e., $$\Ps=\left\{\bs{T_1},\dots,\bs{T_m}\right\}$$ with $\bs{T_h}\subseteq P$, $h=1,\ldots,m$. For each $h$, define a matrix variable $\bs{A_h}$ whose nonzero elements correspond to the submatrix induced by $\bs{T_h}$, and }
consider the valid inequality $\phi_\Ps(\bs{z},\bs{\beta})\leq t$, where $\phi_\Ps: [0,1]^\rev{p}\times \R^\rev{p}\to \R$ is defined as 
\begin{subequations}\label{eq:optDecomp}
\begin{align}
\phi_\Ps(\bs{z},\bs{\beta}) := \max_{\bs{A_\rev{h}},\bs{R}}\;&\bs{\beta^\top R\beta}+\sum_{\rev{h=1}}^\rev{m}\frac{\bs{\beta^\top A_\rev{h}\beta}}{\min\{1,z(T_\rev{h})\}}\label{eq:optDecompObj}\\
\text{s.t.}\;&
\rev{\sum_{h=1}^m}\bs{A_\rev{h}} +\bs{R} = \bs{X}^\top\bs{X}+\lambda \bs{I}\label{eq:optDecomp_decomp}\\
&\rev{\left(A_h\right)_{ij}=0} &\hspace{-3cm}\rev{\forall h=1,\ldots,m,\; i\not\in T_h\text{ or }j\not\in T_j}\\
&\bs{A_\rev{h}}\in \Sp^\rev{P} \quad &\hspace{-3cm}\forall \rev{h=1,\ldots,m}\\
&\bs{R}\in \Sp^P,\label{eq:optDecomp_psd}
\end{align}
\end{subequations}
where strengthening  \eqref{eq:valid} is applied to each low-dimensional quadratic term $\bs{\beta^\top A_\rev{h}\beta}$.
For a fixed value of $(\bs{z},\bs{\beta})$, problem \eqref{eq:optDecomp} finds the best decomposition of the matrix $\bs{X}^\top\bs{X}+\lambda \bs{I}$ as a sum of positive semidefinite matrices $\bs{A_\rev{h}}$, $\rev{h=1,\ldots,m}$, and a remainder positive semidefinite matrix $\bs{R}$ to maximize the strengthening.  

For a given decomposition, the objective \eqref{eq:optDecompObj} is convex in $(\bs{z},\bs{\beta})$, thus $\phi_\Ps$ is a supremum of convex functions and is convex on its domain. Observe that the inclusion or omission of the empty set does not affect function $\phi_\Ps$, and we assume for simplicity that $\emptyset\in \Ps$. 

Since inequalities \eqref{eq:valid} are ideal for rank-one matrices, inequality $\phi_\Ps(\bs{z},\bs{\beta})\leq t$ is particularly strong if matrices $\bs{A_\rev{h}}$ are rank-one in optimal solutions of \eqref{eq:optDecomp}. As we now show, this is indeed the case if $\Ps$ is downward closed.

\begin{proposition}\label{prop:rank1Decomp}
If \Ps is downward closed, i.e., $V \in \Ps \implies U\in \Ps$ for all $U\subseteq V$, then there exists an optimal solution to \eqref{eq:optDecomp}  where all matrices $\bs{A_\rev{h}}$ are rank-one.
\end{proposition}

\begin{proof}
Let $T\in \Ps$, \rev{let $\bs{A}$ be the matrix variable associated with $T$,} and suppose $\bs{A_T}$ is not rank-one in an optimal solution to \eqref{eq:optDecomp}, also suppose for simplicity that $T=\{1,\ldots,p_0\}$ for some $p_0\leq p$, and let $\bar T_i=\{i,\ldots,p_0\}$ for $i=1,\ldots,p_0$. Since $\bs{A_T}$ is positive semidefinite, there exists a Cholesky decomposition $\bs{A_T}=\bs{LL}^{\top}$ where $\bs{L}$ is a lower triangular matrix (possibly with zeros on the diagonal if $\bs{A_T}$ is not positive definite). Let \bs{L_i} denote the $i$-the column of $\bs{L}$. 
Since  $\bs{A_T}$ is not a rank-one matrix, there exist at least two non-zero columns of $\bs{L}$. Let
$\bs{L_j}$ with $j>1$ be the second non-zero column. Then
\begin{align}
\frac{\bs{\beta_T^\top A_T\beta_T}}{\min\{1,z(T)\}}=&\frac{\bs{\beta_T^\top}\left(\sum_{i\neq j}(\bs{L_i L_i^\top})\right)\bs{\beta_T}}{\min\{1,z(T)\}}+\frac{\bs{\beta_T^\top}(\bs{L_j L_j^\top})\bs{\beta_T}}{\min\{1,z(T)\}}\notag\\
\leq&\frac{\bs{\beta_T^\top}\left(\sum_{i\neq j}(\bs{L_i L_i^\top})\right)\bs{\beta_T}}{\min\{1,z(T)\}}+\frac{\bs{\beta_T^\top}(\bs{L_j L_j^\top})\bs{\beta_T}}{\min\{1,z(\bar{T_j})\}} \cdot \label{eq:betterDecomp}
\end{align}
Finally, since $\bar{T_j}\in \Ps$, the (better) decomposition \eqref{eq:betterDecomp} is feasible for \eqref{eq:optDecomp}, and the proposition is proven. 
\end{proof}

By dropping the complementary constraints \eqref{eq:MIOEpi_complementary}, replacing the integrality constraints $\bs{z}\in \{0,1\}^\rev{p}$ with bound constraints $\bs{z}\in [0,1]^\rev{p}$, and utilizing the convex function $\phi_\Ps$ to reformulate \eqref{eq:MIOEpi_epigraph}, we obtain the convex relaxation of \eqref{eq:bestSubsetSelection}
\begin{subequations}\label{eq:rank1RelaxationOriginalSpace}
	\begin{align}
	\bs{y^\top y}+\min_{\rev{\bs{\beta},\bs{z},\bs{u}}}\;&-2\bs{y^\top X}\bs{\beta}+\mu\left(\bs{e^\top u}\right)+\phi_\Ps(\bs{z},\bs{\beta})\\
	&\bs{e^\top z}\leq k\label{eq:rank1Card}\\
	& \bs{\beta}\leq \bs{u},\; -\bs{\beta}\leq \bs{u}\label{eq:rank1Abs}\\
	&\bs{\beta}\in \R^\rev{p},\; \bs{z}\in [0,1]^\rev{p},\; \bs{u}\in \R_+^\rev{p} \label{eq:rank1Bounds}
	\end{align}
\end{subequations}
for a given $\Ps\subseteq 2^P$. In the next section, we give an interpretation of formulation \eqref{eq:rank1RelaxationOriginalSpace} as a sparsity-inducing regularization penalty. 

\subsection{Interpretation as regularization}\label{sec:regularization2} 

Note that the relaxation \eqref{eq:rank1RelaxationOriginalSpace} can be rewritten as:
\begin{align*}
\min_{\bs{\beta}\in \R^\rev{p}}\;&\|\bs{y}-\bs{X\beta}\|_2^2+\lambda\|\bs{\beta}\|_2^2+\mu\|\bs{\beta}\|_1+\rho_{\text{\RO}}(\bs{\beta};k)
\end{align*}
where \begin{align}\label{eq:regularizationPenalty}
\rho_{\text{\RO}}(\bs{\beta};k):=\min_{\bs{z}\in [0,1]^\rev{p}}\phi_\Ps(\bs{z},\bs{\beta})-\bs{\beta^\top} (\bs{X^\top X}+\lambda\bs{I})\bs{\beta}\text{ s.t. }\bs{e^{\top}z}\leq k.
\end{align}
is the (non-convex) \emph{rank-one regularization penalty}. Observe that $\rho_{\text{\RO}}(\bs{\beta};k)$ is the difference of two convex functions: the quadratic function $\bs{\beta^\top} (\bs{X^\top X}+\lambda\bs{I})\bs{\beta}$ arising from the fitness term and the Tikhonov regularization; and the projection of its convexification $\phi_\Ps(\bs{z},\bs{\beta})$ in the original space of the regression variables $\bs{\beta}$. As we now show, unlike the usual $\ell_1$ penalty, the \emph{rank-one regularization penalty} does not induce a bias when $\bs{\beta}$ is sparse.
\begin{theorem}
	If $\|\bs{\beta}\|_0\leq k$, then $\rho_{\text{\RO}}(\bs{\beta};k)=0$.
\end{theorem}
\begin{proof}
Let $(\bs{\beta},\bs{z})\in \R^\rev{p}\times [0,1]^\rev{p}$, and let $\bs{R}$ and $\bs{A_\rev{h}}$, $\rev{h=1,\ldots,m}$, correspond to an optimal solution of \eqref{eq:optDecomp}. Since 
\begin{align*}
\bs{\beta^\top} (\bs{X^\top X}+\lambda\bs{I})\bs{\beta}&=\bs{\beta^\top R\beta}+\rev{\sum_{h=1}^m}\bs{\beta^\top A_\rev{h}\beta}\\&\leq\bs{\beta^\top R\beta}+\rev{\sum_{h=1}^m}\frac{\bs{\beta^\top A_\rev{h}\beta}}{\min\{1,z(T_\rev{h})\}}=\phi_\Ps(\bs{z},\bs{\beta}),
\end{align*}
it follows that $\rho_{\text{\RO}}(\bs{\beta};k)\geq 0$ for any $\bs{\beta}\in \R^\rev{p}$. Now let $\bs{\hat \beta}$ satisfy $\|\bs{\hat \beta}\|_0\leq k$, let $\hat T=\left\{i\in P: \hat \beta_i\neq 0 \right\}$ be the support of $\bs{\hat\beta}$ and let $\bs{\hat z}$ such that $\hat z_i=\mathbbm{1}_{i\in \hat T}$ be the indicator vector of $\hat T$. By construction, $\bs{e^\top \hat z}\leq k$ and $\bs{\hat z}$ is feasible for problem~\eqref{eq:regularizationPenalty}. Moreover 
\begin{align*}
\rho_{\text{\RO}}(\bs{\hat \beta};k)\leq& \ \phi_\Ps(\bs{\hat z},\bs{\hat \beta})-\bs{\hat \beta^\top} (\bs{X^\top X}+\lambda\bs{I})\bs{\hat \beta}\\
=&\rev{\sum_{\substack {1\leq h\leq m\\ T_h\cap \hat T\neq \emptyset}}}\left(\frac{\bs{\hat \beta^\top A_\rev{h}\hat \beta}}{\min\{1,\hat z(T_\rev{h})\}}-\bs{\hat \beta^\top A_\rev{h}\hat \beta}\right)=0.
\end{align*}
Thus, $\rho_{\text{\RO}}(\bs{\hat \beta};k)=0$. 
\end{proof}

The rank-one regularization penalty $\rho_{\text{\RO}}$ can also be interpreted from an optimization perspective: note that problem \eqref{eq:optDecomp} is the \emph{separation} problem that, given $(\bs{\beta},\bs{z})\in \R^\rev{p}\times [0,1]^\rev{p}$, finds a decomposition that results in a most violated inequality after applying the rank-one strengthening. Thus, the regularization penalty $\rho_{\text{\RO}}(\bs{\beta};k)$ is precisely the violation of this inequality when $\bs{z}$ is chosen optimally. 

In \S\ref{sec:regularization} we derive an explicit form of $\rho_{\text{\RO}}(\bs{\beta};k)$ when $p=2$; Figure~\ref{fig:regularization} plots the graphs of the usual regularization penalties and $\rho_{\text{\RO}}$ for the two-dimensional case, and Figure~\ref{fig:constrained} illustrates the better sparsity inducing properties of regularization $\rho_{\text{\RO}}$. Deriving explicit forms of $\rho_{\text{\RO}}$ is cumbersome for $p\geq 3$. Fortunately, problem \eqref{eq:rank1RelaxationOriginalSpace} can be explicitly reformulated in an extended space as an SDP and tackled using off-the-shelf conic optimization solvers. 

\subsection{Extended SDP formulation}\label{sec:exactSDP}  To state the extended SDP formulation, in addition to variables $\bs{z}\in [0,1]^\rev{p}$ and $\bs{\beta}\in \R^\rev{p}$, we introduce variables $\bs{w}\in [0,1]^\rev{m}$ corresponding to terms $w_\rev{h}:=\min\{1,z(T_\rev{h})\}$ and $\bs{B}\in \R^{\rev{p}\times \rev{p}}$ corresponding to terms $B_{ij}=\beta_i\beta_j$. 
\begin{theorem}\label{theo:sdp}
	Problem \eqref{eq:rank1RelaxationOriginalSpace} is equivalent to the SDP
	\begin{subequations}\label{eq:sdp}
		\begin{align}
		\bs{y^\top y}+\min\;& -2\bs{y^\top X\beta} +\bs{e^\top u} +\langle \bs{X}^\top\bs{X}+\lambda \bs{I}, \bs{B}\rangle\\
		{\normalfont \text{s.t.}}\;&\bs{e^\top z}\leq k \label{eq:sdp_z}\\
		& \bs{\beta}\leq \bs{u},\; -\bs{\beta}\leq \bs{u}\label{eq:sdp_u}\\
		& w_{\rev{h}}\leq \bs{e_{T_\rev{h}}^\top z_{T_\rev{h}}}\; \quad\quad\quad\quad\quad\;\;\;  \forall \rev{h=1,\ldots,m}\label{eq:sdp_wT}\\
		& w_{\rev{h}}\bs{B_{T_\rev{h}}}-\bs{\beta_{T_\rev{h}}}\bs{\beta_{T_\rev{h}}^\top}\in \Sp^{T_\rev{h}} \quad\;  \forall \rev{h=1,\ldots,m}\label{eq:sdp_spT}\\
		&\bs{B}- \bs{\beta\beta^\top}\in \Sp^P \label{eq:sdp_sp}\\
		&\bs{\beta}\in \R^\rev{p},\; \bs{z}\in [0,1]^\rev{p},\; \bs{u}\in \R_+^\rev{p},\; \bs{w}\in [0,1]^\rev{m},\; \bs{B}\in \R^{\rev{p}\times \rev{p}}.
		\end{align}
	\end{subequations}
\end{theorem}
Observe that \eqref{eq:sdp} is indeed an SDP, as
\begin{align*}
w_{\rev{h}}\bs{B_{T_\rev{h}}}-\bs{\beta_{T_\rev{h}}}\bs{\beta_{T_\rev{h}}^\top}\in \Sp^{T_\rev{h}} \Leftrightarrow \ \begin{pmatrix}w_\rev{h} & \bs{\beta_{T_\rev{h}}^\top} \\ \bs{\beta_{T_\rev{h}}} & \bs{B_{T_\rev{h}}}\end{pmatrix}\succeq 0;
\end{align*}
thus constraints \eqref{eq:sdp_spT} and \eqref{eq:sdp_sp} are indeed SDP-representable and the remaining constraints and objective are linear.
\begin{proof}[Proof of Theorem~\ref{theo:sdp}] 
	It is easy to check that \eqref{eq:sdp} is strictly feasible (set $\bs{\beta}=0$, $\bs{z}=\bs{e}$, $\bs{w}>\bs{0}$ and $\bs{B}=\bs{I}$). Adding surplus variables $\bs{\Gamma}$, $\bs{\Gamma_\rev{h}}$ \rev{for $h=1,\ldots,m$},  write \eqref{eq:sdp} as 
	\small
	\begin{align}
	\bs{y^\top y}+\min_{\small (\bs{\beta},\bs{z},\bs{u},\bs{w})\in C}\,\,\Big\{-2\bs{y^\top X\beta} +\bs{e^\top u} +\min_{\bs{B},\bs{\Gamma_\rev{h}},\bs{\Gamma}}\;&\langle \bs{X}^\top\bs{X}+\lambda \bs{I}, \bs{B}\rangle\Big\}\notag\\
	\text{s.t.}\;
	& w_{\rev{h}}\bs{B_{T_\rev{h}}}-\bs{\Gamma_\rev{h}}=\bs{\beta_{T_\rev{h}}}\bs{\beta_{T_\rev{h}}^\top} \;\;  \forall \rev{h} \tag{$\bs{ A_\rev{h}}$}\\
	& \bs{B}-\bs{\Gamma}= \bs{\beta\beta^\top} \tag{$\bs{R}$}\\
	&\bs{\Gamma_\rev{h}}\in \Sp^{T_\rev{h}} \;\; \quad\quad\quad\quad\quad\;\; \forall \rev{h}\notag\\
	&\bs{\Gamma}\in \Sp^P \notag\\
	&\bs{B}\in \R^{\rev{p}\times \rev{p}},\notag
	\end{align}\normalsize
	where $C=\left\{\bs{\beta}\in \R^\rev{p}, \bs{z}\in [0,1]^\rev{p}, \bs{u}\in \R_+^\rev{p}, \bs{w}\in [0,1]^\rev{m}: \eqref{eq:sdp_z},\eqref{eq:sdp_u},\eqref{eq:sdp_wT}\right\}$. Using conic duality for the inner minimization problem, we find the dual
	\begin{align*}
	\bs{y^\top y}+\min_{(\bs{\beta},\bs{z},\bs{u},\bs{w})\in C}\;\Big\{-2\bs{y^\top X\beta} +\bs{e^\top u}+\max_{\bs{ A_\rev{h}},\bs{R}}\;&\langle \bs{\beta\beta^\top}, \bs{R}\rangle+\rev{\sum_{h=1}^m}\langle \bs{\beta\beta^\top},\bs{ A_\rev{h}}\rangle\Big\}\\
	\text{s.t.}\;& \rev{\sum_{j=1}^m} w_\rev{h} \bs{ A_\rev{h}} + \bs{R} = \bs{X}^\top\bs{X}+\lambda \bs{I} \\ 
	&\rev{(A_h)_{ij}=0} \qquad\rev{\text{for }i\not\in T_h \text{ or }j\not\in T_h}\\
	& \bs{ A_\rev{h}}\in \Sp^P \quad \forall T\in \Ps \\ 
	& \bs{R}\in \Sp^P.  
	\end{align*}
	After substituting $\bs{\bar A_\rev{h}}=w_\rev{h}\bs{A_\rev{h}}$ and noting that there exists an optimal solution with $w_\rev{h}=\min\{1,z(T_\rev{h})\}$, we obtain formulation \eqref{eq:optDecomp}.
\end{proof}

Note that if $\Ps=\{\emptyset\}$, there is no strengthening and \eqref{eq:sdp} is equivalent to \texttt{elastic net} ($\lambda,\mu>0$), \texttt{lasso} ($\lambda=0$, $\mu>0$), \texttt{ridge regression} ($\lambda>0$, $\mu=0$) or \texttt{ordinary least squares} ($\lambda=\mu=0$). 
As $|\Ps|$ increases, the quality of the conic relaxation \eqref{eq:sdp} for the non-convex $\ell_0$-problem \eqref{eq:bestSubsetSelection} improves, but the computational burden required to solve the resulting SDP also increases. In particular, the \emph{full} rank-one strengthening with $\Ps=2^P$ requires $2^{\rev{p}}$ semidefinite constraints and is impractical. Proposition~\ref{prop:rank1Decomp} suggests using down-monotone sets $\Ps$ with limited size 
\begin{subequations}\label{eq:sdpr}
	\begin{align}
	\bs{y^\top y}+\min\;& -2\bs{y^\top X\beta} +\bs{e^\top u} +\langle \bs{X}^\top\bs{X}+\lambda \bs{I}, \bs{B}\rangle\\
	\text{s.t.}\;&\bs{e^\top z}\leq k \\
	& \bs{\beta}\leq \bs{u},\; -\bs{\beta}\leq \bs{u}\\
(\sdp{r})  \ \ \ \ \ \ \ \ \ \ \ \ \ \ \ \ \ 	& 0\leq w_{T}\leq \min\{1,\bs{e_T^\top z_T}\} \quad  &\hspace{-2cm}\forall T: |T|\leq r\label{eq:sdpr_T1}\;\\
	&\; w_{T}\bs{B_T}-\bs{\beta_T\beta_T^\top}\in \Sp^T\quad  &\hspace{-2cm}\forall T: |T|\leq r \label{eq:sdpr_T2}\;\\
	&\bs{B}- \bs{\beta\beta^\top}\in \Sp^P \label{eq:sdpr_P}\\
	&\bs{\beta}\in \R^\rev{p},\; \bs{z}\in [0,1]^\rev{p},\; \bs{u}\in \R_+^\rev{p},\; \bs{B}\in \R^{\rev{p}\times \rev{p}},\bs{w}\in \R^m
	\end{align}
	\end{subequations}
for some $r\in \Z_+$ -- \rev{note that in the above formulation, $w_T$ is a scalar corresponding to the $T$-th coordinate of the $m$-dimensional vector $\bs{w}$}. In fact, if $r=1$, then \sdp{1} reduces to the formulation of the optimal \texttt{perspective relaxation} proposed in \cite{dong2015regularization}, which is equivalent to using \MC\ regularization. Our computations experiments show that whereas \sdp{1} may be a weak convex relaxation for problems with low diagonal dominance, \sdp{2} achieves excellent relaxation bounds even for the case of low diagonal-dominance within reasonable compute times. \rev{For clarity, we give the explicit form of the case \sdp{2}:
\begin{subequations}\label{eq:sdp2}
	\begin{align}
	\bs{y^\top y}+\min\;& -2\bs{y^\top X\beta} +\bs{e^\top u} +\langle \bs{X}^\top\bs{X}+\lambda \bs{I}, \bs{B}\rangle\\
	\text{s.t.}\;&\bs{e^\top z}\leq k \\
	& \bs{\beta}\leq \bs{u},\; -\bs{\beta}\leq \bs{u}\\
	& \begin{pmatrix}z_i& \beta_i\\
	\beta_i & B_{ii}
	\end{pmatrix}\succeq 0 &\hspace{-0.2cm}\forall i=1,\dots,p\label{eq:sdp2_T1}\\
	(\sdp{2})  \ \ \ \ \ \ \ \ \ \ \ \ \ \ \ \ \ & 0\leq w_{ij}\leq \min\{1,z_i+z_j\} \quad  &\hspace{-0.2cm}\forall i<j\\
	&\begin{pmatrix}w_{ij}& \beta_i & \beta_j\\
	\beta_i & B_{ii}&B_{ij}\\
	\beta_j & B_{ij}&B_{jj}\end{pmatrix}\succeq 0\quad  &\hspace{-0.2cm}\forall i<j \label{eq:sdp2_T2}\\
	&\begin{pmatrix} 1 & \bs{\beta^\top}\\
	\bs{\beta}& \bs{B}\end{pmatrix}\succeq 0 \label{eq:sdp2_P}\\
	&\bs{\beta}\in \R^p,\; \bs{z}\in [0,1]^p,\; \bs{u}\in \R_+^p,\; \bs{B}\in \R^{p\times p}.
	\end{align}
\end{subequations}
}

\section{Regularization for the two-dimensional case}\label{sec:regularization}
To better understand the properties of the proposed conic relaxations, in this section, we study them from a regularization perspective. 
 Consider formulation \eqref{eq:rank1Card} in Lagrangean form with multiplier $\kappa$:
 	\begin{subequations}\label{eq:lagrangean}
 	\begin{align}	
 \bs{y^\top y}+\min\;&-2\bs{y^\top X}\bs{\beta}+\bs{e^\top u}+\phi_\Ps(\bs{z},\bs{\beta})+\kappa \bs{e^\top z}\\
 & \bs{\beta}\leq \bs{u},\; -\bs{\beta}\leq \bs{u}\\
 &\bs{\beta}\in \R^P,\; \bs{z}\in [0,1]^P,\; \bs{u}\in \R_+^P, 
 \end{align}
 \end{subequations}
 where $p=2$, and 
\begin{equation}\label{eq:assumptionP2}\bs{X}^\top \bs{X}+\lambda \bs{I}=\begin{pmatrix}1+\delta_1 & 1\\ 1 & 1+\delta_2\end{pmatrix} \cdot
\end{equation}
Observe that assumption \eqref{eq:assumptionP2} is without loss of generality, provided that $\bs{X}^\top \bs{X}$ is not diagonal: given a two-dimensional convex quadratic function  $a_1\beta_1^2 +2a_{12}\beta_1\beta_2+a_2\beta_2^2$ (with $a_{12} \neq 0$), the substitution $\bar \beta_1=\alpha \beta_1$ and $\bar \beta_2=(a_{12}/\alpha) \beta_2$ with $|a_{12}|/a_2\leq \alpha\leq a_1$ yields a quadratic form satisfying \eqref{eq:assumptionP2}. Also note that we are using the Lagrangean form instead of the cardinality constrained form given in \eqref{eq:regularizationPenalty} for simplicity; however, since $\phi_\Ps(\bs{z},\bs{\beta})$ is convex in $\bs{z}$, there exists a value of $\kappa$ such that both forms are equivalent, i.e., result in the same optimal solutions $\bs{\hat \beta}$ for the regression problem, and the objective values differ by the constant $\kappa \cdot k$. 

If $\Ps=\{\emptyset, \{1\},\{2\}\}$, then \eqref{eq:lagrangean} reduces to a perspective strengthening of the form 
\begin{equation}\label{eq:perspRegularization}\bs{y}'\bs{y}+\min_{\bs{z}\in [0,1]^2,\bs{\beta}\in \R^2,}\; -2\bs{y'}\bs{X\beta}+\left(\beta_1+\beta_2\right)^2+\delta_1\frac{\beta_1^2}{z_1}+\delta_2\frac{\beta_2^2}{z_2}+\mu\|\bs{\beta}\|_1+\kappa\|\bs{z}\|_1.\end{equation} The links between \eqref{eq:perspRegularization} and regularization were studied\footnote{The case with $\mu=0$ is explicitly considered in \citet{dong2015regularization}, but the results extend straightforwardly to the case with  $\mu>0$ . The results presented here differ slightly from those in \cite{dong2015regularization} to account for a different scaling in the objective function.} in \cite{dong2015regularization}. 
\begin{proposition}[\citet{dong2015regularization}]\label{prop:regularizationPerspective}
	Problem \eqref{eq:perspRegularization} is equivalent to the regularization problem
	$$\min_{\bs{\beta}\in \R^2}\;\|\bs{y}-\bs{X\beta}\|_2^2+\lambda\|\bs{\beta}\|_2^2+\mu\|\bs{\beta}\|_1+\rho_{\text{\MC}}(\bs{\beta};\kappa,\bs{\delta})$$
	where 
	$$\rho_{\text{\MC}}(\bs{\beta};\kappa,\bs{\delta})=\begin{cases}\sum_{i=1}^2\left(2\sqrt{\kappa\delta_i}|\beta_i|-\delta_i\beta_i^2\right) & \text{if }\delta_i \beta_i^2\leq \kappa,\; i=1,2\\ 
	\kappa +2\sqrt{\kappa\delta_i}|\beta_i|-\delta_i\beta_i^2 & \text{if }\delta_i \beta_i^2\leq \kappa\text{ and } \delta_j \beta_j^2> \kappa\\
	2\kappa & \text{if }\delta_i \beta_i^2> \kappa,\; i=1,2.\end{cases}$$
\end{proposition} 
Regularization $\rho_{\text{\MC}}$ is non-convex and separable. Moreover, as pointed out in \cite{dong2015regularization}, the regularization given in Proposition~\ref{prop:regularizationPerspective} is the same as the Minimax Concave Penalty given in \cite{zhang2010nearly}; and, if $\lambda=\delta_1=\delta_2$, then the regularization given in Proposition~\ref{prop:regularizationPerspective} reduces to the reverse Huber penalty derived in \cite{pilanci2015sparse}. Observe that the regularization function $\rho_{\text{\MC}}$ is highly dependent on the diagonal dominance $\bs{\delta}$: specifically, in the low diagonal dominance setting with $\bs{\delta}=\bs{0}$, we find that $\rho_{\text{\MC}}(\bs{\beta};\kappa,\bs{0})=0$. 

We now consider conic formulation \eqref{eq:lagrangean} for the case $\Ps=\{\emptyset, \{1\},\{2\}, \{1,2\}\}$, corresponding to the full rank-one strengthening:
\begin{equation}\label{eq:rank1Regularization}\bs{y^\top y}+\min_{\bs{z}\in [0,1]^2,\bs{\beta}\in \R^2,}\! \! -2\bs{y^\top X\beta}+\frac{\left(\beta_1+\beta_2\right)^2}{\min\{1,z_1+z_2\}}+\delta_1\frac{\beta_1^2}{z_1}+\delta_2\frac{\beta_2^2}{z_2}+\mu\|\bs{\beta}\|_1+\kappa\|\bs{z}\|_1.\end{equation}

\begin{proposition}\label{prop:regularizationRank1}
	Problem \eqref{eq:rank1Regularization} is equivalent to the regularization problem
	$$\min_{\bs{\beta}\in \R^2}\;\|\bs{y}-\bs{X\beta}\|_2^2+\lambda\|\bs{\beta}\|_2^2+\mu\|\bs{\beta}\|_1+\rho_{\RO}(\bs{\beta};\kappa,\bs{\delta})$$
	where \small
	$$\rho_{\RO}(\bs{\beta};\kappa,\bs{\delta})=\begin{cases}2\sqrt{\kappa}\sqrt{\bs{\beta}'(\bs{X}^\top\bs{X}+\lambda \bs{I})\bs{\beta}+2\sqrt{\delta_1\delta_2}|\beta_1\beta_2|}-\bs{\beta}'(\bs{X}^\top\bs{X}+\lambda \bs{I})\bs{\beta}\\ & \hspace{-5cm}\hfill\text{if }\bs{\beta}'(\bs{X}^\top\bs{X}+\lambda \bs{I})\bs{\beta}+2\sqrt{\delta_1\delta_2}|\beta_1\beta_2|< \kappa\\
	\kappa +2\sqrt{\delta_1\delta_2}|\beta_1\beta_2| \\ &\hspace{-9cm}\hfill \text{if }\left(\sqrt{\delta_1}|\beta_1|+\sqrt{\delta_2}|\beta_2|\right)^2\leq \kappa \leq\bs{\beta}'(\bs{X}^\top\bs{X}+\lambda \bs{I})\bs{\beta}+2\sqrt{\delta_1\delta_2}|\beta_1\beta_2|\\
	\sum_{i=1}^2 \left(2\sqrt{\kappa\delta_i}|\beta_i|-\delta_i\beta_i^2\right)\\&\hspace{-7cm}\hfill\text{if }\left(\sqrt{\delta_1}|\beta_1|+\sqrt{\delta_2}|\beta_2|\right)^2> \kappa\text{ \& }\delta_i\beta_i^2\leq \kappa,\; i=1,2\\
	\kappa+\sqrt{\kappa\delta_i}|\beta_i|-\delta_i\beta_i^2 &\hfill\hspace{-2cm}\text{if }\delta_i\beta_i^2\leq \kappa \text{ \& }\delta_j\beta_j^2>\kappa\\
	2\kappa &\hspace{-5cm}\hfill \text{if }\delta_i\beta_i^2> \kappa,\; i=1,2.
	\end{cases}$$\normalsize
\end{proposition}

Observe that, unlike $\rho_{\text{\MC}}$, the function $\rho_{\RO}$ is not separable in $\beta_1$ and $\beta_2$ and does not vanish when $\bs{\delta}=0$: indeed, for $\bs{\delta}=0$ we find that \small
$$\rho_{\RO}(\bs{\beta};\kappa,\bs{0})=\begin{cases}2\sqrt{\kappa}\sqrt{\bs{\beta}'(\bs{X}^\top\bs{X}+\lambda \bs{I})\bs{\beta}}-\bs{\beta}'(\bs{X}^\top\bs{X}+\lambda \bs{I})\bs{\beta} & \text{if }\bs{\beta}'(\bs{X}^\top\bs{X}+\lambda \bs{I})\bs{\beta}< \kappa\\
\kappa & \text{if }0\leq \kappa \leq\bs{\beta}'(\bs{X}^\top\bs{X}+\lambda \bs{I})\bs{\beta}.
\end{cases}$$ 
\normalsize

\begin{proof}[Proof of Proposition ~\ref{prop:regularizationRank1}]
We prove the result by projecting out the $z$ variables in \eqref{eq:rank1Regularization}, i.e., giving closed form solutions for them. There are three cases to consider, depending on the optimal value for $z_1+z_2$.

\paragraph{$\bullet$ Case 1: $z_1+z_2<1$} In this case, we find by setting the derivatives of the objective in \eqref{eq:rank1Regularization} with respect to $z_1$ and $z_2$  that
\[
\left .
\begin{aligned}
\kappa-\delta_1\frac{\beta_1^2}{z_1^2}-\frac{\left(\beta_1+\beta_2\right)^2}{\left(z_1+z_2\right)^2} = 0\\
\kappa-\delta_2\frac{\beta_2^2}{z_2^2}-\frac{\left(\beta_1+\beta_2\right)^2}{\left(z_1+z_2\right)^2} = 0
\end{aligned} \ \ \right \}
\implies z_2 =\sqrt{\frac{\delta_2}{\delta_1}}\frac{|\beta_2|}{|\beta_1|}z_1.
\]
Define $\bar{z}:=\frac{z_1}{\sqrt{\delta_1}|\beta_1|}$, so $z_2=\sqrt{\delta_2}|\beta_2|\bar z$, and $z_1+z_2=\left(\sqrt{\delta_1}|\beta_1|+\sqrt{\delta_2}|\beta_2|\right)\bar z$. Moreover, we find that \eqref{eq:rank1Regularization} reduces to 
\begin{align}\bs{y^\top y}+\min_{\bar z>0,\bs{\beta}\in \R^2}& -2\bs{y^\top X\beta}+\mu\|\bs{\beta}\|_1\notag\\
+& \frac{\left(\beta_1+\beta_2\right)^2+\left(\sqrt{\delta_1}|\beta_1|+\sqrt{\delta_2}|\beta_2|\right)^2}{\left(\sqrt{\delta_1}|\beta_1|+\sqrt{\delta_2}|\beta_2|\right)\bar z}+\kappa\left(\sqrt{\delta_1}|\beta_1|+\sqrt{\delta_2}|\beta_2|\right)\bar z.\label{eq:intermediate1}\end{align}
An optimal solution of \eqref{eq:intermediate1} is attained at $$\bar{z}^*=\sqrt{\frac{\frac{\left(\beta_1+\beta_2\right)^2+\left(\sqrt{\delta_1}|\beta_1|+\sqrt{\delta_2}|\beta_2|\right)^2}{\left(\sqrt{\delta_1}|\beta_1|+\sqrt{\delta_2}|\beta_2|\right)}}{\kappa\left(\sqrt{\delta_1}|\beta_1|+\sqrt{\delta_2}|\beta_2|\right)}}=\frac{\sqrt{\left(\beta_1+\beta_2\right)^2+\left(\sqrt{\delta_1}|\beta_1|+\sqrt{\delta_2}|\beta_2|\right)^2}}{\sqrt{\kappa}\left(\sqrt{\delta_1}|\beta_1|+\sqrt{\delta_2}|\beta_2|\right)}$$
with objective value 
\begin{align}&\bs{y^\top y}+\min_{\bs{\beta}\in \R^2} -2\bs{y^\top X\beta}+\mu\|\bs{\beta}\|_1
+ 2\sqrt{\kappa}\sqrt{\left(\beta_1+\beta_2\right)^2+\left(\sqrt{\delta_1}|\beta_1|+\sqrt{\delta_2}|\beta_2|\right)^2}\notag\\
=& \min_{\bs{\beta}\in \R^2} \|\bs{y}-\bs{X\beta}\|_2^2+\lambda \|\bs{\beta}\|_2^2 +\mu \|\bs{\beta}\|_1\notag\\
&+\left(2\sqrt{\kappa}\sqrt{\left(\beta_1+\beta_2\right)^2+\left(\sqrt{\delta_1}|\beta_1|+\sqrt{\delta_2}|\beta_2|\right)^2}-(\beta_1+\beta_2)^2-\delta_1\beta_1^2-\delta_2\beta_2^2\right) \cdot\notag
\end{align}
Finally, this case happens when $z_1+z_2< 1\Leftrightarrow (\beta_1+\beta_2)^2 +(\sqrt{\delta_1}|\beta_1|+\sqrt{\delta_2}|\beta_2)^2< \kappa$. \\

\paragraph{$\bullet$ Case 2: $z_1+z_2>1$} In this case, we find by setting the derivatives of the objective in \eqref{eq:rank1Regularization} with respect to $z_1$ and $z_2$ that
$\bar z_i=\sqrt{\frac{\delta_i}{\kappa}}|\beta_i|$ for $i=1,2$. Thus, in this case, for an optimal solution $\bs{z}^*$ of \eqref{eq:rank1Regularization}, we have $z_i^*=\min\{\bar z_i,1\}$, and problem \eqref{eq:rank1Regularization} reduces to 
\small\begin{align*}&\bs{y^\top y}+\min_{\bs{\beta}\in \R^2,}\; -2\bs{y^\top X\beta}+\left(\beta_1+\beta_2\right)^2+\sum_{i=1}^2\max\left\{\delta_i \beta_i^2, \sqrt{\kappa\delta_i}|\beta_i|\right\}+\mu\|\bs{\beta}\|_1\\
&\quad\quad\quad\quad\quad\quad\quad +\sum_{i=1}^2\min\left\{\sqrt{\kappa\delta_i}|\beta_i|,\kappa\right\}\\
=&\min_{\bs{\beta}\in \R^2}\|\bs{y}-\bs{X\beta}\|_2^2+\lambda\|\bs{\beta}\|_2^2+\mu\|\bs{\beta}\|_1 \! + \! \sum_{i=1}^2\left(\max\left\{\delta_i \beta_i^2, \sqrt{\kappa\delta_i}|\beta_i|\right\} \! + \! \min\left\{\sqrt{\kappa\delta_i}|\beta_i|,\kappa\right\}\!-\! \delta_i\beta_i^2\right)\\
=&\min_{\bs{\beta}\in \R^2}\|\bs{y}-\bs{X\beta}\|_2^2+\lambda\|\bs{\beta}\|_2^2+\mu\|\bs{\beta}\|_1 \! + \!
\begin{cases}
\sum_{i=1}^2 \left(2\sqrt{\kappa\delta_i}|\beta_i|-\delta_i\beta_i^2\right) \! \! &\text{if }\delta_i\beta_i^2\leq \kappa,\; i=1,2\\
\sqrt{\kappa\delta_i}|\beta_i|-\delta_i\beta_i^2+\kappa \! \! &\text{if }\delta_i\beta_i^2\leq \kappa \  \& \ \delta_j\beta_j^2>\kappa\\
2\kappa \! \!& \text{if }\delta_i\beta_i^2> \kappa,\; i=1,2.
\end{cases}
\end{align*}\normalsize
Finally, this case happens when $z_1+z_2> 1\Leftrightarrow \left(\sqrt{\delta_1}|\beta_1|+\sqrt{\delta_2}|\beta_2|\right)^2> \kappa$. Observe that, in this case, the penalty function is precisely the one given in Proposition~\ref{prop:regularizationPerspective}. \\

\paragraph{$\bullet$ Case 3: $z_1+z_2=1$} In this case, problem \eqref{eq:rank1Regularization} reduces to
\begin{equation}\label{eq:intermediate2}\bs{y^\top y}+\min_{0\leq z_1\leq 1,\bs{\beta}\in \R^2,}\; -2\bs{y^\top X\beta}+\left(\beta_1+\beta_2\right)^2+\delta_1\frac{\beta_1^2}{z_1}+\delta_2\frac{\beta_2^2}{1-z_1}+\mu\|\bs{\beta}\|_1+\kappa.\end{equation}
Setting derivative with respect to $z_1$ in \eqref{eq:intermediate2} to $0$, we have 
\begin{align*}
0&=\delta_1\beta_1^2(1-z_1)^2-\delta_2\beta_2^2z_1^2\\
&=\delta_1\beta_1^2-2\delta_1\beta_1^2z_1+(\delta_1\beta_1^2-\delta_2\beta_2^2)z_1^2.
\end{align*}
Thus, we find that \begin{align*}z_1&=\frac{2\delta_1\beta_1^2 \pm \sqrt{4\delta_1^2\beta_1^4-4\delta_1\beta_1^2(\delta_1\beta_1^2-\delta_2\beta_2^2)}}{2\left(\delta_1\beta_1^2-\delta_2\beta_2^2\right)}\\
&=\frac{\delta_1\beta_1^2 \pm \sqrt{\delta_1\delta_2}|\beta_1\beta_2|}{\delta_1\beta_1^2-\delta_2\beta_2^2}=\frac{\sqrt{\delta_1}|\beta_1|(\sqrt{\delta_1}|\beta_1|\pm \sqrt{\delta_2}|\beta_2|)}{(\sqrt{\delta_1}|\beta_1|+\sqrt{\delta_2}|\beta_2|)(\sqrt{\delta_1|\beta_1|}-\sqrt{\delta_2}|\beta_2|)} \cdot 
\end{align*}
Moreover, since $0\leq z_1\leq 1$, we have 
$z_1=\frac{\sqrt{\delta_1}|\beta_1|}{\sqrt{\delta_1}|\beta_1|+\sqrt{\delta_2}|\beta_2|}$ and $1-z_1=\frac{\sqrt{\delta_2}|\beta_2|}{\sqrt{\delta_1}|\beta_1|+\sqrt{\delta_2}|\beta_2|}$. Substituting in \eqref{eq:intermediate2}, we find the equivalent form
\begin{align*}&\bs{y^\top y}+\min_{\bs{\beta}\in \R^2,}\; -2\bs{y^\top X\beta}+\left(\beta_1+\beta_2\right)^2+\left(\sqrt{\delta_1}|\beta_1|+\sqrt{\delta_2}|\beta_2|\right)^2+\mu\|\bs{\beta}\|_1+\kappa\\
=&\min_{\bs{\beta}\in \R^2}\|\bs{y}-\bs{X\beta}\|_2^2+\lambda\|\bs{\beta}\|_2^2+\mu \|\bs{\beta}\|_1+\kappa +2\sqrt{\delta_1\delta_2}|\beta_1\beta_2|.\end{align*}
This final case occurs when neither case 1 or 2 does, i.e., when $\left(\sqrt{\delta_1}|\beta_1|+\sqrt{\delta_2}|\beta_2|\right)^2 \leq \kappa \leq (\beta_1+\beta_2)^2 +(\sqrt{\delta_1}|\beta_1|+\sqrt{\delta_2}|\beta_2)^2$. 
\end{proof}

\ignore{
Figure~\ref{fig:regularization} compares usual regularization penalties with $\rho_{\text{\MC}}$ and $\rho_{\text{\RO}}$ (shown for $\kappa=1$). Additionally, Figure~\ref{fig:constrained} illustrates the better sparsity inducing properties of regularization $\rho_{\text{\RO}}$ with respect to $\rho_{\text{\MC}}$. Note that for the general $p$-dimensional case, the conic relaxation \eqref{eq:rank1RelaxationOriginalSpace} can be rewritten as:
	\begin{align*}
	\min\;&\|\bs{y}-\bs{X\beta}\|_2^2+\lambda\|\bs{\beta}\|_2^2+\bs{e^\top u}+\rho_{\text{\RO}}(\bs{\beta};k)\\
\text{s.t.} \;	& \bs{\beta}\leq \bs{u},\; -\bs{\beta}\leq \bs{u}\\
	&\bs{\beta}\in \R^P,\; \bs{u}\in \R_+^P,
	\end{align*}
with \begin{align*}
\rho_{\text{\RO}}(\bs{\beta};k):=\min_{\bs{z}\in [0,1]^P}\phi_\Ps(\bs{z},\bs{\beta})-\bs{\beta^\top} (\bs{X^\top X}+\lambda\bs{I})\bs{\beta}\text{ s.t. }\bs{e^{\top}z}\leq k.
\end{align*}
Explicit forms of such regularization penalties are difficult to derive for $p>2$. 
}

The plots of $\rho_{\text{\MC}}$ and $\rho_{\RO}$ shown in Figures~\ref{fig:regularization} and \ref{fig:constrained} correspond to setting the natural value $\kappa=1$.

\section{Conic quadratic relaxations}\label{sec:socp}

As mentioned in \S\ref{sec:intro}, strong convex relaxations of problem \eqref{eq:bestSubsetSelection}, such as \sdp{r}, can either be directly used to obtain good estimators via conic optimization, which is the approach we use in our computations, or can be embedded in a branch-and-bound algorithm to solve \eqref{eq:bestSubsetSelection} to optimality. However, using SDP formulations such as \eqref{eq:sdp} in branch-and-bound may be daunting since, to date, efficient branch-and-bound algorithms with SDP relaxations are not available.
\rev{In contrast, conic quadratic optimization problems are considerably easier to solve that semidefinite optimization problems, thus scaling to larger dimensions. Moreover there exist off-the-shelf mixed-integer conic quadratic optimization solvers that are actively maintained and improved by numerous software vendors.}
In this section we show how the proposed conic relaxations, and specifically \sdp{2}, can be implemented in a conic quadratic framework. \rev{The resulting convex formulations can then be directly used as a fast approximation to the SDP formulations presented in \S\ref{sec:convexification}, and pave the way towards an} integration with branch-and-bound solvers\footnote{An effective implementation would require careful constraint management strategies and integration with the different aspects of branch-and-bound solvers, e.g., branching strategies and heuristics. Such an implementation is beyond the scope of the paper.}. 

\subsection{\rev{Two-dimensional PSD constraints}}
Constraint \eqref{eq:sdp2_T1}, $\left.\beta_i^2\leq z_iB_{ii}\right.$, is a rotated cone constraint as $z_i\geq 0$ and $B_{ii}\geq 0$ in any feasible solution of \eqref{eq:sdpr}, and thus conic quadratic representable.

\subsection{\rev{Three-dimensional PSD constraints}}\label{sec:sdd}
As we now show, constraints \eqref{eq:sdp2_T2} can be accurately approximated using conic quadratic constraints. 
\begin{proposition}
	Problem \sdp{2} is equivalent to the optimization problem
	\begin{subequations}\label{eq:optimalSDDSOCP}
		\begin{alignat}{2}
		\bs{y^\top y}+\min\;& -2\bs{y^\top X\beta} +\bs{e^\top u} +\langle \bs{X}^\top\bs{X}+\lambda \bs{I}, \bs{B}\rangle\\
		{\normalfont\text{s.t.}}\;&\bs{e^\top z}\leq k  &&\label{eq:optimalSDDSOCP_card}\\
		& \bs{\beta}\leq \bs{u},\; -\bs{\beta}\leq \bs{u}  \label{eq:optimalSDDSOCP_abs} & \\
		& z_iB_{ii}\geq \beta_i^2 &\forall i\in P & \label{eq:optimalSDDSOCP_persp}\\
		&0\leq w_{ij}\leq 1,\; w_{ij}\leq z_i+z_j \quad &\forall i\neq j & \label{eq:optimalSDDSOCP_w}\\
		& 0\geq \max_{\alpha\geq 0}\left\{\frac{\alpha\beta_i^2+2\beta_i\beta_j+\beta_j^2/\alpha}{w_{ij}}\!-\!2B_{ij}\!-\!\alpha B_{ii}\!-\!B_{jj}/\alpha\right\} &\forall i\neq j & \label{eq:optimalSDDSOCP_2d1}\\
		& 0\geq \max_{\alpha\geq 0}\left\{\frac{\alpha\beta_i^2-2\beta_i\beta_j+\beta_j^2/\alpha}{w_{ij}}\!+\!2B_{ij}\!-\!\alpha B_{ii}\!-\!B_{jj}/\alpha\right\} &\forall  i\neq j & \label{eq:optimalSDDSOCP_2d2}\\
		&\bs{B}- \bs{\beta\beta}'\in \Sp^P \label{eq:optimalSDDSOCP_psd}\\
		&\bs{\beta}\in \R^\rev{p},\; \bs{z}\in [0,1]^\rev{p},\; \bs{u}\in \R_+^\rev{p},\; \bs{B}\in \R^{\rev{p}\times \rev{p}}\label{eq:optimalSDDSOCP_bounds}.
		\end{alignat}
	\end{subequations}
\end{proposition}
\begin{proof}
It suffices to compute the optimal value of $\alpha$ in \eqref{eq:optimalSDDSOCP_2d1}--\eqref{eq:optimalSDDSOCP_2d2}. Observe that the rhs of \eqref{eq:optimalSDDSOCP_2d1} can be written as
\begin{equation}\label{eq:innerMax}
v=\frac{2\beta_i\beta_j}{w_{ij}}-2B_{ij}-\min_{\alpha\geq 0}\left\{\alpha\left(B_{ii}-\frac{\beta_i^2}{w_{ij}}\right)+\frac{1}{\alpha}\left(B_{jj}-\frac{\beta_j^2}{w_{ij}}\right)\right\} \cdot
\end{equation}
Moreover, in an optimal solution of \eqref{eq:optimalSDDSOCP}, we have that $w_{ij}=\min\{1,z_i+z_j\}$. Thus,
due to constraints \eqref{eq:optimalSDDSOCP_persp}, we find that $B_{ii}-\nicefrac{\beta_i^2}{w_{ij}}\geq 0$ in optimal solutions of \eqref{eq:optimalSDDSOCP}, and equality only occurs if either $z_i=1$ or $z_j=0$. If either $B_{ii}=\nicefrac{\beta_i^2}{\min\{1,z_i+z_j\}}$ or $B_{jj}=\nicefrac{\beta_j^2}{\min\{1,z_i+z_j\}}$, then the optimal value of \eqref{eq:innerMax} is $v=\nicefrac{2\beta_i\beta_j}{\min\{1,z_i+z_j\}}-2B_{ij}$, by setting $\alpha\to \infty$ or $\alpha=0$, respectively. Otherwise, the optimal $\alpha$ equals \begin{equation}\label{eq:optimalAlpha}\alpha=\sqrt{\frac{B_{jj}w_{ij}-\beta_j^2}{B_{ii}w_{ij}-\beta_i^2}},\end{equation}
with the objective value 
$$v=\frac{2\beta_i\beta_j}{w_{ij}}-2B_{ij}-2\sqrt{\left(B_{ii}-\frac{\beta_i^2}{w_{ij}}\right)\left(B_{jj}-\frac{\beta_j^2}{w_{ij}}\right)}.$$
Observe that this expression is also correct when $B_{ii}=\nicefrac{\beta_i^2}{\min\{1,z_i+z_j\}}$ or $B_{jj}=\nicefrac{\beta_j^2}{\min\{1,z_i+z_j\}}$.
Thus, constraint \eqref{eq:optimalSDDSOCP_2d1} reduces to 
\begin{equation}\label{eq:simplifiedSDDSOCP_2d1}0\geq \beta_i\beta_j-B_{ij}w_{ij}-\sqrt{\left(B_{ii}w_{ij}-\beta_i^2\right)\left(B_{jj}w_{ij}-\beta_j^2\right)}.\end{equation}
Similarly, it can be shown that constraint \eqref{eq:optimalSDDSOCP_2d2} reduces to 
\begin{equation}\label{eq:simplifiedSDDSOCP_2d2}0\geq -\beta_i\beta_j+B_{ij}w_{ij}-\sqrt{\left(B_{ii}w_{ij}-\beta_i^2\right)\left(B_{jj}w_{ij}-\beta_j^2\right)}.\end{equation}
More compactly, constraints \eqref{eq:simplifiedSDDSOCP_2d1}--\eqref{eq:simplifiedSDDSOCP_2d2} are equivalent to
\begin{equation}\label{eq:simplifiedSDDSOCP_2d}
\left(w_{ij}B_{ii} -\beta_i^2\right)\left(w_{ij}B_{jj} -\beta_j^2\right)\geq \left(w_{ij}B_{ij}-\beta_i\beta_j\right)^2.
\end{equation}

Moreover, note that constraints \eqref{eq:sdpr_T2} with $T=\{i,j\}$ are equivalent to
\begin{align*}
& \begin{pmatrix}w_{ij}B_{ii} -\beta_i^2& w_{ij}B_{ij}-\beta_i\beta_j \\
w_{ij}B_{ij}-\beta_i\beta_j&w_{ij}B_{jj} -\beta_j^2\end{pmatrix}\in \Sp^2\\
\Leftrightarrow\;&w_{ij}B_{ii} -\beta_i^2\geq 0,\; w_{ij}B_{jj} -\beta_j^2\geq 0,\text{ and } \eqref{eq:simplifiedSDDSOCP_2d}.
\end{align*}
Since the first two constraints are implied by \eqref{eq:optimalSDDSOCP_persp} and $w_{ij}=\min\{1,z_i+z_j\}$ in optimal solutions, the proof is complete.
\end{proof}

Observe that, for any fixed value of $\alpha$, constraints \eqref{eq:optimalSDDSOCP_2d1}--\eqref{eq:optimalSDDSOCP_2d2} are conic quadratic representable. Thus, we can obtain relaxations of \eqref{eq:optimalSDDSOCP} of the form
\begin{subequations}\label{eq:relaxSDDSOCP}
	\begin{align}
	\bs{y^\top y}+\min\;& -2\bs{y^\top X\beta} +\bs{e^\top u} +\langle \bs{X}^\top\bs{X}+\lambda \bs{I}, \bs{B}\rangle\\
	\text{s.t.}\;&\eqref{eq:optimalSDDSOCP_card},\; \eqref{eq:optimalSDDSOCP_abs},\;\eqref{eq:optimalSDDSOCP_persp},\;\eqref{eq:optimalSDDSOCP_w},\; \eqref{eq:optimalSDDSOCP_psd},\; \eqref{eq:optimalSDDSOCP_bounds}\\
	& 0\geq \frac{\alpha\beta_i^2+2\beta_i\beta_j+\beta_j^2/\alpha}{\min\{1,z_i+z_j\}}\!-\! 2B_{ij} \!- \!\alpha B_{ii}- \! B_{jj}/\alpha,  \forall i\neq j,\! \alpha\in V_{ij}^+\\
	& 0\geq \frac{\alpha\beta_i^2-2\beta_i\beta_j+\beta_j^2/\alpha}{\min\{1,z_i+z_j\}}\!+\! 2B_{ij} \!- \! \alpha B_{ii} \! - \! B_{jj}/\alpha, \forall i\neq j,\! \alpha\in V_{ij}^-,
	\end{align}
\end{subequations}
where $V_{ij}^+$ and  $V_{ij}^-$ are any finite subsets of $\R_+$. Relaxation \eqref{eq:relaxSDDSOCP} can be refined dynamically: given an optimal solution of \eqref{eq:relaxSDDSOCP}, new values of $\alpha$ generated according to \eqref{eq:optimalAlpha} (resulting in most violated constraints) can be added to sets $V_{ij}^+$ and $V_{ij}^-$, resulting in tighter relaxations. Note that the use of cuts (as described here) to improve the continuous relaxations of mixed-integer optimization problems is one of the main reasons of the dramatic improvements of MIO software \cite{bixby2012brief}.

In relaxation \eqref{eq:relaxSDDSOCP}, $V_{ij}^+$ and $V_{ij}^-$ can be initialized with any (possibly empty) subsets of $\R_+$. However, setting $V_{ij}^+=V_{ij}^-=\{1\}$ yields a relaxation with a simple interpretation, discussed next.  

\subsection{Diagonally dominant matrix relaxation}
Let $\bs{\Lambda}\in \Sp^P$ be diagonally dominant matrix. Observe that for any $(\bs{z},\bs{\beta})\in \{0,1\}^\rev{p}\times \R^\rev{p}$ such that $\bs{\beta}\circ (\bs{e}-\bs{z})=\bs{0}$,  
\begin{align}t\geq\bs{\beta^\top\Lambda \beta}\Leftrightarrow & \;t\geq\sum_{i=1}^p\bigg(\Lambda_{ii}-\sum_{j\neq i}|\Lambda_{ij}|\bigg)\beta_i^2+\sum_{i=1}^p\sum_{j=i+1}^p|\Lambda_{ij}|\left(\beta_i+\sign(\Lambda_{ij})\beta_j\right)^2\notag\\
\Leftrightarrow & \; t\geq\sum_{i=1}^p\bigg(\Lambda_{ii}-\sum_{j\neq i}|\Lambda_{ij}|\bigg)\frac{\beta_i^2}{z_i}+\sum_{i=1}^p\sum_{j=i+1}^p|\Lambda_{ij}|\frac{\left(\beta_i+\sign(\Lambda_{ij})\beta_j\right)^2}{\min\{1,z_i+z_j\}},\label{eq:ddStrengthening}
\end{align}
where the last line follows from using perspective strengthening for the separable quadratic terms, and using \eqref{eq:valid} for the non-separable, rank-one terms. See \cite{atamturk2018sparse} for a similar strengthening for signal estimation based on nonnegative pairwise quadratic terms.

We now consider using decompositions of the form $\bs{\Lambda}+\bs{R}=\bs{X}^\top\bs{X}+\lambda \bs{I}$, where $\bs{\Lambda}$ is a diagonally dominant matrix and $\bs{R}\in \Sp^P$. Given such a decomposition, inequalities \eqref{eq:ddStrengthening} can be used to strengthen the formulations. Specifically, we consider relaxations of \eqref{eq:MIO} of the form
\begin{subequations}\label{eq:ddRelaxationOriginalSpace}
	\begin{align}
	\bs{y^\top y}+\min\;&-2\bs{y^\top X}\bs{\beta}+\bs{e^\top u}+\hat \phi(\bs{z},\bs{\beta})\\
	&\eqref{eq:rank1Card},\; \eqref{eq:rank1Abs},\; \eqref{eq:rank1Bounds},
	\end{align}
\end{subequations}
where
\begin{subequations}\label{eq:ddDecomp}
	\begin{align}
	\hat\phi(\bs{z},\bs{\beta}) := \max_{\bs{\Lambda},\bs{R}}\;&\bs{\beta^\top \! R\beta}+ \! \sum_{i=1}^p\! \bigg (\! \Lambda_{ii} \! - \! \sum_{j\neq i}|\Lambda_{ij}|\bigg)\! \frac{\beta_i^2}{z_i}\! + \! \sum_{i=1}^p \! \sum_{j=i+1}^p|\Lambda_{ij}|\frac{\left(\beta_i+\sign(\Lambda_{ij})\beta_j\right)^2}{\min\{1,z_i+z_j\}}\\
	\text{s.t.}\;&
     \bs{\Lambda}+\bs{R} = \bs{X}^\top\bs{X}+\lambda \bs{I}\\
	&\Lambda_{ii}\geq \sum_{j< i}|\Lambda_{ji}|+\sum_{j> i}|\Lambda_{ij}|\quad\quad \forall i\in P\\
	&\bs{R}\in \Sp^P.
	\end{align}
\end{subequations}

\begin{proposition}\label{prop:dd}
	Problem \eqref{eq:ddRelaxationOriginalSpace} is equivalent to 
	\begin{subequations}\label{eq:optimalDDSOCP}
		\begin{align}
		\bs{y^\top y}+\min\;& -2\bs{y^\top X\beta} +\bs{e^\top u} +\langle \bs{X}^\top\bs{X}+\lambda \bs{I}, \bs{B}\rangle\\
		\text{s.t.}\;&\bs{e^\top z}\leq k \\
		& \bs{\beta}\leq \bs{u},\; -\bs{\beta}\leq \bs{u}\\
		& z_iB_{ii}\geq \beta_i^2 \quad&\forall i\in P\\
	(\sdp{dd})  \ \ \ \ \ \ \ \ \ \ \ \ \ \ \ \ \	&0\leq w_{ij}\leq 1,\; w_{ij}\leq z_i+z_j \quad &\forall i\neq j \\
		& 0\geq \frac{\beta_i^2+2\beta_i\beta_j+\beta_j^2}{w_{ij}}-2B_{ij}- B_{ii}-B_{jj}\quad &\forall i\neq j\\
		& 0\geq \frac{\beta_i^2-2\beta_i\beta_j+\beta_j^2}{w_{ij}}+2B_{ij}- B_{ii}-B_{jj}\quad &\forall i\neq j\\
		&\bs{B}- \bs{\beta\beta}'\in \Sp^P \\
		&\bs{\beta}\in \R^\rev{p},\; \bs{z}\in [0,1]^\rev{p},\; \bs{u}\in \R_+^\rev{p},\; \bs{B}\in \R^{\rev{p}\times \rev{p}}.
		\end{align}
	\end{subequations}
\end{proposition}
\begin{proof}
Let $\bs{\Gamma},\bs{\Gamma^+},\bs{\Gamma^-}$ be nonnegative $p\times p$ matrices such that: $\Gamma_{ii}=\Lambda_{ii}$ and $\Gamma_{ij}=0$ for $i\neq j$; $\Gamma_{ii}^+=\Gamma_{ii}^-=0$ and $\Gamma_{ij}^+-\Gamma_{ij}^-=\Lambda_{ij}$ for $i\neq j$. 
Problem \eqref{eq:ddDecomp} can be written as 
\begin{subequations}\label{eq:dd_dual}
	\begin{align}
	\hat\phi(\bs{z},\bs{\beta}) := \max_{\bs{\Gamma},\bs{\Gamma^+},\bs{\Gamma^-}\bs{R}}\;&\bs{\beta^\top R\beta}+\sum_{i=1}^p\bigg(\Gamma_{ii}-\sum_{j\neq i}(\Gamma_{ij}^++\Gamma_{ij}^-)\bigg)\frac{\beta_i^2}{z_i}\\
	&+\sum_{i=1}^p\sum_{j=i+1}^p\bigg(\Gamma_{ij}^+\frac{\left(\beta_i+\beta_j\right)^2}{\min\{1,z_i+z_j\}}+\Gamma_{ij}^-\frac{\left(\beta_i-\beta_j\right)^2}{\min\{1,z_i+z_j\}}\bigg)\\
	\text{s.t.}\;&
	\bs{\Gamma}+\bs{\Gamma^+}+\bs{\Gamma^-}+\bs{R} = \bs{X}^\top\bs{X}+\lambda \bs{I}\\
	&\Gamma_{ii}\geq \sum_{j< i}(\Gamma_{ji}^++\Gamma_{ji}^-)+\sum_{j> i}(\Gamma_{ij}^++\Gamma_{ij}^-)\quad\quad \forall i\in P\\
	&\bs{R}\in \Sp^P.
	\end{align}
\end{subequations}
Then, similarly to the proof of Theorem~\ref{theo:sdp}, it is easy to show that the dual of \eqref{eq:dd_dual} is precisely \eqref{eq:optimalDDSOCP}.
\end{proof}

\subsection{Relaxing the \rev{$\bs{(p+1)}$-dimensional PSD constraint}}\label{sec:relaxPsdR} 
We now discuss a relaxation of the $p$-dimensional semidefinite constraint $\bs{B}- \bs{\beta\beta^\top}\in \Sp^P$, present in all formulations. 
\rev{Let $\bs{V}$ 
	be a matrix 
	whose $j$-th column $\bs{V_j}$ is an eigenvector of $\bs{X^\top X}$. }Consider the optimization problem
\begin{subequations}\label{eq:LPextraction}
	\begin{align}
	\underline{\phi_\Ps}(\bs{z},\bs{\beta}) := \max_{\bs{A_T},\bs{R},\bs{\pi}}\;&\bs{\beta^\top R\beta}+\sum_{T\in \Ps}\frac{\bs{\beta_T^\top A_T\beta_T}}{\min\{1,z(T)\}}\label{eq:LPextraction_obj}\\
	\text{s.t.}\;&
	\sum_{T\in \Ps}\bs{A_T} +\bs{R} = \bs{X}^\top\bs{X}+\lambda \bs{I}\\
	&\bs{A_T}\in \Sp^T \quad &\forall T\in \Ps \label{eq:LPextraction_AT}\\
	&\bs{R}=\rev{\bs{V}\diag(\bs{\pi})\bs{V^\top}}\\
	&\bs{\pi}\in \R_+^n.
	\end{align}
\end{subequations}

Observe that the objective and constraints \eqref{eq:LPextraction_obj}--\eqref{eq:LPextraction_AT} are identical to \eqref{eq:optDecomp}. However, instead of \eqref{eq:optDecomp_psd}, we have $\bs{R} =\rev{\sum_{j=1}^{\min\{p,n\}} \pi_j \bs{V_jV_j^\top}}$. Moreover, since $\bs{\pi}\geq \bs{0}$, $\bs{R}\in \Sp^P$ in any feasible solution of \eqref{eq:LPextraction}, thus \eqref{eq:optDecomp} is a relaxation of \eqref{eq:LPextraction}, and, hence, $\underline{ \phi_\Ps}$ is indeed a lower bound on $\phi_\Ps$. \rev{Finally, \eqref{eq:LPextraction} is feasible if $\lambda=0$ or \Ps\ contains all singletons, as it is possible to set $\bs{A_{\{i\}}}=\lambda$, $\bs{A_{T}}=0$ for $|T|>1$, and set $\bs{\pi}$ equal to the eigenvalues of $\bs{X^\top X}$.} 
Therefore, instead of \eqref{eq:rank1RelaxationOriginalSpace}, one may use the simpler convex relaxation
\begin{subequations}\label{eq:rank1RelaxationOriginalSpaceSOCP}
	\begin{align}
	\bs{y^\top y}+\min\;&-2\bs{y^\top X}\bs{\beta}+\bs{e^\top u}+\underline{ \phi_\Ps}(\bs{z},\bs{\beta})\\
	&\bs{e^\top z}\leq k\\
	& \bs{\beta}\leq \bs{u},\; -\bs{\beta}\leq \bs{u}\\
	&\bs{\beta}\in \R^\rev{p},\; \bs{z}\in [0,1]^\rev{p},\; \bs{u}\in \R_+^\rev{p}
	\end{align}
\end{subequations}
for \eqref{eq:bestSubsetSelection}.

\begin{proposition}
	\rev{If $\Ps=\left\{T\subseteq P: |T|\leq 2\right\}$, then} problem \eqref{eq:rank1RelaxationOriginalSpaceSOCP} is equivalent to
	\rev{
	\begin{subequations}\label{eq:socp1}
		\begin{align}
		\bs{y^\top y}+\min\;& -2\bs{y^\top X\beta} +\bs{e^\top u} +\langle \bs{X}^\top\bs{X}+\lambda \bs{I}, \bs{B}\rangle\\
		\text{s.t.}\;&\bs{e^\top z}\leq k \\
		& \bs{\beta}\leq \bs{u},\; -\bs{\beta}\leq \bs{u}\\
		& \begin{pmatrix}z_i& \beta_i\\
		\beta_i & B_{ii}
		\end{pmatrix}\succeq 0 &\hspace{-0.2cm}\forall i=1,\dots,p\\
		(\sdp{LB})  \ \ \ \ \ \ \ \ \ \ \ \ \ \ \ \ \ & 0\leq w_{ij}\leq \min\{1,z_i+z_j\} \quad  &\hspace{-3cm}\forall i<j\\
		&\begin{pmatrix}w_{ij}& \beta_i & \beta_j\\
		\beta_i & B_{ii}&B_{ij}\\
		\beta_j & B_{ij}&B_{jj}\end{pmatrix}\succeq 0\quad  &\hspace{-0.2cm}\forall i<j \label{eq:socp1_st}\\
		&\bs{\rev{V}_j^\top}\left(\bs{B}- \bs{\beta\beta^\top}\right)\bs{\rev{V}_j}\geq 0 &\hspace{-3cm} \forall \rev{j=1,\ldots,\min\{n,p\}} \label{eq:socp1_sp}\\ 
		&\bs{\beta}\in \R^p,\; \bs{z}\in [0,1]^p,\; \bs{u}\in \R_+^p,\; \bs{B}\in \R^{p\times p}.
		\end{align}
	\end{subequations}
}
\ignore{	
	\begin{subequations}\label{eq:socp1}
		\begin{alignat}{2}
		\bs{y^\top y}+\min\;& -2\bs{y^\top X\beta} +\bs{e^\top u} +\langle \bs{X}^\top\bs{X} +\lambda \bs{I},&& \bs{B}\rangle\\
		{\normalfont\text{s.t.}}\;&\bs{e^\top z}\leq k \\
		& \bs{\beta}\leq \bs{u},\; -\bs{\beta}\leq \bs{u}\\
(\cqp{r})  \ \ \ \ \ \ \ \ \ \  \ 			& w_{T}\leq \bs{e_T^\top z_T} \quad&& \forall T\in \Ps\\
		& w_{T}\bs{B_T}-\bs{\beta_T}\bs{\beta_T^\top}\in \Sp^T \quad&&   \forall T\in \Ps\label{eq:socp1_st}\\
		&\bs{\rev{V}_j^\top}\left(\bs{B}- \bs{\beta\beta^\top}\right)\bs{\rev{V}_j}\geq 0 && \forall \rev{j=1,\ldots,\min\{n,p\}} \label{eq:socp1_sp}\\ 
		&\bs{\beta}\in \R^\rev{p},\; \bs{z}\in [0,1]^\rev{p},\; \bs{u}\in \R_+^\rev{p},\; &&\bs{w}\in [0,1]^\rev{|\Ps|},\; \bs{B}\in \R^{\rev{p}\times \rev{p}}.
		\end{alignat}
	\end{subequations}
}
\end{proposition}

\begin{proof}
	The proof is based on conic duality similar to the proof of Theorem~\ref{theo:sdp}. 
\end{proof}
Observe that in formulation \eqref{eq:socp1}, the $(p+1)$-dimensional semidefinite constraint \eqref{eq:sdp_sp} is replaced with $\min\{p,n\}$ rank-one quadratic constraints \eqref{eq:socp1_sp}. We denote by \sdp{LB} the relaxation of \sdp{2} obtained by replacing \eqref{eq:sdpr_P} with \eqref{eq:socp1_sp}. In general, \sdp{LB} is still an SDP due to constraints \eqref{eq:socp1_st}; however, note that \sdp{LB} can be implemented in a conic quadratic framework by using cuts, as described in \S\ref{sec:sdd}. Moreover, constraints \eqref{eq:socp1_sp} could also be dynamically refined to better approximate the SDP constraint, or formulation \eqref{eq:socp1} could be improved with ongoing research on approximating SDP via mixed-integer conic quadratic optimization, e.g., see \cite{kocuk2016strong,kocuk2018matrix}.

\rev{
\begin{remark}
	We observe that formulation \eqref{eq:socp1} is solved substantially faster than \sdp{2} (with Mosek) with constraints \eqref{eq:socp1_st} formulated as semi-definite constraints. Indeed, the $\mathcal{O}(p^2)$ low-dimensional constraints \eqref{eq:sdp2_T2} can actually be handled efficiently, but the major computational bottleneck towards solving \sdp{2} is handling the single large-dimensional positive semi-definite constraint \eqref{eq:sdp2_P}.
\end{remark}
}

\ignore{
\begin{remark}[Further rank-one strenghtening]
	Since constraints \eqref{eq:socp1_sp} are rank-one quadratic constraints, additional strengthening can be achieved using the inequalities given in \S\ref{sec:rank1}. Specifically, let $T_j=\left\{i\in P: X_{ji}\neq 0\right\}$, and inequalities \eqref{eq:socp1_sp} may be replaced with stronger versions
	$$\langle \bs{B},\bs{\rev{V}_j\rev{V}_j^\top}\rangle- \frac{\left(\bs{\rev{V}_j^\top\beta}\right)^2}{\min\{1,z(T_j)\}}\geq 0\quad \;\;\;\forall j\rev{=1,\ldots,\min\{n,p\}}.$$
\qed
\end{remark}
}

\section{Computations}\label{sec:computations}

In this section, we report computational experiments with the proposed conic relaxations 
on synthetic as well as benchmark datasets. \rev{Semidefinite} optimization problems are solved with MOSEK 8.1 solver\rev{, and conic quadratic optimization problems (continuous and mixed-integer) are solved with CPLEX 12.8 solver. All computations are performed} on a laptop with a 1.80GHz Intel\textregistered Core\textsuperscript{TM} i7-8550U CPU and 16 GB main memory. All solver parameters were set to their default values. We divide our discussion in two parts: first, in \S\ref{sec:real}, we focus on the relaxation quality of \sdp{r} and its ability to approximate the exact $\ell_0$-problem \eqref{eq:bestSubsetSelection}; then, in \S\ref{sec:synthetic}, we adopt the same experimental framework used in \cite{bertsimas2016best,hastie2017extended} to generate synthetic instances and evaluate the proposed conic formulations from an inference perspective. In both cases, our results compare favorably with existing approaches in the literature. 

\subsection{Datasets} \label{sec:data}We use the benchmark datasets in Table~\ref{tab:diagonalDominance}. The first five were first used in \cite{miyashiro2015mixed} in the context of MIO algorithms for best subset selection, and later used in \cite{gomez2018mixed}. The \texttt{diabetes} dataset with all second interactions was introduced in \cite{efron2004least} in the context of \texttt{lasso}, and later used in \cite{bertsimas2016best}. 
A few datasets require some manipulation to eliminate missing values and handle categorical variables. The processed datasets before standardization\footnote{In our experiments, the datasets were standardized first.} can be  downloaded from \\ \texttt{http://atamturk.ieor.berkeley.edu/data/sparse.regression}.

In addition, we also use synthetic datasets generated similarly to \cite{bertsimas2016best,hastie2017extended}. Here we present a summary of the simulation setup and refer the readers to \cite{hastie2017extended} for an extended description. . For given dimensions $n,p$, sparsity $s$, predictor autocorrelation $\rho$, and signal-to-noise ratio SNR, the instances are generated as follows:
\begin{enumerate}
	\item The (true) coefficients $\bs{\beta}_0$ have the first $s$ components equal to one, and the rest equal to zero.
	\item The rows of the predictor matrix $\bs{X}\in \R^{n\times p}$ are drawn from i.i.d. distributions $\mathcal{N}_p(\bs{0},\bs{\Sigma})$, where $\bs{\Sigma}\in \R^{p\times p}$ has entry $(i,j)$ equal to $\rho^{|i-j|}$.
	\item The response vector $\bs{y}\in \R^n$ is drawn from $\mathcal{N}_p(\bs{X\beta_0},\sigma^2\bs{I})$, where $\sigma^2=\bs{{\beta_0}^\top X\beta_0}/\text{SNR}$.
\end{enumerate}
Similar data generation has been used in the literature \cite{bertsimas2016best,hastie2017extended}. 

\subsection{\rev{Relaxation quality}}
\label{sec:real}

\rev{
In this section we test the ability of \sdp{2}, given in \eqref{eq:sdp2}, and of \sdp{LB}, given in \eqref{eq:socp1}, to provide near-optimal solutions to problem \eqref{eq:bestSubsetSelection}, and compare its performance with MIO approaches. In \S\ref{sec:comp_lambda0}, we focus on the pure best subset selection problem with $\lambda=0$, which has received relatively little attention in the literature \cite{bertsimas2016best}; in \S\ref{sec:comp_lambdaPos} we consider problems with $\ell_0$-$\ell_2$ regularization, which has received more attention in the literature \cite{bertsimas2017sparse,hazimeh2018fast,hazimeh2020sparse,xie2018ccp}; in \S\ref{sec:comp_r} we study the impact of model complexity parameter $r$ on the relaxation quality, and in \S\ref{sec:comp_scal} we study the scalability of the proposed methods. }

\paragraph{\textbf{Computing optimality gaps \rev{for \sdp{r}}}} The optimal objective value $\nu_\ell^*$ of \sdp{r} provides a lower bound on the optimal objective value of \eqref{eq:bestSubsetSelection}. To obtain an upper bound, we use a simple greedy heuristic to retrieve a feasible solution for \eqref{eq:bestSubsetSelection}: given an optimal solution vector ${\bs{
		\bar\beta}}^*$ for \sdp{r}, let $\bar{\beta}_{(k)}^*$ denote the $k$-th largest absolute value. For $T=\left\{i\in P: |\bar\beta_i^*|\geq \bar{\beta}_{(k)}^*\right\}$, let $\bs{\hat \beta_T}$ be the $k$-dimensional \texttt{ols}/\texttt{ridge} estimator using only predictors in $T$, i.e., $$\bs{\hat \beta_T}=(\bs{X_T^\top X_T}+\lambda\bs{I_T})^{-1}\bs{X_T^{\top}y},$$
where $\bs{X_T}$ denotes the $n\times k$ matrix obtained by removing the columns with indexes not in $T$, and let $\bs{\tilde \beta}$ be the $P$-dimensional vector obtained by filling the missing entries in $\bs{\hat \beta_T}$ with zeros.
Since $\|\bs{\tilde\beta}\|_0\leq k$ by construction, $\bs{\tilde\beta}$ is feasible for \eqref{eq:bestSubsetSelection}, and its objective value $\nu_u$ is an upper bound on the optimal objective value of \eqref{eq:bestSubsetSelection}. Moreover, the optimality gap \rev{provided by any approach} can be computed as
\begin{equation}
\label{eq:gap}
\texttt{gap}=\frac{\nu_u-\nu_\ell^*}{\nu_\ell^*}\times 100.
\end{equation}
While stronger relaxations result in improved lower bounds $\nu_\ell^*$, the corresponding heuristic upper bounds $\nu_u$ are not necessarily better; thus, the optimality gaps are not guaranteed to improve with stronger relaxations. Nevertheless, as shown next, stronger relaxations in general yield much smaller gaps in practice.

We point out that the main focus of the strong relaxations is to obtain improved lower bounds $\nu_\ell^*$. 
Randomized rounding methods \cite{pilanci2015sparse,xie2018ccp}, more sophisticated rounding heuristics \cite{dong2015regularization}, or alternative heuristic methods \cite{hazimeh2018fast} can be used to obtain
improved upper bounds. Nevertheless, the quality of the upper bounds obtained from the greedy rounding method can be used to estimate how well the solutions from the relaxations match the sparsity pattern of the optimal solution. 

\subsubsection{\rev{$\lambda=0$ case}}
\label{sec:comp_lambda0}
\rev{
For each dataset with $\lambda=\mu=0$, we solve the conic relaxations of \eqref{eq:bestSubsetSelection} \sdp{1} and \sdp{2} as well as \sdp{LB} and the mixed-integer formulation \mio\, given by \eqref{eq:MIO}--\eqref{eq:bigM}. In our experiments, we set $M=3\|\bs{\beta_\text{ols}}\|_\infty$, where $\bs{\beta_\text{ols}}$ is the ordinary least square estimator\footnote{\citet{bertsimas2016best} set $M=2\|\bs{\hat \beta}\|_\infty$ for some heuristic solution $\bs{\hat \beta}$} and set a time limit of 10 minutes. For data with $p\leq 40$ we solve problems with cardinalities $k\in \{3,\dots,10\}$, and for \texttt{diabetes} and \texttt{crime} we solve problems with $k\in\{3,\dots,30\}$. Table~\ref{tab:resultsLambda0} shows, for each dataset and method, the average lower bound (\texttt{LB}) and upper bound (\texttt{UB}) found by each method, the \texttt{gap} \eqref{eq:gap}, and the \texttt{time} required to solve the problems (in seconds) -- the average is taken across all $k$ values. In all cases, lower and upper bounds are scaled so that the best upper bound for any given instance has value $\nu_u^*=100$.}

\begin{table}[!t]
	\caption{\rev{Results with $\lambda=0$ on real instances. Lower and upper bounds are scaled so that the best upper bound for a given instance has value 100. Mean $\pm$ stdev are reported.}}
	\label{tab:resultsLambda0}
	\scalebox{1.0}{
		\begin{tabular}{c c| r r r | c }
			\hline \hline
			\texttt{dataset}& \texttt{method} & \texttt{LB} & \texttt{UB} & \texttt{gap(\%)}&\texttt{time}\\
			\hline
			\multirow{4}{*}{\texttt{housing}}&\sdp{1}&99.4$\pm$0.6 & 100.1$\pm$0.1 & 0.7$\pm$0.0 & 0.03$\pm$0.02\\
			&\sdp{2}& 99.6$\pm$0.6 & 100.1$\pm$0.1 & 0.5$\pm$0.6 & 0.07$\pm$0.03\\
			&\sdp{LB}& 98.8$\pm$0.6 & 100.4$\pm$0.1 & 1.6$\pm$0.8 & 0.06$\pm$0.03\\
			&\mio& 100.0$\pm$0.0 & 100.0$\pm$0.0 & 0.0$\pm$0.0 & 0.01$\pm$0.01\\
			&&&&&\\
			\multirow{4}{*}{\texttt{servo}}&\sdp{1}&86.8$\pm$5.5 & 109.5$\pm$10.3 & 27.3$\pm$20.6 & 0.02$\pm$0.01\\
			&\sdp{2}& 94.9$\pm$2.9 & 106.2$\pm$16.5 & 12.2$\pm$19.8 & 0.10$\pm$0.01\\
			&\sdp{LB}& 89.5$\pm$2.7 & 109.2$\pm$15.5 & 21.8$\pm$14.9 & 0.17$\pm$0.03\\
			&\mio$\dagger$& $\dagger$ & $\dagger$ & $\dagger$ & $\dagger$\\
			&&&&&\\
			\multirow{4}{*}{\texttt{auto MPG}}&\sdp{1}&75.3$\pm$10.3 & 115.3$\pm$6.0 & 55.8$\pm$23.7 & 0.07$\pm$0.04\\
			&\sdp{2}& 96.7$\pm$3.3 & 100.5$\pm$0.8 & 4.0$\pm$4.2 & 0.24$\pm$0.02\\
			&\sdp{LB}& 78.8$\pm$7.7 & 101.6$\pm$2.7 & 30.0$\pm$14.0 & 0.40$\pm$0.09\\
			&\mio$\dagger$& $\dagger$ & $\dagger$ & $\dagger$ & $\dagger$\\
			&&&&&\\
			\multirow{4}{*}{\texttt{solar flare}}&\sdp{1}&97.5$\pm$1.5 & 103.3$\pm$1.1 & 6.0$\pm$2.0 & 0.07$\pm$0.03\\
			&\sdp{2}& 99.2$\pm$0.8 & 100.0$\pm$0.0 & 1.0$\pm$0.6 & 0.28$\pm$0.06\\
			&\sdp{LB}& 97.8$\pm$1.6 & 102.3$\pm$1.9 & 4.6$\pm$2.7 & 0.13$\pm$0.02\\
			&\mio$\dagger\dagger$& 98.1$\pm$1.7  & 98.1$\pm$1.7 & - & 0.01$\pm$0.01\\
			&&&&&\\
			\multirow{4}{*}{\texttt{breast cancer}}&\sdp{1}&88.9$\pm$3.1 & 101.5$\pm$1.7 & 14.4$\pm$5.6 & 0.15$\pm$0.02\\
			&\sdp{2}& 98.0$\pm$0.6 & 100.4$\pm$0.8 & 2.4$\pm$1.1 & 0.77$\pm$0.07\\
			&\sdp{LB}& 94.8$\pm$0.5 & 100.5$\pm$0.7 & 6.0$\pm$0.5 & 0.40$\pm$0.03\\
			&\mio$\dagger$& $\dagger$  & $\dagger$ & $\dagger$ & $\dagger$\\
			&&&&&\\
			\multirow{4}{*}{\texttt{diabetes}}&\sdp{1}&95.2$\pm$3.2 & 115.2$\pm$11.8 & 22.2$\pm$16.3 & 3.58$\pm$0.77\\
			&\sdp{2}& 97.4$\pm$1.3 & 105.4$\pm$4.2 & 8.2$\pm$5.2 & 9.28$\pm$1.12\\
			&\sdp{LB}$\dagger$& $\dagger$ & $\dagger$ & $\dagger$ & $\dagger$\\
			&\mio& 99.0$\pm$0.9  & 100.0$\pm$0.0 & 1.0$\pm$0.9 & 416.17$\pm$260.57 \\
			&&&&&\\
			\multirow{4}{*}{\texttt{crime}}&\sdp{1}&97.8$\pm$1.3 & 103.2$\pm$2.4 & 5.6$\pm$3.6 & 17.82$\pm$0.98\\
			&\sdp{2}& 99.0$\pm$0.8 & 101.6$\pm$2.0 & 2.7$\pm$2.7 &45.29$\pm$4.06\\
			&\sdp{LB}& 94.6$\pm$2.0 & 109.7$\pm$2.8 & 16.0$\pm$4.9 & 5.87$\pm$0.43\\
			&\mio& 96.4$\pm$1.7  & 100.0$\pm$0.0 & 3.7$\pm$1.8 & 527.03$\pm$185.64 \\
			\hline \hline
			\multicolumn{5}{l}{$\dagger$ Error in solving problem.}\\
			\multicolumn{5}{l}{$\dagger\dagger$ Infeasible solution is reported as optimal.}\\
	\end{tabular}}
\end{table}

\rev{The \mio\ method is highly inconsistent and prone to numerical difficulties, due to the use of big-$M$ constraints. First, for three datasets (\texttt{servo}, \texttt{auto MPG} and \texttt{breast cancer}) the method fails due to numerical issues (``failure to solve MIP subproblem"). In addition, for \texttt{solar flare} the solver reports very fast solution times but the solutions are in fact infeasible for problem \eqref{eq:bestSubsetSelection}: by default in CPLEX, if $z_i\leq 10^{-5}$ in a solution then $z_i$ is deemed to satisfy the integrality constraint $z_i\in \{0,1\}$. Thus, if the big-$M$ constant is large enough, then constraint \eqref{eq:bigM} may in fact allow nonzero values for $\beta_i$ even when ``$z_i=0$". In particular, in \texttt{solar flare} we found that the solution $\bs{\beta_\text{mio}}$ reported by the MIO solver satisfies\footnote{We consider $\beta_i\neq 0$ whenever $\|\beta_i\|>10^{-4}$.} $\|\bs{\beta_\text{mio}}\|_0=20$, regardless of the value of $k$ used, violating the sparsity constraint. We also point out that \sdp{LB} struggles with numerical difficulties in \texttt{diabetes}: the problems are incorrectly found to be unbounded. In contrast, \sdp{r} methods are solved without numerical difficulties.}

\rev{In terms of the relaxation quality, we find that \sdp{2} is the best as expected. It consistently delivers better lower and upper bounds compared to the other conic relaxations, and even outperforming \mio\ in terms of lower bounds and gaps in the largest dataset (\texttt{crime}). The strength of the relaxation comes at the expense 2--4-fold larger computation time than \sdp{1}, but on the other hand \sdp{2} is substantially faster than \texttt{big-M}\ on large datasets. We see that neither \sdp{1} nor \sdp{LB} dominates each other in terms of relaxation quality. While \sdp{1} is faster on the smaller datasets, \sdp{LB} is faster on \texttt{crime}, indicating that \sdp{LB} may scale better (we corroborate this statement in \S\ref{sec:comp_scal}). Finally \mio, in datasets where numerical issues do not occur, is able to find high quality solutions consistently, but struggles to find matching lower bound in larger instances, despite significantly higher computation time spent.}

\rev{Figures~\ref{fig:diabetesl0} and \ref{fig:crimel0} present detailed results on lower bounds and gaps as a function of the sparsity parameter $k$ for the \texttt{diabetes} and \texttt{crime} datasets. For small values of $k$, \texttt{big-M} is arguably the best method, solving the problems to optimality. However, as $k$ increases, the quality of the lower bounds and gaps deteriorate: for \texttt{diabetes}, \sdp{2} finds better solutions than \texttt{big-M} for $k\geq 18$; for \texttt{crime}, \sdp{1} and \sdp{2} find better lower bounds for $k\geq 8$ (and, in the case of \sdp{2}, better gaps as well), and \sdp{LB} matches the lower bound found by \texttt{big-M} for $k\geq 14$, despite requiring only five seconds (instead of 10 minutes) to find such lower bounds. Observe that the number of possible supports ${p \choose k}=\mathcal{O}(p^k)$ for problem \eqref{eq:bestSubsetSelection} scales exponentially with $k$, thus enumerative methods such as branch-and-bound may struggle as $k$ grows.  }

\begin{figure}[!h]
	\centering
	\subfloat[Lower bounds]{\includegraphics[width=0.49\textwidth,trim={11cm 5.8cm 11cm 5.8cm},clip]{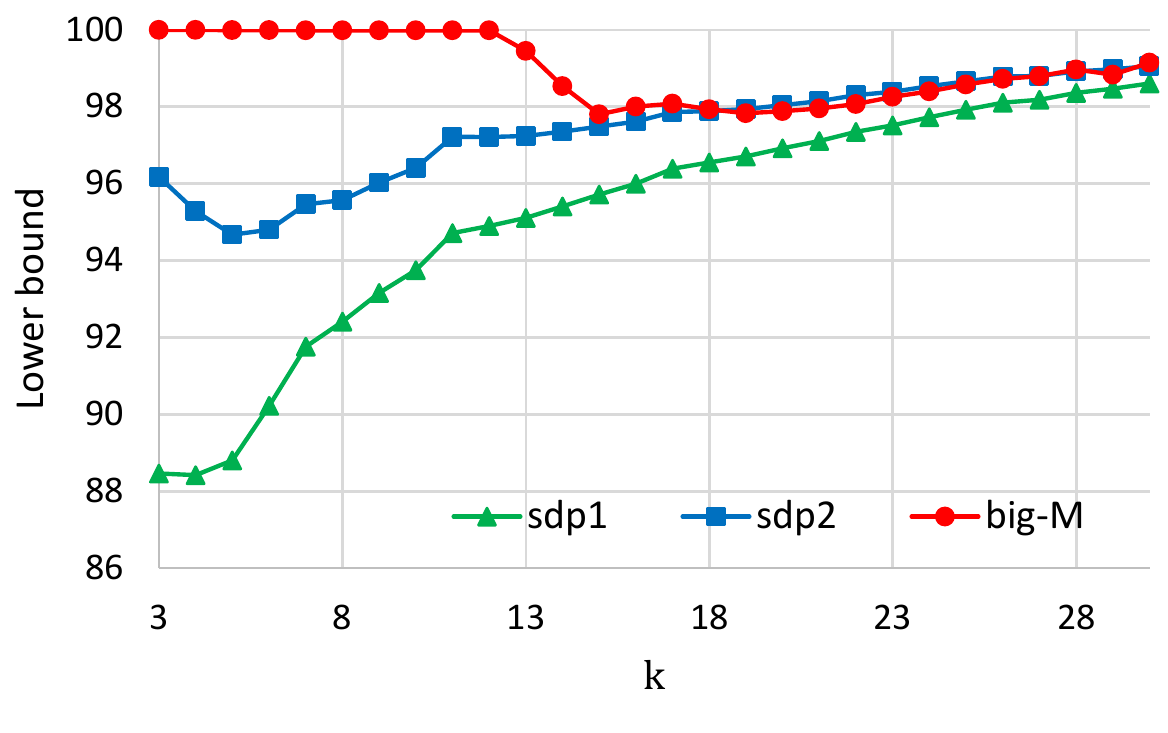}}\hfill\subfloat[Gaps]{\includegraphics[width=0.49\textwidth,trim={11cm 5.8cm 11cm 5.8cm},clip]{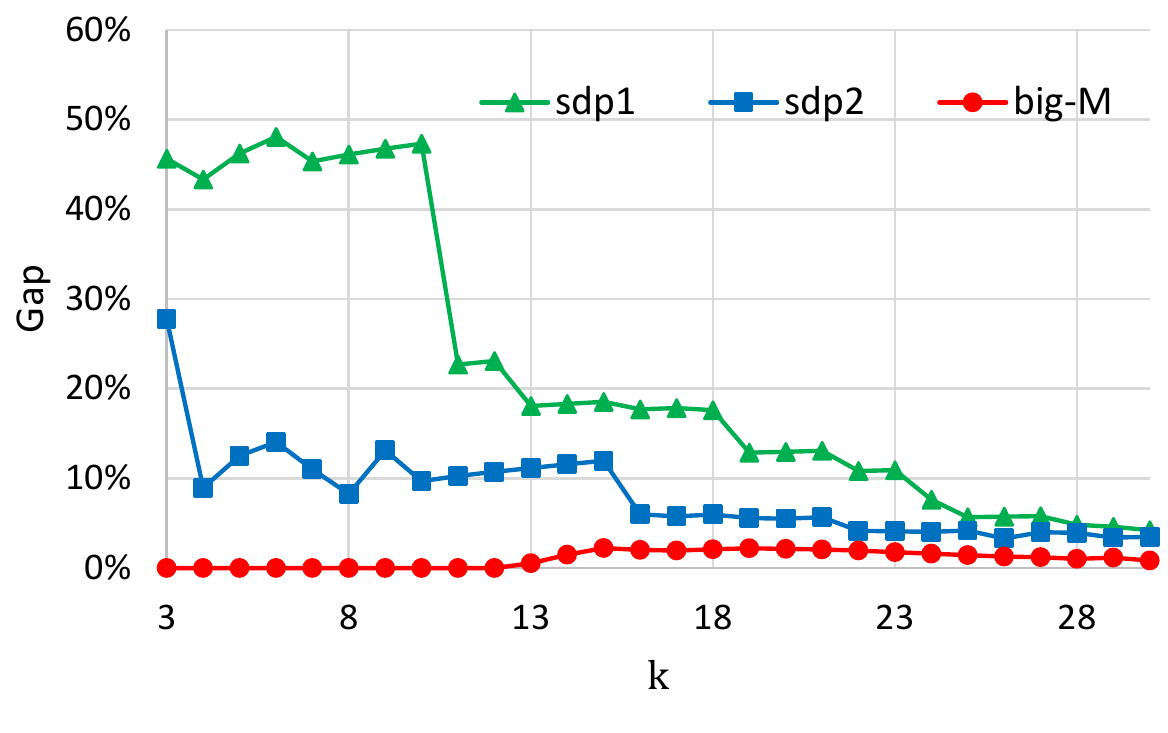}}
	\caption{\small Detailed results on the \texttt{diabetes} dataset with $\lambda=0$.}
	\label{fig:diabetesl0}
\end{figure}

\begin{figure}[!h]
	\centering
	\subfloat[Lower bounds]{\includegraphics[width=0.49\textwidth,trim={11cm 5.8cm 11cm 5.8cm},clip]{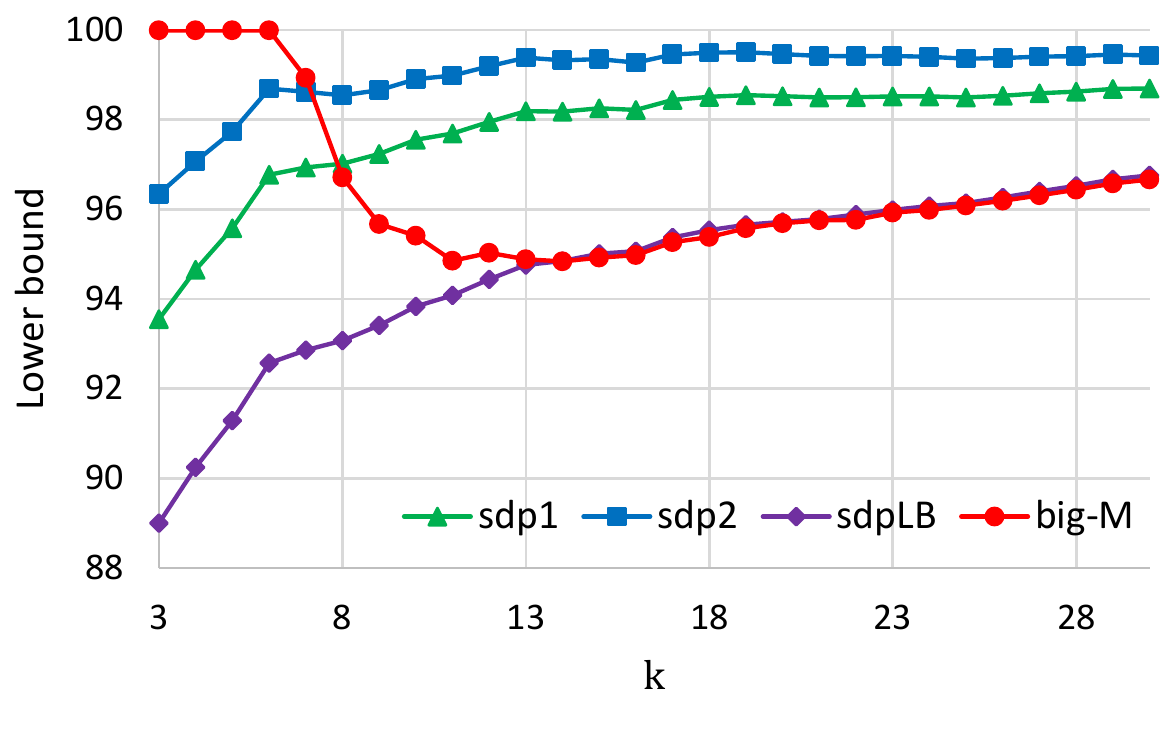}}\hfill\subfloat[Gaps]{\includegraphics[width=0.49\textwidth,trim={11cm 5.8cm 11cm 5.8cm},clip]{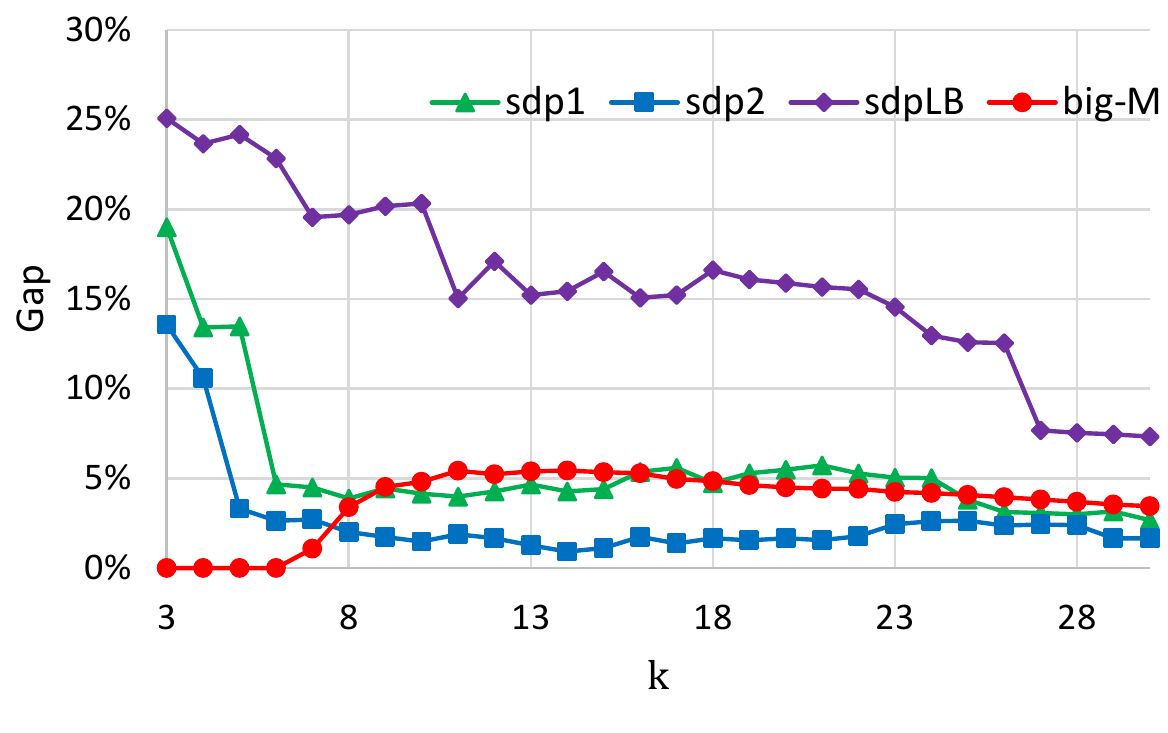}}
	\caption{\small Detailed results on the \texttt{crime} dataset with $\lambda=0$.}
	\label{fig:crimel0}
\end{figure}

\subsubsection{\rev{$\lambda>0$ case}}
\label{sec:comp_lambdaPos}

\rev{For each dataset with\footnote{\rev{Since data is standardized so that each column has unit norm, a value of $\lambda=0.05$ corresponds to an increase of 5\% in the diagonal elements of the matrix $\bs{X^\top X}+\lambda \bs{I}$.}} $\lambda=0.05$ and $\mu=0$, we solve the conic relaxations of \eqref{eq:bestSubsetSelection} \sdp{1}, \sdp{2} and \sdp{LB} and the ``big-$M$ free" mixed-integer formulation \eqref{eq:MIOPersp} with a time limit of 10 minutes (\mioP). This MIO formulation is possible since $\lambda>0$, and has been shown to be competitive \cite{xie2018ccp,hazimeh2020sparse} with the tailored algorithm proposed in \cite{bertsimas2017sparse}. For datasets with $p\leq 40$ we solve the problems with cardinalities $k\in \{3,\dots,10\}$, and for \texttt{diabetes} and \texttt{crime} we solve the problems with $k\in\{3,\dots,30\}$.
Table~\ref{tab:resultsLambda005} shows, for each dataset and method, the average lower bound (\texttt{LB}) and upper bound (\texttt{UB}) found by each method, the \texttt{gap} \eqref{eq:gap}, and the \texttt{time} required to solve the problems (in seconds) -- the average is taken across all $k$ values. In all cases, lower and upper bounds are scaled so that the best upper bound for any given instance has value $\nu_u^*=100$.}

\begin{table}[!h]
	\caption{\rev{Results with $\lambda=0.05$ on real instances. Lower and upper bounds are scaled so that the best upper bound for a given instance has value 100. Mean $\pm$ stdev are reported.}}
	\label{tab:resultsLambda005}
	\scalebox{1.0}{
		\begin{tabular}{c c| r r r | c }
			\hline \hline
			\texttt{dataset}& \texttt{method} & \texttt{LB} & \texttt{UB} & \texttt{gap(\%)}&\texttt{time}\\
			\hline
			\multirow{4}{*}{\texttt{housing}}&\sdp{1}&99.7$\pm$0.4 & 100.2$\pm$0.3 & 0.5$\pm$0.6 & 0.03$\pm$0.02\\
			&\sdp{2}& 99.8$\pm$0.3 & 100.1$\pm$0.2 & 0.3$\pm$0.5 & 0.06$\pm$0.02\\
			&\sdp{LB}& 99.5$\pm$0.4 & 100.3$\pm$0.3 & 0.8$\pm$0.6 & 0.06$\pm$0.02\\
			&\mioP& 100.0$\pm$0.0 & 100.0$\pm$0.0 & 0.0$\pm$0.0 & 0.11$\pm$0.03\\
			&&&&&\\
			\multirow{4}{*}{\texttt{servo}}&\sdp{1}&95.9$\pm$3.0 & 102.2$\pm$6.7 & 6.7$\pm$4.1 & 0.03$\pm$0.01\\
			&\sdp{2}& 99.5$\pm$0.5 & 100.6$\pm$1.1 & 1.1$\pm$1.6 & 0.11$\pm$0.01\\
			&\sdp{LB}& 97.6$\pm$1.4 & 102.0$\pm$2.1 & 4.6$\pm$3.3 & 0.16$\pm$0.02\\
			&\mioP& 100.0$\pm$0.0 & 100.0$\pm$0.0 & 0.0$\pm$0.0 & 0.28$\pm$0.13\\
			&&&&&\\
			\multirow{4}{*}{\texttt{auto MPG}}&\sdp{1}&89.1$\pm$6.1 & 101.4$\pm$1.2 & 14.4$\pm$8.5 & 0.05$\pm$0.01\\
			&\sdp{2}& 99.8$\pm$0.2 & 100.0$\pm$0.1 & 0.2$\pm$0.3 & 0.25$\pm$0.04\\
			&\sdp{LB}& 92.7$\pm$3.1 & 101.1$\pm$1.5 & 9.2$\pm$4.0 & 0.35$\pm$0.02\\
			&\mioP& 100.0$\pm$0.0 & 100.0$\pm$0.0 & 0.0$\pm$0.0 & 1.29$\pm$0.60\\
			&&&&&\\
			\multirow{4}{*}{\texttt{solar flare}}&\sdp{1}&99.3$\pm$0.5 & 100.1$\pm$0.1 & 0.8$\pm$0.5 & 0.07$\pm$0.01\\
			&\sdp{2}& 99.9$\pm$0.1 & 100.1$\pm$0.1 & 0.2$\pm$0.1 & 0.28$\pm$0.03\\
			&\sdp{LB}& 99.2$\pm$0.7 & 100.4$\pm$1.2 & 1.2$\pm$1.0 & 0.16$\pm$0.03\\
			&\mioP& 100.0$\pm$0.0  & 100.0$\pm$0.0 & 0.0$\pm$0.0 & 1.75$\pm$1.07\\
			&&&&&\\
			\multirow{4}{*}{\texttt{breast cancer}}&\sdp{1}&94.9$\pm$1.8 & 100.8$\pm$0.4 & 6.3$\pm$2.4 & 0.18$\pm$0.04\\
			&\sdp{2}& 99.6$\pm$0.2 & 100.1$\pm$0.2 & 0.5$\pm$0.3 & 0.72$\pm$0.06\\
			&\sdp{LB}& 97.5$\pm$0.6 & 100.5$\pm$0.4 & 2.9$\pm$0.9 & 0.36$\pm$0.05\\
			&\mioP& 100.0$\pm$0.0  & 100.0$\pm$0.0 & 0.0$\pm$0.0 & 56.12$\pm$44.34\\
			&&&&&\\
			\multirow{4}{*}{\texttt{diabetes}}&\sdp{1}&98.9$\pm$0.6 & 100.2$\pm$0.2 & 1.2$\pm$0.7 & 2.13$\pm$0.24\\
			&\sdp{2}& 99.6$\pm$0.2 & 100.1$\pm$0.1 & 0.5$\pm$0.3 & 5.83$\pm$0.79\\
			&\sdp{LB}& 98.2$\pm$1.3 & 100.3$\pm$0.3 & 2.2$\pm$1.4 & 1.48$\pm$0.18\\
			&\mioP& 99.4$\pm$0.5  & 100.0$\pm$0.0 & 0.6$\pm$0.5 & 441.90$\pm$258.29\\
			&&&&&\\
			\multirow{4}{*}{\texttt{crime}}&\sdp{1}&99.3$\pm$0.9 & 100.3$\pm$0.9 & 1.1$\pm$1.7 & 19.15$\pm$1.30\\
			&\sdp{2}& 99.7$\pm$0.4 & 100.2$\pm$0.8 & 0.5$\pm$1.0 & 43.86$\pm$2.38\\
			&\sdp{LB}& 98.7$\pm$1.0 & 100.7$\pm$1.3 & 2.0$\pm$2.3 & 5.30$\pm$0.35\\
			&\mioP& 99.5$\pm$0.4  & 100.1$\pm$0.1 & 0.6$\pm$0.4 & 518.03$\pm$175.65\\
			\hline \hline
	\end{tabular}}
\end{table}

\rev{We observe that instances with $\lambda=0.05$ are much easier to solve than those with $\lambda=0$: no numerical issues occur for \sdp{LB} or \texttt{persp}, and lower and upper bounds are much better for all methods. The mixed integer formulation \texttt{persp} comfortably solves the small instances with $p\leq 40$ to optimality, but \sdp{2} yields better lower bounds and gaps for the larger instances \texttt{diabetes} and \texttt{crime} in a fraction of the time used by \texttt{persp}. }

\rev{Figures~\ref{fig:diabetesl2} and \ref{fig:crimel2} present lower bounds and gaps as a function of the regularization parameter $\lambda$, for \texttt{diabetes} and \texttt{crime} datasets (with $k=15$).
We observe that for low value of $\lambda$, \texttt{persp} struggles to find good lower bounds, e.g., it is outperformed by all conic relaxations in \texttt{crime} for $\lambda\leq 0.02$, and is worse than \sdp{2} for $\lambda \leq 0.1$ in terms of lower bounds and gaps in both datasets. As $\lambda$ increases, all methods deliver better bounds, and \texttt{persp} is eventually able to solve all problems to optimality. }

\begin{figure}[!h]
	\centering
	\subfloat[Lower bounds]{\includegraphics[width=0.49\textwidth,trim={11cm 5.8cm 11cm 5.8cm},clip]{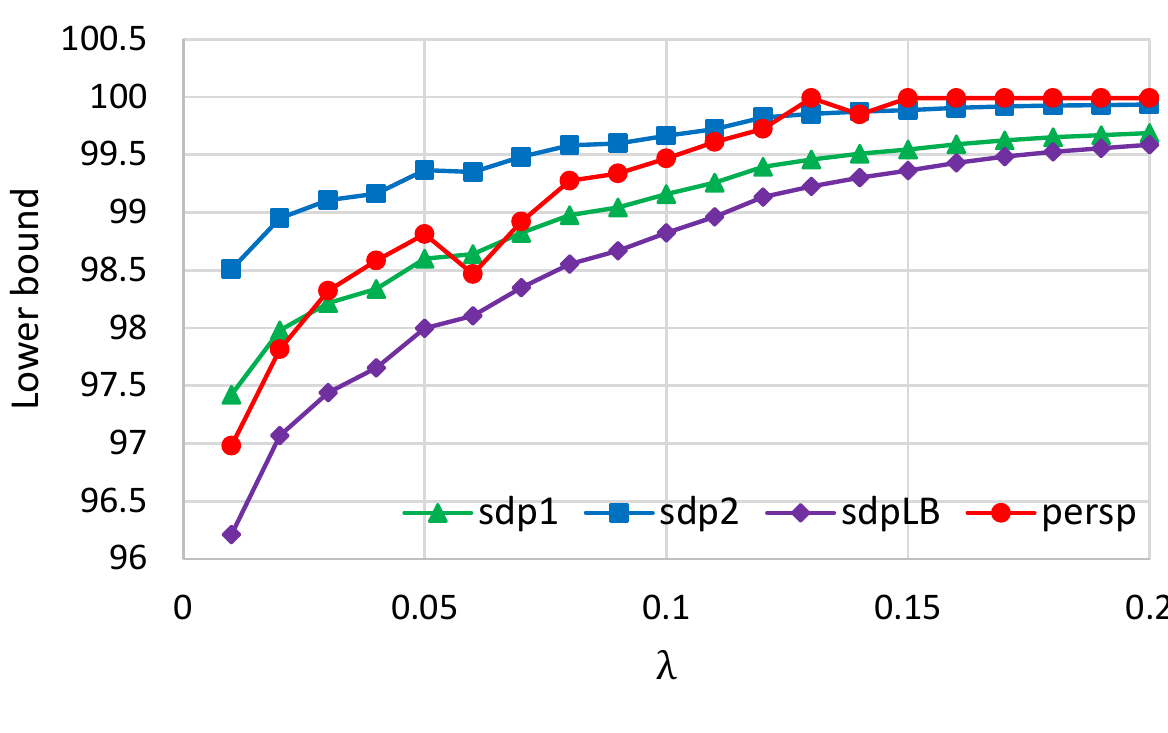}}\hfill\subfloat[Gaps]{\includegraphics[width=0.49\textwidth,trim={11cm 5.8cm 11cm 5.8cm},clip]{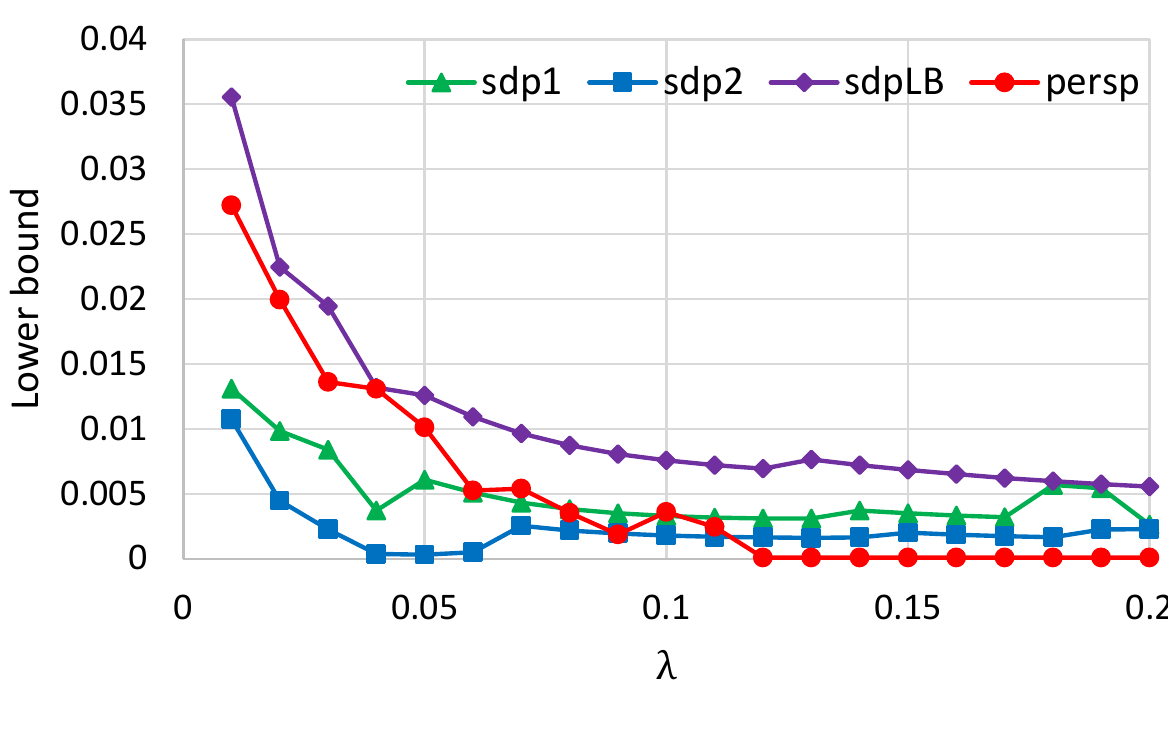}}
	\caption{\small Detailed results on the \texttt{diabetes} dataset with $k=15$.}
	\label{fig:diabetesl2}
\end{figure}

\begin{figure}[!h]
	\centering
	\subfloat[Lower bounds]{\includegraphics[width=0.49\textwidth,trim={11cm 5.8cm 11cm 5.8cm},clip]{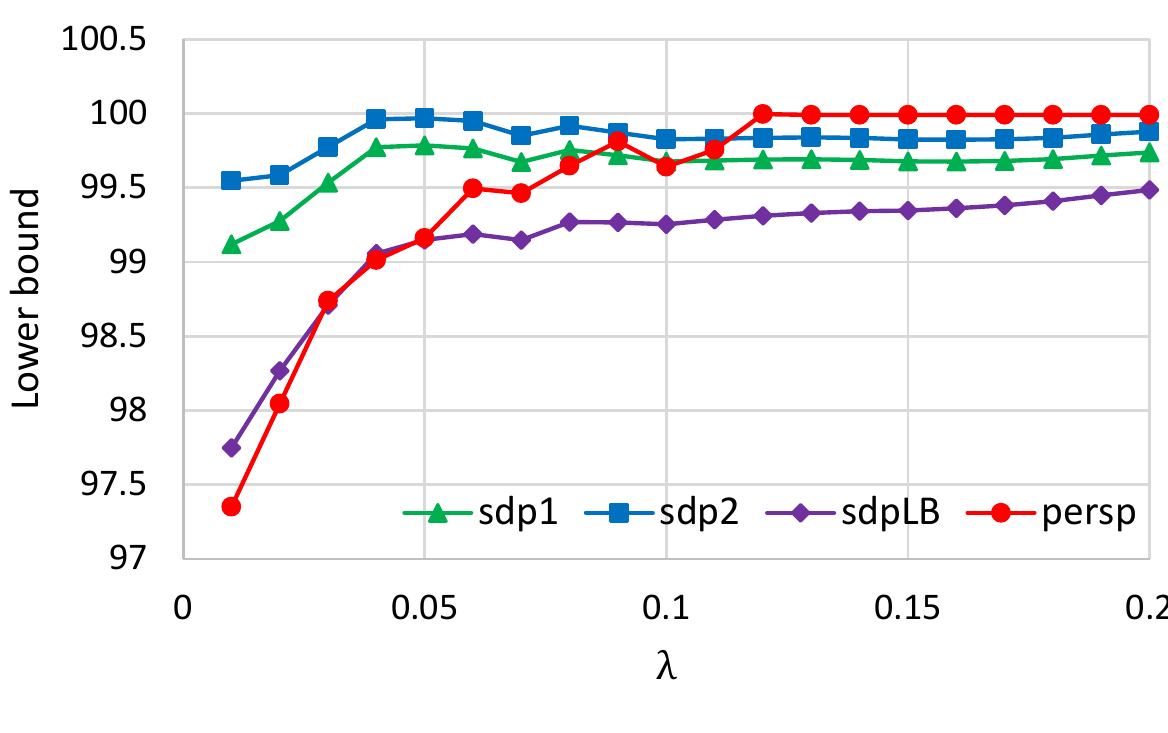}}\hfill\subfloat[Gaps]{\includegraphics[width=0.49\textwidth,trim={11cm 5.8cm 11cm 5.8cm},clip]{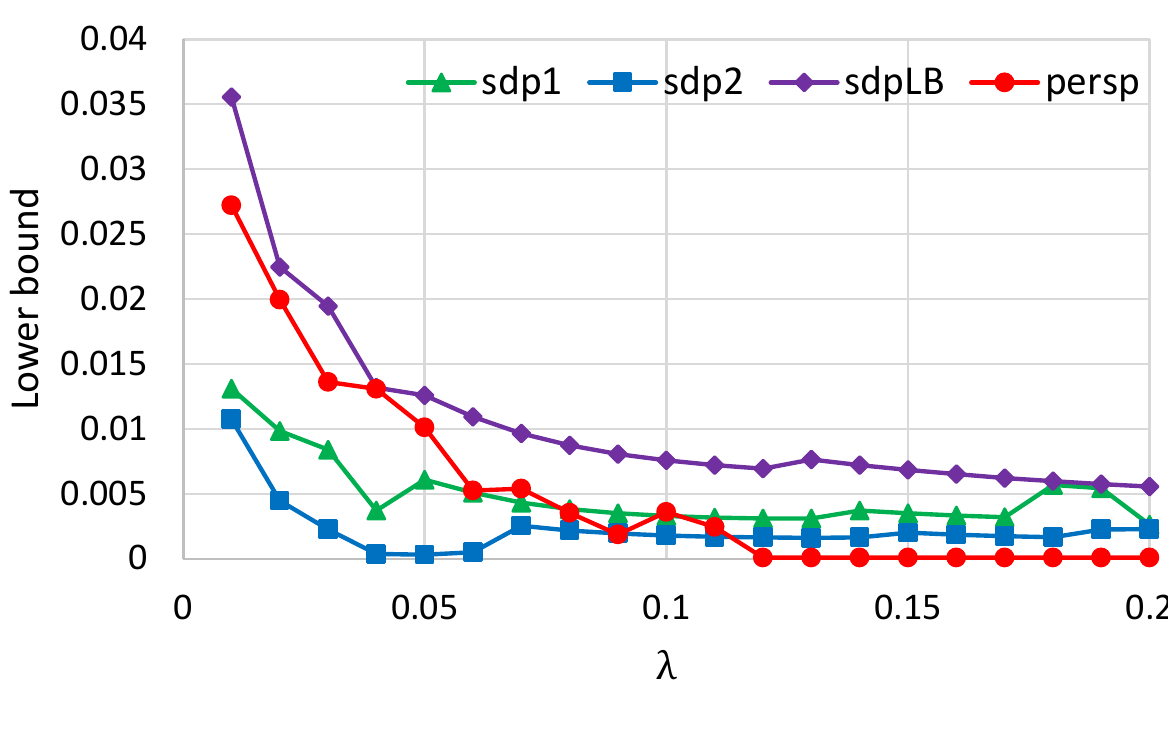}}
	\caption{\small Detailed results on the \texttt{crime} dataset with $k=15$.}
	\label{fig:crimel2}
\end{figure}

\rev{As expected, the performance of \texttt{persp} improves as $\lambda$ increases. The perspective relaxation discussed in \S\ref{sec:perspective} exploits the separable terms introduced by the $\ell_2$-regularization: as $\lambda$ increases, this separable terms have a larger weight in the objective, and the strength of the relaxation improves as a consequence. Note that the conic relaxations also improve with larger $\lambda$: they are based on decompositions of the matrix $\bs{X^\top X}+\lambda \bs{I}$ into one- and two-variable terms, and the addition of the separable terms allows for a much richer set of decompositions. For large values of $\lambda$, $\bs{X^\top X}+\lambda \bs{I}$ becomes highly diagonal dominant, and
the perspective relaxation alone provides a substantial strengthening. In this case, the advanced conic relaxations have a marginal impact and MIO methods with perspective strengthening performs better overall. In contrast, for low values of $\lambda$, the conic relaxations result in substantial strengthening over the perspective relaxation, and \sdp{r} outperforms \texttt{persp} as a consequence.  }

\subsubsection{\rev{The effect of model complexity $r$}}
\label{sec:comp_r}

\rev{In \S\ref{sec:comp_lambda0}--\ref{sec:comp_lambdaPos} we reported computations with $\sdp{r}$ with $r\leq 2$. In experiments with those datasets, \sdp{3} yields almost the same strengthening as \sdp{2}, but with much larger computational cost. Since \sdp{2} already achieves gaps close to $0$ in those instances, there is little room for improvement with higher values of $r$.}

\rev{If the matrix $\bs{X^\top X}+\lambda\bs{I}$ has high rank, which happens if $n>p$ or if $\lambda$ is large, then there are many ways to decompose it into low-dimensional rank-one terms, and \sdp{r} with $r$ small achieves good relaxations. In contrast, if the matrix $\bs{X^\top X}+\lambda\bs{I}$ has low rank, it may be difficult to extract low-dimensional rank-one terms. In the extreme case of a rank-one case matrix, while \sdp{p} results in the convex description, \sdp{r} with $r<p$ achieves no improvement. 
In this section we illustrate this phenomenon on 
small synthetic instances with $p=15$ and $n=10$.}

\rev{Specifically, we set the true sparsity to $s=5$, autocorrelation $\rho=0.35$, signal-noise-ration $\text{SNR}\in \{1,5\}$, sparsity $k\in\{3,4,5,6,7,8\}$, and for each combination of parameters we generate five instances. We report in Figure~\ref{fig:effectR} the gaps obtained by \sdp{r} for different values of $r$ and $\lambda$ -- averaging across instances and different values of SNR and $k$. In addition, Figure~\ref{fig:timeR} depicts the distribution of computational times required to solve the problems.} 

\begin{figure}[!h]
	\centering
	\subfloat[$\lambda=0$]{\includegraphics[width=0.33\textwidth,trim={9.5cm 5.8cm 10cm 5.8cm},clip]{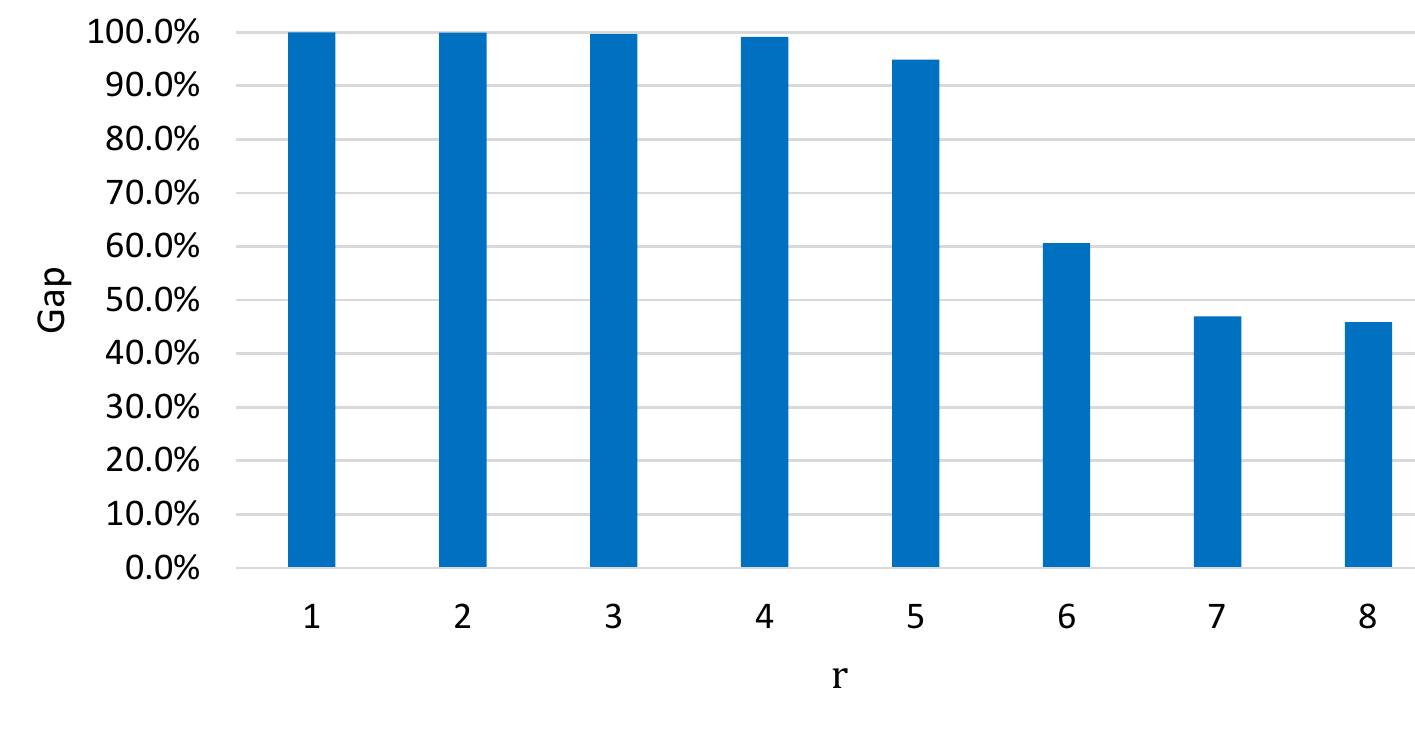}}\hfill	\subfloat[$\lambda=0.01$]{\includegraphics[width=0.33\textwidth,trim={9.5cm 5.8cm 10cm 5.8cm},clip]{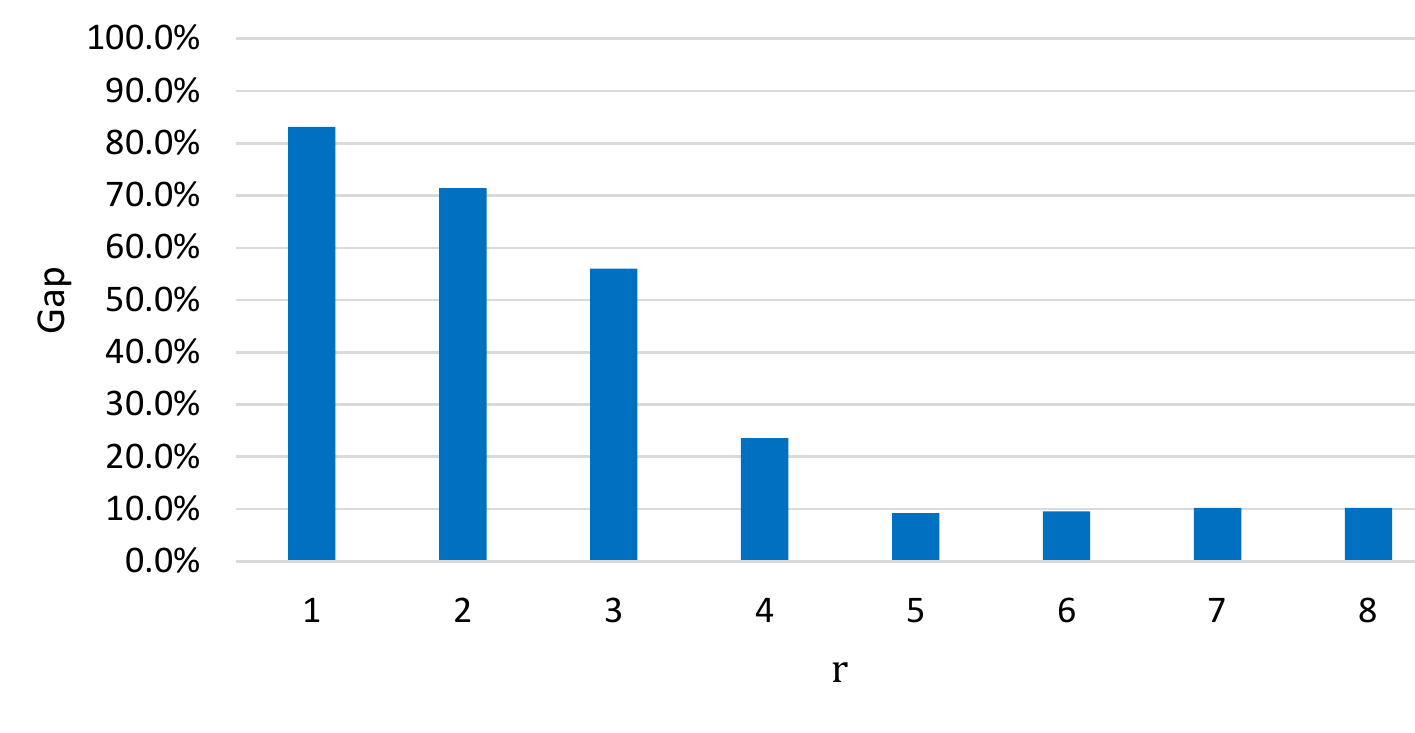}}\hfill	\subfloat[$\lambda=0.02$]{\includegraphics[width=0.33\textwidth,trim={9.5cm 5.8cm 10cm 5.8cm},clip]{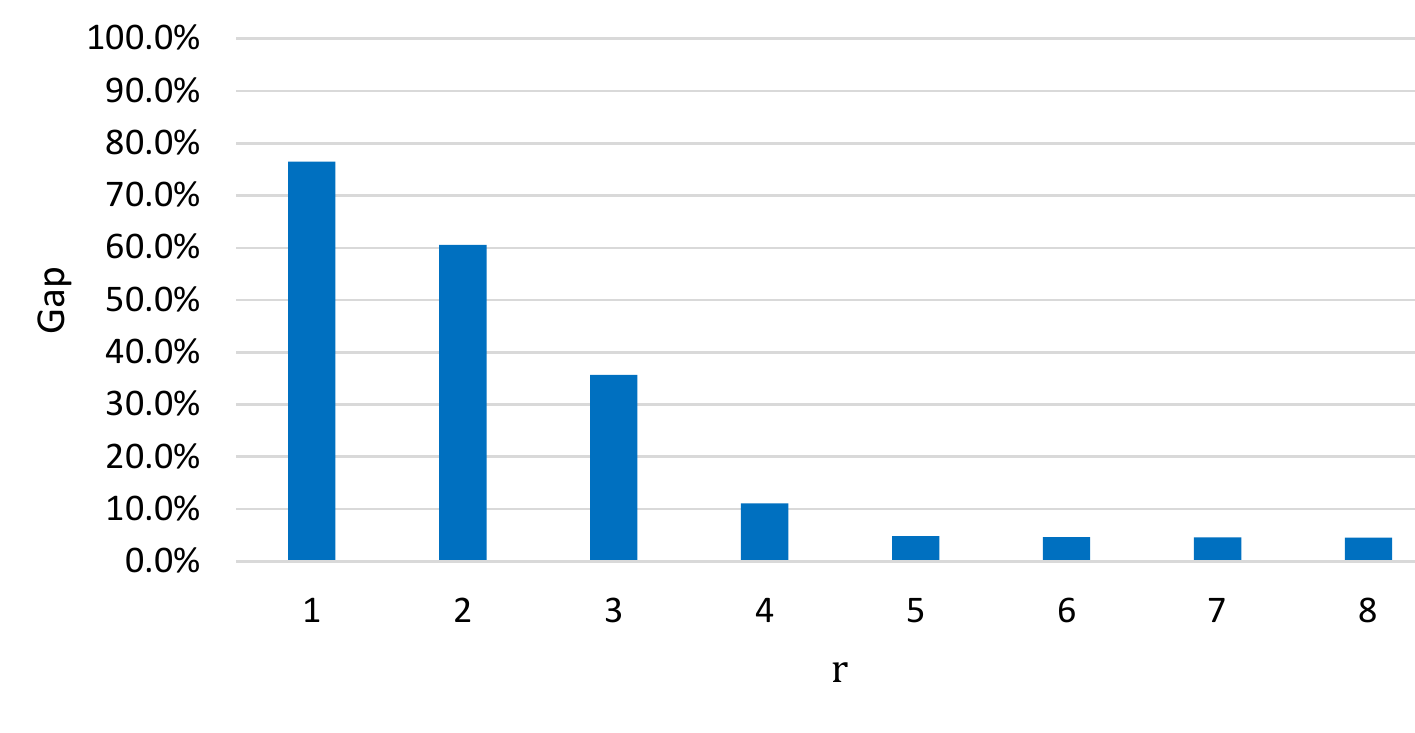}} \\
	\subfloat[$\lambda=0.05$]{\includegraphics[width=0.33\textwidth,trim={9.5cm 5.8cm 10cm 5.8cm},clip]{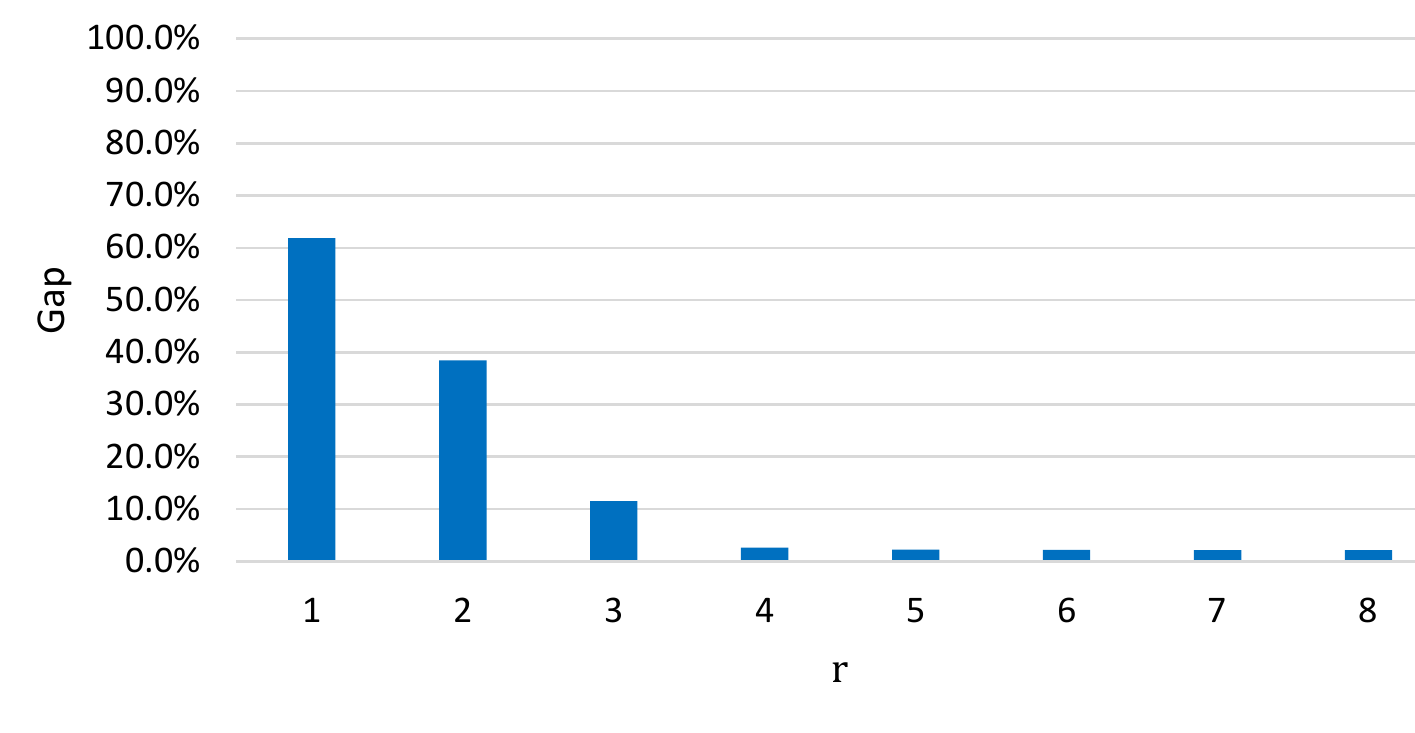}}\hfill	\subfloat[$\lambda=0.10$]{\includegraphics[width=0.33\textwidth,trim={9.5cm 5.8cm 10cm 5.8cm},clip]{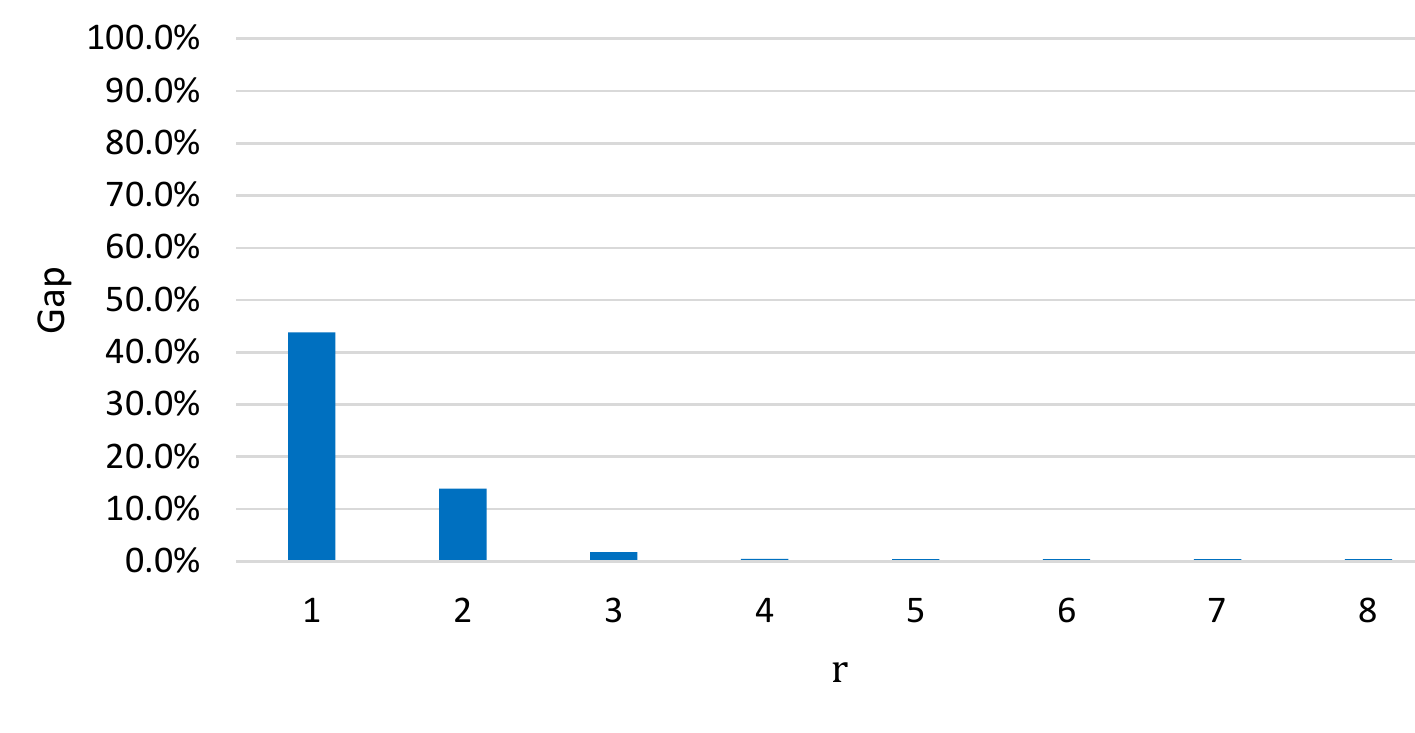}}\hfill	\subfloat[$\lambda=0.15$]{\includegraphics[width=0.33\textwidth,trim={9.5cm 5.8cm 10cm 5.8cm},clip]{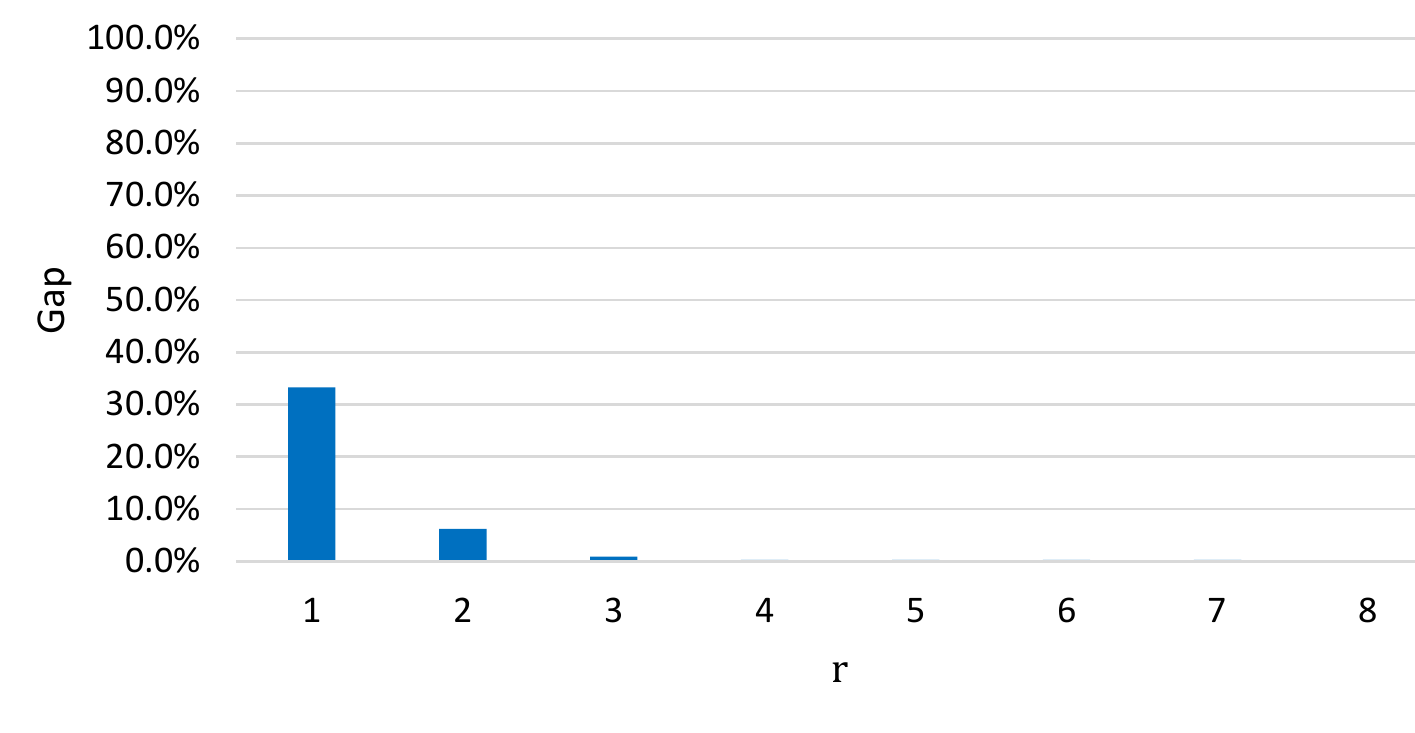}}
	\caption{\small \rev{Optimality gaps of \sdp{r}, $1\leq r \leq 8$}.}
	\label{fig:effectR}
\end{figure}

\begin{figure}
	\includegraphics[width=0.6\textwidth,trim={9.5cm 5.8cm 11cm 5.8cm},clip]{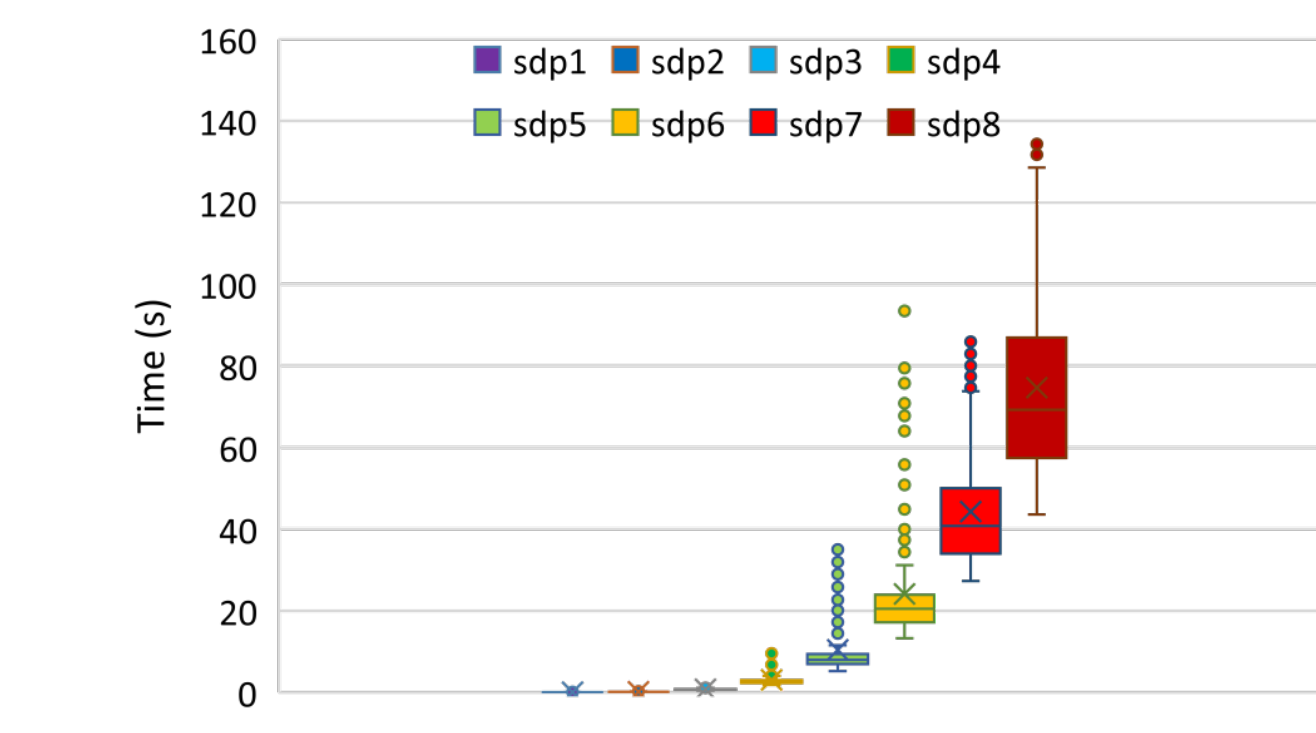}
		\caption{\small \rev{Time required to solve \sdp{r}, $1\leq r \leq 8$}.}
	\label{fig:timeR}
\end{figure}

\rev{We observe that for $\lambda=0$, \sdp{r} with $r\leq 4$ results in no strengthening and gaps of 100\%; \sdp{5} results in a small improvement (note that $5=p-n$), while \sdp{r} with $r\geq 6$ results in larger improvements. These results suggest that, with $\lambda=0$, stronger formulations require rank-one strengthening with at least $p-n$ variables. We also observe that, as $\lambda$ increases, the gaps reported by all methods decrease substantially, and the incremental strengthening obtained from larger values of $r$ decreases: for $\lambda\geq 0.05$ \sdp{4} performs almost identical to \sdp{8}, and for $\lambda=0.15$ \sdp{3} is similar to \sdp{8} and \sdp{2} already results in low optimality gaps. The computational time required to solve \sdp{r} scales exponentially with $r$ since the number of constraints increases exponentially as well. We conclude that \sdp{2} is well suited for the $p>n$ case or for medium values of $\lambda$ (for larger values \sdp{1} or even the simple perspective relaxation may be preferable), while \sdp{r} with $r\geq 3$ achieves a good improvement in relaxation quality for low values of $\lambda$, at the expense of larger computational times. }

\subsubsection{\rev{On scalability}}
\label{sec:comp_scal}

\rev{As discussed in \S\ref{sec:comp_r}, formulation \sdp{r} for large values of $r$ can be expensive to solve. Moreover, even \sdp{1} and \sdp{2} are semidefinite programs, which may not scale well for large values of $p$. In this section we present computations illustrating that while this is indeed the case, formulation \sdp{LB} --which replaces the semidefinite constraint $\bs{B}-\bs{\beta\beta'}\in S_+^P$ with the quadratic constraints \eqref{eq:socp1_sp}-- scales much better and in fact can significantly outperform \texttt{persp} in terms of relaxation quality. }

\rev{We generate synthetic instances with $p\in \{100,150,\dots,500\}$, $n=500$, true sparsity parameter $s=30$, autocorrelation $\rho=0.35$, signal-noise-ration $\text{SNR}\in \{1,5\}$, sparsity $k=30$; for each combination of parameters we generate five instances, and solve them for $\lambda\in \{0.01,0.02,0.05,0.15\}$ and $\mu=0$. Table~\ref{tab:resultsP} reports, for \sdp{1}, \sdp{2}, \sdp{LB} and \texttt{persp} --using formulation \eqref{eq:MIOPersp} with a time limit of 600 seconds--, the time required to solve the problems and the optimality gap proven.} 

\begin{table}[!h]
	\caption{\rev{Computational times and gaps on synthetic instances as a function of $p$. TL= Time Limit. $\dagger$= Unable to solve (either due to very large computational times or memory issues).  Numbers after ``$\pm$" are the sample standard deviation.}}
	\label{tab:resultsP}
	\setlength{\tabcolsep}{2pt}
	\scalebox{0.9}{
		\begin{tabular}{c | c c | c c | c c | c c }
			\hline
			\multirow{2}{*}{\texttt{$p$}}& \multicolumn{2}{c|}{\underline{\sdp{1}}}& \multicolumn{2}{c|}{\underline{\sdp{2}}} &\multicolumn{2}{c|}{\underline{\sdp{LB}}} &\multicolumn{2}{c}{\underline{\mioP}}\\
			&\texttt{time(s)}&\texttt{gap(\%)}
			&\texttt{time(s)}&\texttt{gap(\%)}
			&\texttt{time(s)}&\texttt{gap(\%)}
			&\texttt{time(s)}&\texttt{gap(\%)}\\
			\hline	
			100 & 19$\pm$3&0.5$\pm$0.6&44$\pm$9&0.1$\pm$0.1&5$\pm$1&1.0$\pm$1.1&TL&3.3$\pm$4.2\\
			150 & 153$\pm$19&1.1$\pm$1.5&356$\pm$61&0.2$\pm$0.4&20$\pm$2&1.9$\pm$2.2&TL&6.2$\pm$7.2\\
			200 & 673$\pm$64&2.7$\pm$3.0&1,691$\pm$165&0.6$\pm$0.9&42$\pm$3&4.3$\pm$4.0&TL&12.5$\pm$10.5\\
			250 & $\dagger$&$\dagger$&$\dagger$&$\dagger$&79$\pm$4&7.1$\pm$5.6&TL&17.2$\pm$13.2\\
		    300 & $\dagger$&$\dagger$&$\dagger$&$\dagger$&147$\pm$7&12.2$\pm$7.6&TL&21.7$\pm$14.6\\
		    350 & $\dagger$&$\dagger$&$\dagger$&$\dagger$&248$\pm$14&17.6$\pm$11.0&TL&25.9$\pm$17.0\\
		    400 & $\dagger$&$\dagger$&$\dagger$&$\dagger$&391$\pm$36&24.0$\pm$14.9&TL&29.1$\pm$18.9\\
		    450 & $\dagger$&$\dagger$&$\dagger$&$\dagger$&394$\pm$46&32.2$\pm$18.3&TL&34.3$\pm$21.0\\
		    500 & $\dagger$&$\dagger$&$\dagger$&$\dagger$&462$\pm$43&39.3$\pm$21.9&TL&38.8$\pm$22.8\\
			\hline
	\end{tabular}}
\end{table}

\rev{We observe that \mioP \ is unable to solve the problems within the 10 minute time limit and results in larger gaps than all other approaches, despite using substantially more time in most cases. We also observe that \sdp{r} formulations struggle in instances with $p\geq 200$. Interestingly, \sdp{2} requires consistently 2-4 times more than \sdp{1} regardless of the dimension $p$. A similar factor was observed in Tables~\ref{tab:resultsLambda0} and \ref{tab:resultsLambda005} with real data, suggesting that computational times with \sdp{2} are within the same order-of-magnitude as \sdp{1}. Finally, \sdp{LB} is substantially faster than both \sdp{1} and \sdp{2}. While it results in larger gaps than \sdp{2} as expected, since the high-dimensional constraint \eqref{eq:sdp2_P} is relaxed, it still yields better optimality gaps than \mioP. 
}

\subsection{Inference study on synthetic instances}\label{sec:synthetic}

We now present inference results on synthetic data using the same simulation setup as in \cite{bertsimas2016best,hastie2017extended}, see \cite{hastie2017extended} for an extended description. Specifically, we generate synthetic data as described in \S\ref{sec:data}, and use the evaluation metrics used in \cite{hastie2017extended}, described next.

\subsubsection{Evaluation metrics} Let $\bs{x_0}$ denote the test predictor drawn from $\mathcal{N}_p(\bs{0},\bs{\Sigma})$ and let $y_0$ denote its associated response value drawn from $\mathcal{N}(\bs{x_0^\top\beta_0},\sigma^2)$. Given an estimator $\hat{\bs{\beta}}$ of $\bs{\beta_0}$, the following metrics are reported:
\begin{description}
	\item[Relative risk] $$\text{RR}(\hat{\bs{\beta}})=\frac{\mathbb{E}\left(\bs{x_0^{\top}\hat{\beta}}-\bs{x_0^{\top}\beta_0}\right)^2}{\mathbb{E}\left(\bs{x_0^\top \beta_0}\right)^2}$$
	with a perfect score $0$ and null score of $1$.
	\item[Relative test error] $$\text{RTE}(\hat{\bs{\beta}})=\frac{\mathbb{E}\left(\bs{x_0^{\top}\hat{\beta}}-y_0\right)^2}{\sigma^2}$$
	with a perfect score of $1$ and null score of SNR+1.
	\item[Proportion of variance explained]
	$$1-\frac{\mathbb{E}\left(\bs{x_0^{\top}\hat{\beta}}-y_0\right)^2}{\text{Var}(y_0)}$$
	with perfect score of SNR/(1+SNR) and null score of 0.
	\item[Sparsity] We record the number of nonzeros\footnote{An entry $\hat{\beta}_i$ is deemed to be non-zero if $|\hat{\beta}_i|>10^{-5}$. This is the default integrality precision in commercial MIO solvers.}, $\|\bs{\hat{\beta}}\|_0$, as done in \cite{hastie2017extended}. Additionally, we also report the number of variables correctly identified, given by $\left.\sum_{i=1}^p\mathbbm{1}\{\hat{\beta}_i\neq 0 \text{ and }(\beta_0)_i\neq 0\}\right.$. 
\end{description}

\subsubsection{Procedures} In addition to the training set of size $n$, a validation set of size $n$ is generated with the same parameters, matching the precision of leave-one-out cross-validation. 
We use the following procedures to obtain estimators $\bs{\hat{\beta}}$. 
\begin{description}
	\item[elastic net] \rev{We solve the elastic net procedure using the parametrization
	$$\min_{\bs{\beta}\in \R^p}\|\bs{y}-\bs{X\beta}\|_2^2+\lambda\left(\alpha \|\bs{\beta}\|_1+(1-\alpha)\|\bs{\beta}\|_2^2\right)$$
	where $\alpha,\lambda\geq 0$ are the regularization parameters. We let $\alpha=0.1\ell$ for integer $0\leq\ell\leq 10$, we generated 50 values of $\lambda$ ranging from $\lambda_{max}=\|\bs{X^\top y}\|_\infty$ to $\lambda_{max}/200$ on a log scale, and using the pair $(\lambda,\mu)$ that results in the best prediction error on the validation set. A total of 500 $(\alpha,\lambda)$ pairs are tested.}
	\item[\sdp{2}] The estimator obtained from solving \sdp{2} ($\lambda=\mu=0$) for all values of $k=0,\ldots,7$ and choosing the one that results in the best prediction error on the validation set. 
\end{description}
The \texttt{elastic net} procedure approximately corresponds to the \texttt{lasso} procedure with 100 tuning parameters used in \cite{hastie2017extended}. Similarly, \sdp{2} with cross-validation approximately corresponds to the \texttt{best subset} procedure with $51$ tuning parameters\footnote{\citet{hastie2017extended} use values of $k=0,\ldots,50$. Nonetheless, in our computations with the same tuning parameters, we found that values of $k\geq 8$ are never selected after cross-validation. Thus our procedure with $8$ tuning parameters results in the same results as the one with 51 parameters from a statistical viewpoint, but requires only a fraction of the computational effort.} used in \cite{hastie2017extended}; nonetheless, the estimators from \cite{hastie2017extended} are obtained by running a MIO solver for 3 minutes, while ours are obtained from solving to optimality a strong convex relaxation. 


\subsubsection{Optimality gaps and computation times} Before describing the statistical results, we briefly comment on the relaxation quality and computation time of \sdp{2}. Table~\ref{tab:timesSynt} shows, for instances with $n=500$, $p=100$, and $s=5$, the optimality gap and relaxation quality of \sdp{2} --- each column represents the average over ten instances generated with the same parameters. In all cases, \sdp{2} produces optimal or near-optimal estimators, with optimality gap at most $0.3\%$. In fact, with \sdp{2}, we find that 97\% of the estimators for $\rho=0.00$ and 68\% of the estimators with $\rho=0.35$ are provably \emph{optimal}\footnote{A solution is deemed optimal if \texttt{gap}$<10^{-4}$, which is the default parameter in MIO solvers.} for \eqref{eq:bestSubsetSelection}. For a comparison, \citet{hastie2017extended} report that, in their experiments, the MIO solver (with a time limit of three minutes) is able to prove optimality for only 35\% of the instances generated with similar parameters. Although \citet{hastie2017extended} do not report optimality gaps for the instances where optimality is not proven, we conjecture that such gaps are significantly larger than those reported in Table~\ref{tab:timesSynt} due to weak relaxations with big-$M$ formulations. In summary, for this class of instances, \sdp{2} is able produce optimal or practically optimal estimators of \eqref{eq:bestSubsetSelection} in about 30 seconds. 
\begin{table}[!h]
	\caption{Optimality gap  and computation time (in seconds) of \sdp{2} with $n=500$, $p=100$, $s=k=5$, $\lambda=\mu=0$.}
	\label{tab:timesSynt}
	\scalebox{0.9}{
		\begin{tabular}{c c | c c c c c c c c c c |c}
			\hline \hline
			\multicolumn{2}{c|}{SNR} & 0.05 &0.09 & 0.14& 0.25& 0.42 & 0.71 & 1.22 & 2.07& 3.52& 6.00 &\textbf{avg}\\ 
			\hline
			\multirow{2}{*}{$\rho=0.00$}&\texttt{gap}&0.1 & 0.1 &0.0 & 0.0 &0.0 & 0.0 &0.0 & 0.0 & 0.0 & 0.0 & \textbf{0.0}\\
			&\texttt{time}&45.2 & 38.8 & 38.6 & 29.5 & 29.3 & 28.4 & 27.4 & 26.3 & 26.4 & 25.9 & \textbf{31.6}\\
			&&&&&&&&&&&&\\
			\multirow{2}{*}{$\rho=0.35$}&\texttt{gap}&0.3 & 0.2 &0.3 & 0.1 &0.0 & 0.0 &0.0 & 0.0 & 0.0 & 0.0&\textbf{0.1}\\
			&\texttt{time}&48.0 & 47.6 & 49.4 & 44.1 & 39.3 & 30.7 & 29.0 & 29.1 & 27.3 & 28.0&\textbf{37.3}\\
			\hline \hline
	\end{tabular}}
\end{table}

\subsubsection{Results: accuracy metrics}

Figure \ref{fig:synt000} plots the relative risk, relative test error, proportion of variance explained and sparsity results as a function of the SNR for instances with $n=500$, $p=100$, $s=5$ and $\rho=0$. Figure~\ref{fig:synt035} plots the same results for instances with $\rho=0.35$. The setting with $\rho=0.35$ was also presented in \cite{hastie2017extended}. 
\begin{figure}[p!h]
	\centering
	\subfloat{\includegraphics[width=0.7\textwidth,trim={9.5cm 5.8cm 9.5cm 5.8cm},clip]{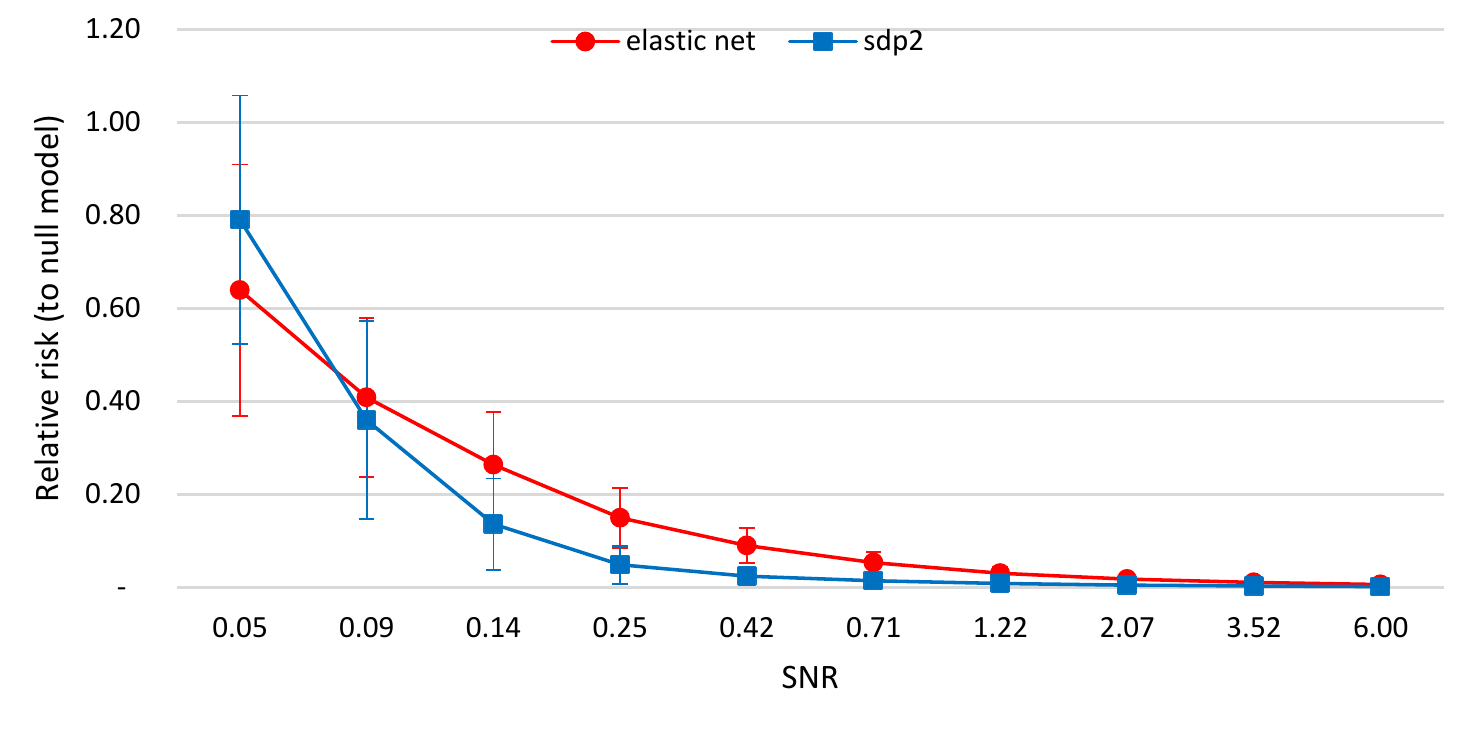}}\hfill
	\subfloat{\includegraphics[width=0.7\textwidth,trim={9.5cm 5.8cm 9.5cm 5.8cm},clip]{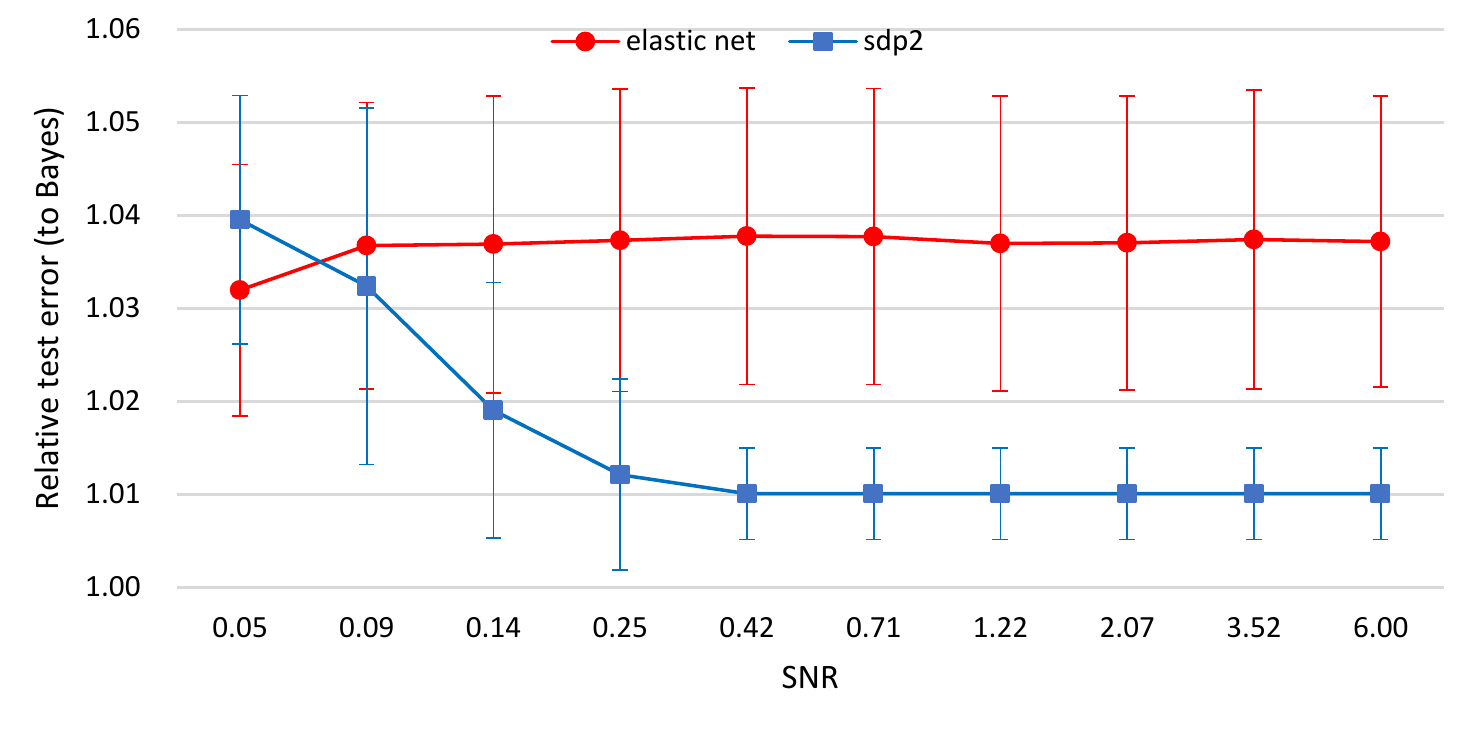}}\hfill
	\subfloat{\includegraphics[width=0.7\textwidth,trim={9.5cm 5.8cm 9.5cm 5.8cm},clip]{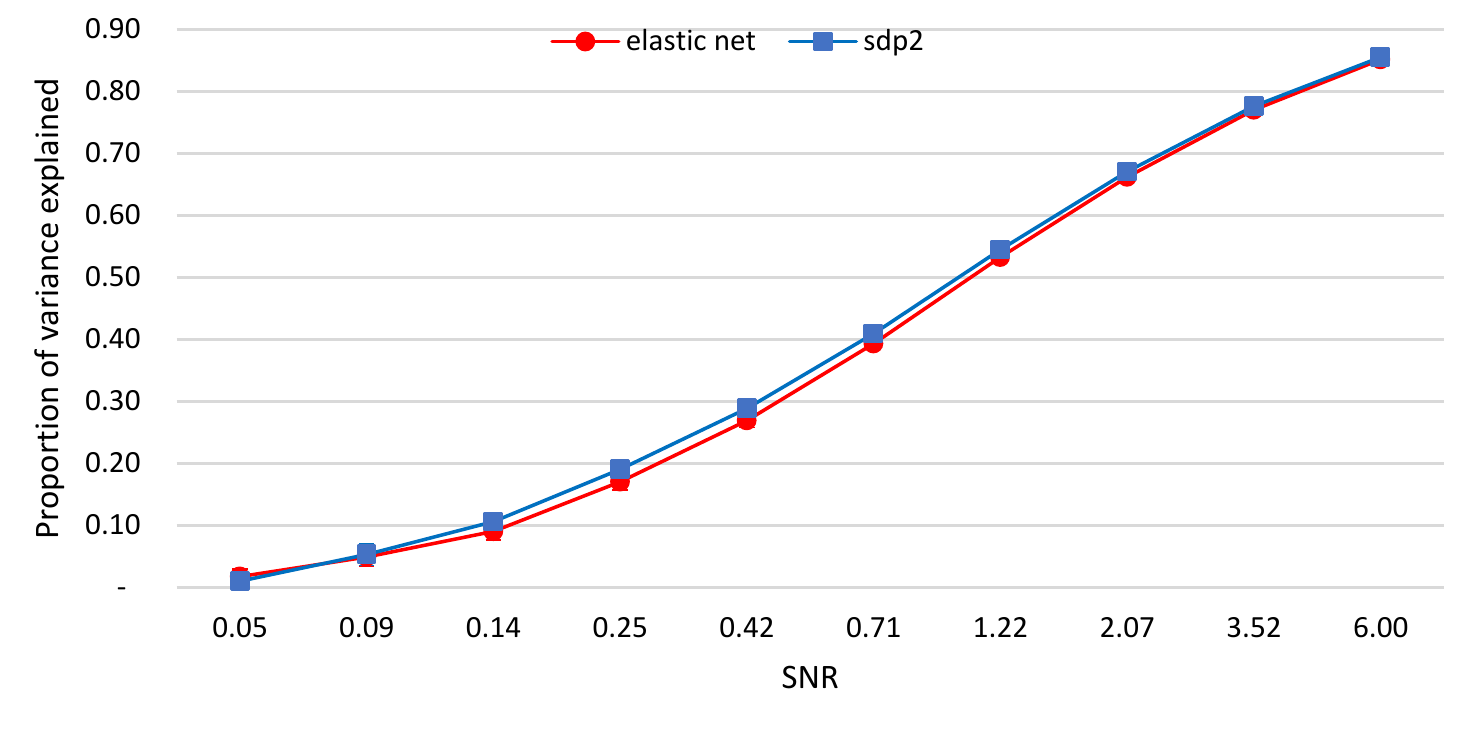}}\hfill
	\subfloat{\includegraphics[width=0.9\textwidth,trim={7cm 5.8cm 7cm 5.8cm},clip]{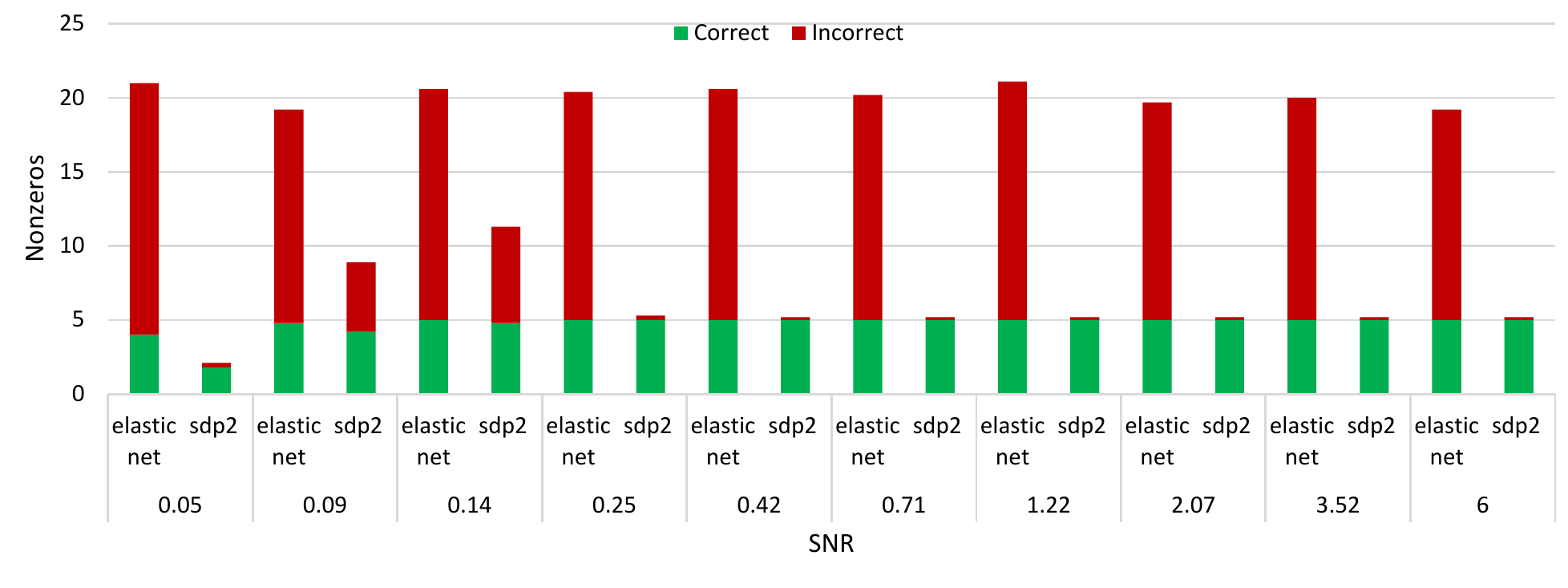}}\hfill
	\caption{\small Relative risk, relative test error, proportion of variance explained and sparsity as a function of SNR, with $n=500$, $p=100$, $s=5$ and $\rho=0.00$.}
	\label{fig:synt000}
\end{figure}

\begin{figure}[p!h]
	\centering
	\subfloat{\includegraphics[width=0.7\textwidth,trim={9.5cm 5.8cm 9.5cm 5.8cm},clip]{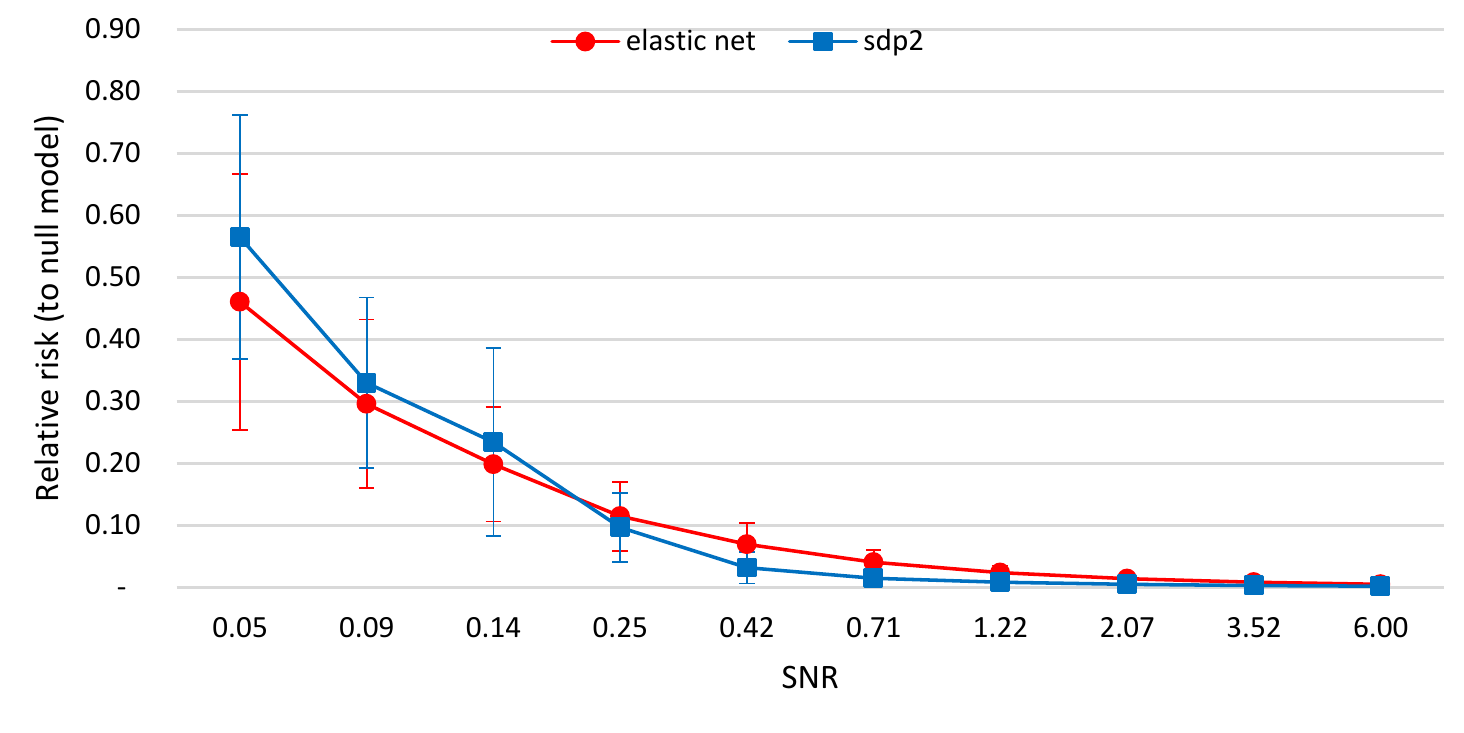}}\hfill
	\subfloat{\includegraphics[width=0.7\textwidth,trim={9.5cm 5.8cm 9.5cm 5.8cm},clip]{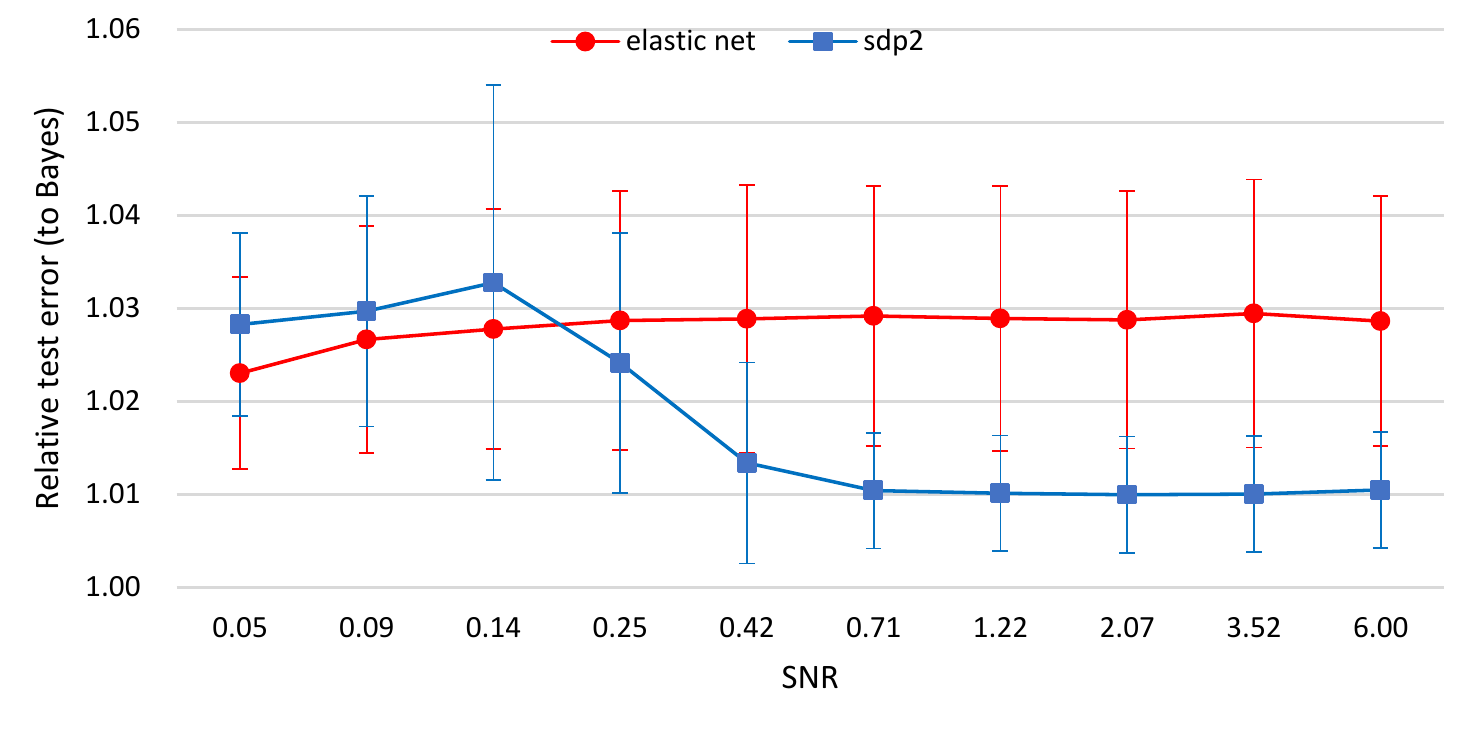}}\hfill
	\subfloat{\includegraphics[width=0.7\textwidth,trim={9.5cm 5.8cm 9.5cm 5.8cm},clip]{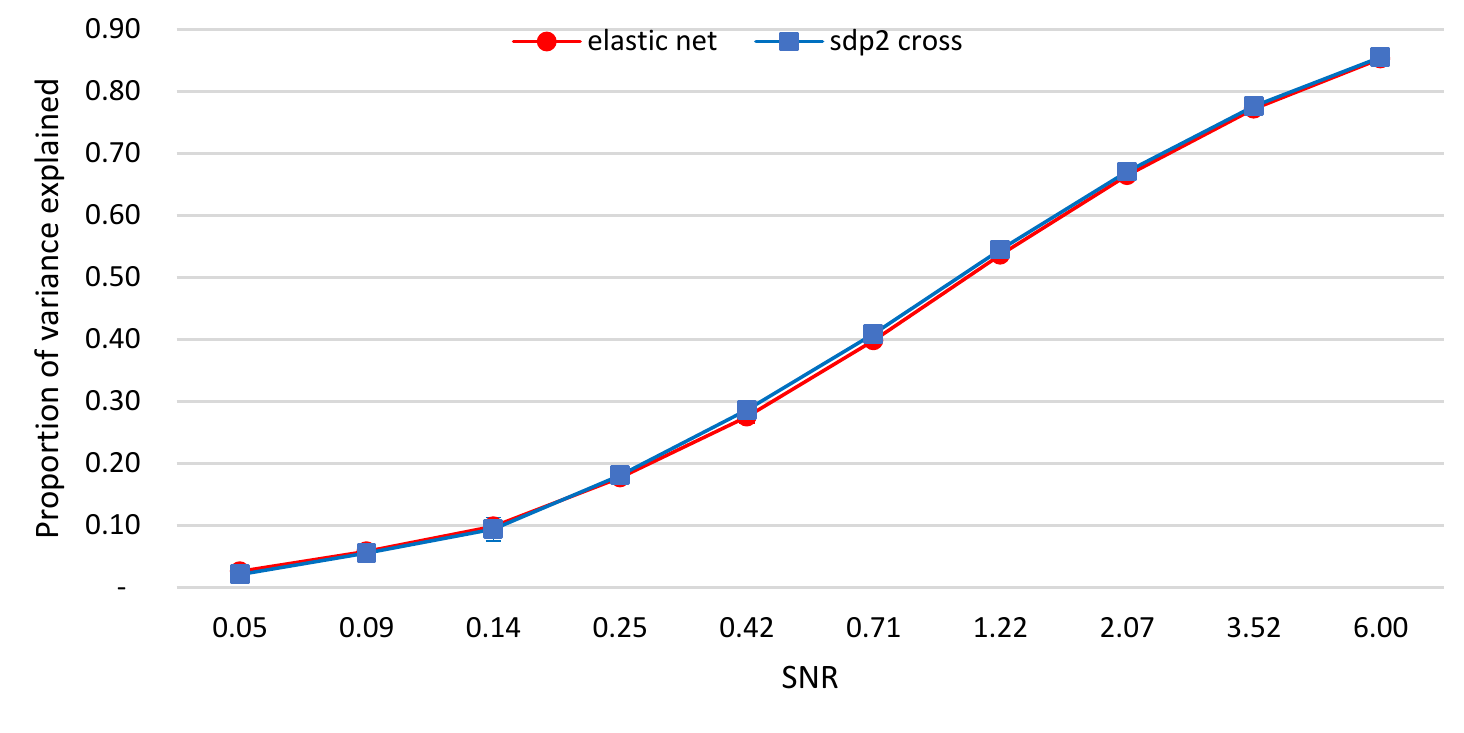}}\hfill
	\subfloat{\includegraphics[width=0.9\textwidth,trim={7cm 5.8cm 7cm 5.8cm},clip]{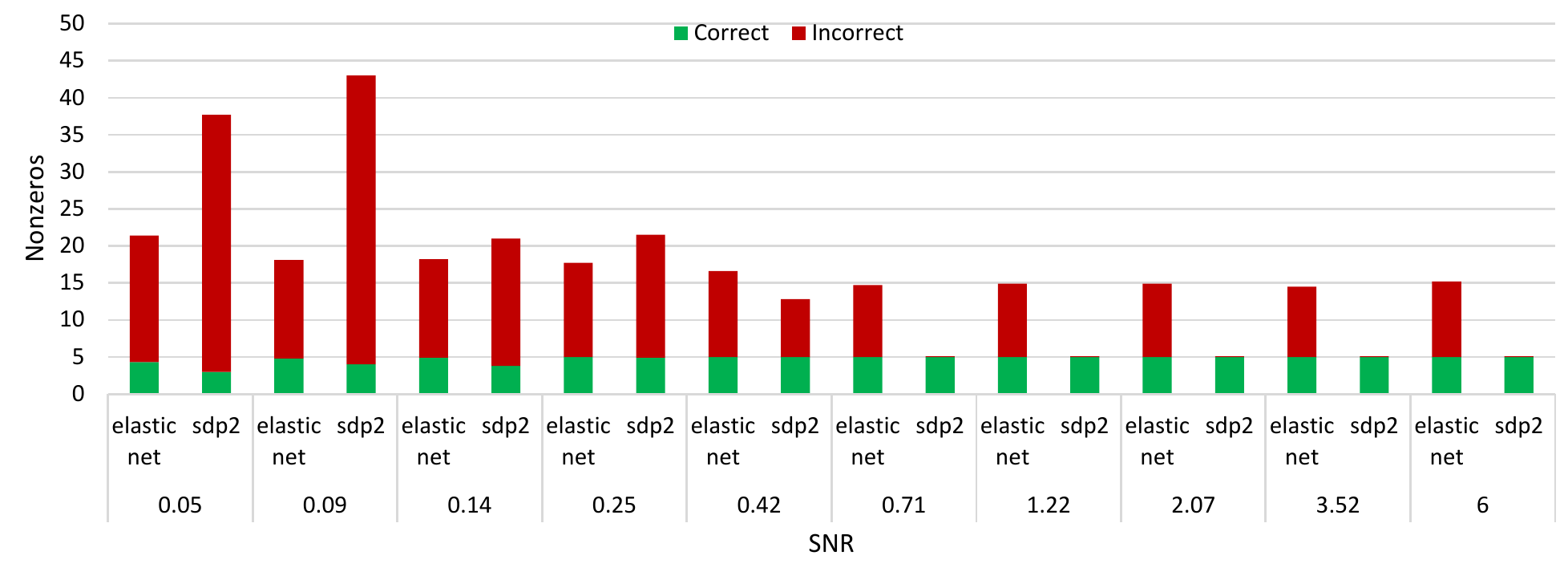}}\hfill
	\caption{\small Relative risk, relative test error, proportion of variance explained and sparsity as a function of SNR, with $n=500$, $p=100$, $s=5$ and $\rho=0.35$.}
	\label{fig:synt035}
\end{figure}

We see that \texttt{elastic net} outperforms \sdp{2} in low SNR settings, i.e., in SNR$=0.05$ for $\rho=0$ and SNR$\leq 0.14$ for $\rho=0.35$, but results in worse predictive performance for all other SNRs. Moreover, \sdp{2} is able to recover the true sparsity pattern of $\bs{\beta_0}$ for sufficiently large SNR, while \texttt{elastic net} is unable to do so. We also see that \sdp{2} performs comparatively better than \texttt{elastic net} in instances with $\rho=0$. Indeed, for large autocorrelations $\rho$, features where $(\beta_0)_i=0$ still have predictive value, thus the dense estimator obtained by \texttt{elastic net} retains a relatively good predictive performance (however, such dense solutions are undesirable from an \emph{interpretability} perspective). In contrast, when $\rho=0$, such features are simply noise and \texttt{elastic net} results in overfitting, while methods that deliver sparse solution such as \sdp{2} perform much better in comparison. We also note that \sdp{2} selects model corresponding to sparsities $k<s$ in low SNRs, while it consistently selects models with $k\approx s$ in high SNRs. We point out that, as suggested in \cite{mazumder2017subset}, the results for low SNR could potentially be improved by fitting models with $\mu>0$.

\section{Conclusions}\label{sec:conclusions}

In this paper we derive strong convex relaxations for sparse regression. The relaxations are based on the ideal formulations for rank-one quadratic terms with indicator variables. The new relaxations are formulated as semidefinite optimization problems in an extended space and are stronger and more general than the state-of-the-art formulations. In our computational experiments, the proposed conic formulations outperform the existing approaches, both in terms of accurately approximating the best subset selection problems and of achieving
desirable estimation properties in statistical inference problems with sparsity.

\section*{Acknowledgments}

A. Atamt\"urk is supported, in part, by Grant No. 1807260 from the National Science Foundation. A. G\'omez is supported, in part, by Grants No. 1818700 and 2006762 from the National Science Foundation.

\bibliographystyle{apalike}
\bibliography{Bibliography}

\end{document}